\documentclass{article}

\usepackage{mkolar_definitions}

\usepackage[preprint]{neurips_2022}
\usepackage[utf8]{inputenc} % allow utf-8 input
\usepackage[T1]{fontenc}    % use 8-bit T1 fonts
\usepackage{hyperref}       % hyperlinks
\usepackage{url}            % simple URL typesetting
\usepackage{booktabs}       % professional-quality tables
\usepackage{amsfonts}       % blackboard math symbols
\usepackage{nicefrac}       % compact symbols for 1/2, etc.
\usepackage{microtype}      % microtypography
\usepackage{xcolor}         % colors
\usepackage{xspace}

\usepackage{bbm}
\usepackage{graphicx}
\graphicspath{ {./} }
\usepackage{algorithm}
\usepackage{algorithmic}
\usepackage{hyperref}

\usepackage{commath}
\usepackage[toc,page,header]{appendix}
\usepackage{minitoc}

\DeclareMathOperator\Reg{Reg}
\DeclareMathOperator\MDP{\Mcal}
\DeclareMathOperator\Onehot{Onehot}
\newcommand{\step}{\mathrm{Step}}
\newcommand{\feedback}{\zeta_E}
\newcommand{\imitloss}{L}
\newcommand{\auxu}{\hat{u}}
\newcommand{\auxpi}{\hat{\pi}}
\newcommand{\auxincrdata}{\hat{D}}
\newcommand{\auxest}{\hat{g}}
\newcommand{\auxtheta}{\hat{\theta}}
\newcommand{\iterloss}{F}
\newcommand{\base}{\Bcal}
\newcommand{\oloa}{\ensuremath{\mathrm{OLOA}}\xspace}
\newcommand{\algreduct}{\ensuremath{\textsc{Logger}}\xspace}
\newcommand{\algm}{\ensuremath{\textsc{MFTPL}}\xspace}
\newcommand{\algmp}{\ensuremath{\textsc{MFTPL-EG}}\xspace}
\newcommand{\algreductm}{\ensuremath{\textsc{Logger-M}}\xspace}
\newcommand{\algreductmp}{\ensuremath{\textsc{Logger-ME}}\xspace}
\newcommand{\DAgger}{\ensuremath{\textsc{DAgger}}\xspace}
\newcommand{\oilcsc}{\ensuremath{\text{COIL}}\xspace}
\newcommand{\bimatrix}{\ensuremath{\textsc{Poly$^{12}$-Bimatrix}}\xspace}
\newcommand{\vimdp}{\ensuremath{\textsc{Poly$^{6}$-VI-MDP}}\xspace}
\newcommand{\vi}{\ensuremath{\textsc{VI}}\xspace}
\renewcommand{\paragraph}[1]{\noindent\textbf{#1}}
\newcommand{\poly}{\mathrm{poly}}
\newcommand{\ff}{f}
\newcommand{\bias}{\mathrm{Bias}}
\newcommand{\est}{g}
\newcommand{\pred}{\hat{g}}

\newcommand{\Prob}{\mathrm{Pr}}

\newcommand{\SReg}{\mathrm{SReg}}
\newcommand{\DReg}{\mathrm{DReg}}
\newcommand{\sign}{\mathrm{sign}}
\newcommand{\polylog}{\mathrm{polylog}}
\newcommand{\multinomial}{\mathrm{Multinomial}}
\newcommand{\approxerr}{\ensuremath{\mathrm{ApproxErr}}\xspace}

\newcommand{\incrdata}{D}
\newcommand{\incrdist}{\Dcal^E}
\newcommand{\Sp}{T}
\newcommand{\const}{\gamma}

\newtheorem{assumption}{Assumption}

\newcommand{\ampolicy}{\sigma}
\newcommand{\coord}[1]{[#1]}
\newcommand{\algo}{\ensuremath{\textsf{Alg1}}\xspace}
\newcommand{\algop}{\ensuremath{\textsf{Alg1}'}\xspace}

\newcommand{\reducef}{\ensuremath{\textsf{f}}\xspace}
\newcommand{\reduceb}{\ensuremath{\textsf{g}}\xspace}
\newcommand{\funsab}{\ensuremath{p}\xspace}
\newcommand{\funnsab}{\ensuremath{q}\xspace}
\newcommand{\SLReg}{\mathrm{LReg}}

\newcommand{\dom}{\mathrm{dom}}

\usepackage{enumitem}
\setlist[itemize]{topsep=.5pt,itemsep=0pt,parsep=2pt,leftmargin=2em}
\setlist[enumerate]{topsep=.5pt,itemsep=0pt,parsep=2pt,leftmargin=2em}

\title{On Efficient Online Imitation Learning \\ via Classification}

\author{%
  Yichen Li \\
  University of Arizona\\
  \texttt{yichenl@arizona.edu}\\
  \And
  Chicheng Zhang\\
  University of Arizona\\
  \texttt{chichengz@cs.arizona.edu}\\
}

\def\shownotes{1}  \ifnum\shownotes=1
\newcommand{\authnote}[2]{$\ll$\textsf{\footnotesize #1 notes: #2}$\gg$}
\else
 \newcommand{\authnote}[2]{}
\fi

\begin{document}
\maketitle
 \doparttoc % Tell to minitoc to generate a toc for the parts
\faketableofcontents % Run a fake tableofcontents command for the partocs
% \part{} % Start the document part
% \parttoc 
\begin{abstract}

Imitation learning (IL) is a general learning paradigm for tackling sequential decision-making problems. Interactive imitation learning, where learners can interactively query for expert demonstrations, has been shown to achieve provably superior sample efficiency guarantees compared with its offline counterpart or reinforcement learning. 
%and computational
In this work, we study classification-based online imitation learning (abbrev. $\ensuremath{\text{COIL}}$) and the fundamental feasibility to design oracle-efficient regret-minimization algorithms in this setting, with a focus on the general nonrealizable case. We make the following contributions: (1) we show that in the $\ensuremath{\text{COIL}}$ problem, any proper online learning algorithm cannot guarantee a sublinear regret in general;
(2) we propose  $\ensuremath{\textsc{Logger}}$,
an improper online learning algorithmic framework, that reduces $\ensuremath{\text{COIL}}$ to \emph{online linear optimization}, by utilizing a new definition of mixed policy class;
(3) we design two oracle-efficient algorithms within the $\ensuremath{\textsc{Logger}}$ framework that enjoy different sample and interaction round complexity tradeoffs, and conduct finite-sample analyses to show their improvements over naive behavior cloning; (4) we show that under the standard complexity-theoretic assumptions, efficient dynamic regret minimization is infeasible in the $\ensuremath{\textsc{Logger}}$ framework. Our work puts classification-based online imitation learning, an important IL setup, into a firmer foundation.

\end{abstract}

\section{Introduction}

% \chicheng{I've rewritten this section a bit. Yichen, can you help check if the story makes sense?}
Imitation learning (IL), also known as learning from expert demonstrations~\cite{pomerleau1988alvinn,osa2018algorithmic}, is a general paradigm for training intelligent behavior for sequential decision making tasks. IL has been successfully deployed in many applications, such as autonomous driving~\cite{pomerleau1988alvinn,pan2020imitation}, robot arm control~\cite{wang2017robust}, game playing~\cite{silver2016mastering}, and sequence prediction~\cite{daume2009search,bengio2015scheduled}. It is now well-known that with the help of a demonstrating expert, 
an IL agent can bypass the exploration challenges of reinforcement learning, achieving a much lower sample requirement than reinforcement learning agents~\cite{sun2017deeply}. 
%get around
% $\pi^E$
% imitation learning

Two major IL paradigms have been studied in the literature: offline and interactive. In offline IL~\cite{abbeel2004apprenticeship,syed2007game,ziebart2008maximum,ho2016generative}, the learner receives a set of expert demonstrations ahead of time; in contrast, in interactive IL~\cite{daume2009search,ross2011reduction,ross2014reinforcement,judah2014active}, the learner has the ability to interactively query the expert for demonstrations on states at its disposal, allowing expert feedback to be provided in a targeted manner.  In both settings, the goal of the learner is to output a policy $\hat{\pi}$ that competes with the expert's policy $\pi^E$, by consuming as few resources (e.g. expert annotations) as possible. Between these two, interactive IL is known to be able to achieve superior policy performance than its offline counterpart under certain favorable assumptions on the expert policy and the environment, in that learning agents can use interaction to address the compounding error challenge~\cite{ross2010efficient,rajaraman2021value}.

Despite the recent progress in the fundamental limits of the interactive imitation learning in the realizable  setting~\cite{sun2017deeply,rajaraman2020toward,rajaraman2021value}, the statistical and computational limits of the interactive imitation learning in the general nonrealizable setting remain open. One promising and influential algorithmic framework for studying and analyzing interactive IL in the nonrealizable setting is \DAgger (Data Aggregation)~\cite{ross2011reduction}, whose key insight is to reduce interactive IL to regret minimization in online learning~\cite{shalev2011online}. Specifically, it constructs a $N$-round online learning game, where at every round $n \in [N]$,  the learner outputs some policy $\pi_n$ from a policy class $\base$, and incurs a loss $\iterloss_n(\pi)$; the loss is carefully constructed so that the learner's instantaneous loss $\iterloss_n(\pi_n)$ characterizes  current policy $\pi_n$'s competitiveness compared to the expert $\pi^E$. 
A representative example of $\iterloss_n(\pi)$ is the expected disagreement between $\pi$ and $\pi^E$ on the state occupancy distribution induced by $\pi_n$, used in the original \DAgger paper~\cite{ross2011reduction}, which can be expressed as the expected zero-one loss of $\pi$ on a distribution of classification examples -- we call such setting \emph{Classification-based Online Imitation Learning} (abbrev. \oilcsc).
The \DAgger reduction framework has spurred an active line of research on IL~\cite{ross2014reinforcement,sun2017deeply,cheng2018convergence,cheng2019accelerating,cheng2019predictor,lee2021dynamic}: it enables conversions from stochastic online optimization algorithms with static or dynamic regret guarantees to IL algorithms with different output policy suboptimality guarantees, allowing the research community to directly translate new results in online learning to the field of IL.

Perhaps surprisingly, from a fundamental perspective, rigorous design of efficient regret minimization algorithms for \oilcsc has been largely overlooked by the prior literature. Specifically, many works assume a fixed parameterization of policies in $\base$, and assume that $F_n(\pi)$'s are convex in $\pi$'s underlying parameters to allow for no-regret online convex optimization~\cite{ross2011reduction,sun2017deeply,cheng2018convergence}. Although natural, this viewpoint has two issues: (1) in \DAgger's reduction, the learner uses finite-sample approximations of $F_n(\pi)$, which are often discontinuous in $\pi$'s underlying parameters (e.g. given a policy $\pi_\theta(s) = \sign(\inner{\theta}{s})$ as a linear classifier, its zero-one loss on a  state, $I(\pi_\theta(s) \neq \pi^E(s))$ is discontinuous in $\theta$), making stochastic gradient-based methods inapplicable.
Convex surrogate loss functions has been proposed as a popular workaround~\cite{ross2011reduction}, but it is well-known that in the nonrealizable setting, even for the special case of supervised learning, minimizing convex surrogate losses can result in very different models compared to minimizing the original zero-one classification loss~\cite{ben2012minimizing};
(2) it makes the usage of policy classes with complex parameterization (e.g. rule-based policies such as decision trees) difficult, as convexity is hard to establish for such classes. 

%cost-sensitive
% for online optimization infeasible
%~\cref{sec:proper}
%is perhaps coming as a surprise:
%there exists an environment, an expert, and a policy class $\Pi$, such that
\paragraph{An overview of our results.} In this paper, we bridge the above-mentioned gaps by studying the fundamental feasibility of designing efficient regret minimization algorithms for \oilcsc, putting the study of  statistical and computational limits of interactive imitation learning in the general nonrealizable setting into a firmer foundation. 
Our first result is that, analogous to Cover's impossibility result in online classification~\cite{cover1966behavior}, in the \oilcsc setting,
any proper online learning algorithm (that outputs a sequence of policies $\{\pi_n\}_{n=1}^N$ from the original class $\base$) cannot guarantee sublinear regret in general (\autoref{sec:proper}). 

%, for every $\pi$ in 
%(Definition~\ref{def:mixed-class})
% in Proposition~\ref{prop:il-olo}
The above negative result motivates the design of improper learning algorithms for regret minimization. 
To this end, we propose to choose policies from a mixed policy class $\Pi_\base$, and provide an algorithmic framework, \algreduct, that reduces \oilcsc to \emph{online linear optimization}. In a nutshell, \algreduct uses a natural parameterization on $\Pi_\base$ that allows to express $F_n(\pi)$ as a linear function of the underyling parameters of $\pi \in \Pi_\base$. We show that any online linear optimization algorithm with (high-probability) regret guarantees can be plugged into \algreduct to obtain an algorithm for \oilcsc with policy suboptimality guarantees (\autoref{sec:mixture-class}).

%aim at ing
%called \emph{separators}
Next, enabled by the \algreduct framework, we design computationally efficient algorithms for static regret minimization. Assuming access to an offline cost-sensitive classification (CSC) oracle $\Ocal$, and a set of unlabeled \emph{separator examples} for $\Bcal$~\cite{syrgkanis2016efficient, dudik2020oracle}, we design \algreductm, a sample and computationally efficient algorithm. Using $O(1/\epsilon^2)$ interaction rounds and $O(1/\epsilon^2)$  expert annotations, \algreductm enjoys a per-round static regret of $\epsilon$ (\autoref{sec:reg-sqrt-n}). 
Underlying \algreductm is a delicate utilization of the connection between Follow-the-Perturbed-Leader and Follow-the-Regularized-Leader, two well-known online learning algorithm families, first observed by~\cite{abernethy2014online}.
Moreover, by exploiting the predictability of the \oilcsc problem~\cite{rakhlin2013online,cheng2019accelerating,cheng2019predictor}, we design an efficient algorithm \algreductmp, that enjoys a per-round static regret of $\epsilon$, with $O(1/\epsilon)$ interaction rounds and  $O(1/\epsilon^2)$  expert annotations (\autoref{sec:reg-const}).
Its reduced number of interaction rounds can enable a more practical deployment of IL agents, especially when interactive expert annotations come in batches or with delays. 
%are 
%\chicheng{emphasize static regret}

% (that $\mathrm{PPAD}$-complete problems cannot be solved in randomized polynomial time)~\cite{chen2006settling,daskalakis2009complexity}
Finally, we study efficient dynamic regret minimization in the \algreduct framework (\autoref{sec:dynamic-reg}). We show that this is unlikely to be feasible: 
under a standard complexity-theoretic assumption, 
no oracle-efficient algorithms can output policies in $\Pi_\Bcal$ with sublinear dynamic regret. Due to space constraints, we discuss key related works throughout the paper, and defer additional related works to Appendix~\ref{sec:add-rel-work}.

\section{Preliminaries}
\label{sec:prelims}
% \chicheng{Yichen, my edit suggestions are red, and the previous texts (suggested to be replaced / removed) are in blue. Some edits are not perfect - if you see other points that need to be addressed, please edit them as you see fit.}

\paragraph{Basic definitions.} Define $[n] := \cbr{1,\ldots, n}$. Define indicator function 
$I(\cdot)$ such that $I(E)=1$ if condition $E$ is true, and $=0$ otherwise. We use $\Delta(W)$ to denote the set of probability distributions over a finite set $W$, and use $\Onehot(w,W) \in \Delta(W)$ to denote the delta mass on $w \in W$. 
For a finite $W$, we will oftentimes treat $u \in \RR^{|W|}$  (e.g. $u \in \Delta(W)$) as a $\abr{W}$-dimensional vector; for $w \in W$, denote by $u\coord{w}$ the $w$-th coordinate of $u$. We abuse the notation of $\{\cdot\}$ to denote multisets.
% Unless mentioned otherwise, we hide the $\ln\rbr{1/\delta}$ factors in $\tilde{O}$, where $\delta$ denotes any probability within $\rbr{0,1}$.

% \chicheng{Would something like $\delta_w$ or $e_w$ also suffice?}
% \yichen{We need Onehot as a function, which is used in our MFTPL algorithm}

\paragraph{Episodic MDPs.} We study imitation learning in episodic Markov decision processes (MDPs).
An episodic MDP $\Mcal$ is a tuple $(\Scal,\Acal,H,c,\rho,P)$, where $\Scal$ is a finite state space (that can be exponentially large), $\Acal$ is a finite action set, $H \in \NN^+$ is the episode length, $c: \Scal\times\Acal \rightarrow [0,1]$ is the cost function, and $\rho \in \Delta(\Scal)$ is the initial state distribution. Also, $P = \{P_t\}_{t=1}^{H-1}$ denotes $\Mcal$'s transition dynamics, with $P_t: \Scal \times \Acal \to \Delta(\Scal)$ being the transition probability at step $t$. Throughout, we use $S$ and $A$ to denote $|\Scal|$ and $|\Acal|$, respectively.
Without loss of generality, we assume that $\Mcal$ is layered, where $\Scal$ can be partitioned into $H$ disjoint sets $\{\Scal_t\}_{t=1}^H$; the initial distribution $\rho$ is supported on $\Scal_1$, and transition distribution $P_t(\cdot \mid s,a)$ is supported on $\Scal_{t+1}$ for all $t,s,a$. 
For state $s \in \Scal$, define $\step(s)$ as the step $t$ such that $s \in \Scal_t$.

%Here,
%the initial distribution of states and
% \chicheng{We should assume that we are working with layered MDP. (In non-layered MDP, usually we consider history-independent policy $\pi = (\pi_t)_{t=1}^H$, which introduces another inconvenient subscript.) Yichen, can you add relevant definitions for layered MDP in (in addition to the ones below)? Take a look at, for example \url{https://arxiv.org/pdf/2107.08622.pdf}, but please do not copy from there (paraphrasing would be fine, by copying will violate academic integrity.)}

A learning agent interacts with $\Mcal$ for one episode using the following protocol: for every step $t \in [H]$: it observes a state $s_t \in \Scal_t$, takes an action $a_t \in \Acal$, incurs cost $c(s_t, a_t)$, and transitions to next state $s_{t+1} \sim P_t(\cdot \mid s_t, a_t)$ except for the last step when it stops. Given a stationary policy  $\pi: \Scal \rightarrow \Delta(\Acal)$, we use $\pi(\cdot|s)$ to denote the action distribution of $\pi$ on $s$. Denote by $\EE_\pi$ and $\PP_\pi$ the expectation and probability over executing (i.e. rolling out) policy $\pi$ in $\Mcal$.
Given policy $\pi$, its state occupancy distribution at step $t$ is defined as $d_\pi^t(\cdot) := \PP_\pi(s_t = \cdot)$; its average state occupancy distribution is denoted as $d_\pi := \frac1H \sum_{t=1}^H d_\pi^{t}$.
Let $J(\pi) :=  \EE_{\pi} \sbr{ \sum_{t=1}^H c(s_t, a_t)} = H \cdot \EE_{s \sim d_{\pi}} \EE_{a \sim \pi(\cdot|s)}\sbr{c(s,a)}$ denote the expected cumulative cost of $\pi$ over an episode. 
For policy $\pi$, we denote its value function $V_{\pi}(s) := \EE\sbr{ \sum_{t=\step(s)}^{H} c(s_t,a_t) \mid s, \pi}$ and
action-value function 
$Q_\pi(s, a) :=  c(s,a) + \EE\sbr{ \sum_{t=\step(s)+1}^{H} c(s_t,a_t) \mid  s, a, \pi}$; 
in words, they are the expected costs of rolling out $\pi$ starting from $s$ and $(s,a)$, respectively.  
For policy $\pi$, define its advantage function as $A_{\pi}(s,a):= Q_{\pi}(s,a) - V_{\pi}(s)$, which measures the expected performance difference by one step deviation of $\pi$ by taking action $a$ at state $s$. Also, we define the recoverability constant as
the ability of $\pi$ to recover from deviation when rolled out in $\Mcal$:

% as , i.e.
%of 
% Here, we assume the expert can recover from a mistake and still enjoy low cumulative cost by the $\mu$-recoverability assumption, i.e. }
%\quad
\begin{definition}[$\mu$-recoverability]
A (MDP, policy) pair $(\Mcal,\pi)$ is said to be $\mu$-recoverable, if $\forall s\in \Scal, a\in \Acal$, $\abr{ A_\pi(s,a) }\leq \mu $.~\footnote{The $\mu$-recoverability definition here is slightly different from the original ones in~\cite{ross2011reduction}, in that it also requires that $A_\pi(s,a) \geq -\mu$; we can drop this assumption with a slightly worse sample complexity analysis.}  
\label{def:recover}
\end{definition}

\paragraph{Interactive IL.} We study interactive imitation learning~\cite{daume2009search,ross2010efficient}, where the learner has access to a stationary deterministic demonstrating expert $\pi^E$ and would like to learn a policy with low expected cost. Throughout, we assume that $(\Mcal,\pi^E)$ is $\mu$-recoverable, for some $\mu \leq H$ that can potentially be $\ll H$. At each interaction round, the learner interacts with $\Mcal$ for a few episodes, obtaining trajectories of the form $\tau = (s_1,a_1, s_2, a_2, \ldots, s_H, a_H)$ and queries the expert for feedback on some of the states. Specifically, given a state $s$, the feedback given by the expert is of the form $(\feedback(s,a))_{a \in \Acal} \in \RR^{A}$. Two notable examples are: (1) direct expert annotation~\cite{ross2011reduction}, i.e. given state $s$, expert provides demonstration $\pi^E(s)$, and we use it to construct an $A$-dimensional feedback $(\feedback(s,a))_{a \in \Acal} = (\mu \cdot I(a \neq \pi^E(s)))_{a \in \Acal}$;
(2) estimates of value functions based on experts' rollout~\cite{ross2014reinforcement}, i.e. 
$(\feedback(s,a))_{a \in \Acal} = (A^{E}(s,a))_{a \in \Acal} := (A_{\pi^E}(s,a))_{a \in \Acal}$.\footnote{Strictly speaking, in AggreVate and its variants~\cite{ross2014reinforcement,sun2017deeply}, the learner requests expert rollout to obtain \emph{unbiased estimators} of $A^E(s,a)$; our sample complexity analysis can also be adapted to this setting.}
Throughout, we assume  $\feedback(s,a)$ satisfies $\forall s\in \Scal, a \in \Acal$, $A^E(s,a) \leq \feedback(s,a) \leq \mu \cdot I(a \neq \pi^E(s))$; this is satisfied by the two examples above.
\label{def:feedback}
Define a policy $\pi$'s \emph{imitation loss} as: $\imitloss(\pi) := \EE_{s \sim d_{\pi}}\EE_{a \sim \pi(\cdot|s)} \sbr{\feedback(s,a)}$.
\label{def:imitloss}
By the performance difference lemma~\cite{Kakade2002ApproximatelyOA} (see Lemma~\ref{lem: performance difference lemma} in Appendix~\ref{sec:auxiliary}), $J(\pi) - J(\pi^E) \leq H \cdot \imitloss(\pi)$, implying that if $\pi$ has a small imitation loss, it will have expected cost competitive with $\pi^E$.
In light of this connection, in interactive IL, 
the learner would like to obtain policy $\hat{\pi}$ with low $\imitloss(\hat{\pi})$. Subject to this, throughout the paper, we consider the learner to optimize two measures of data efficiency: 
\begin{itemize}
%the goal of the learner is
%\itemsep0em 
\item \emph{Sample complexity}: the total number of expert annotations $\zeta_E$ requested. A smaller sample complexity reduces the total cost of expert annotations (which often takes human effort).
% often come with significant human effort, and having a smaller sample complexity reduces the overall cost of the training process. 
\item \emph{Interaction round complexity}: total number of adaptive interaction rounds. A small number of interaction rounds enables more parallelized annotations within an interaction round, and mitigates the delayed annotation issue~\cite{yang2013buy,wang2021one}.
\end{itemize}

\paragraph{The \DAgger reduction framework for interactive IL.} The \DAgger framework reduces minimizing $\imitloss(\pi)$ to no-regret online learning~\cite{ross2011reduction,ross2014reinforcement}. It constructs a $N$-round online learning game, where at every round $n \in [N]$,  the learner outputs some policy $\pi_n$, which induces a loss function $\iterloss_n(\pi) = \EE_{s \sim d_{\pi_n}} \EE_{a \sim \pi(\cdot|s)} \sbr{\feedback(s,a)}$ \label{def:iterloss}.
Its key insight is that, by the definition of $\cbr{\iterloss_n}_{n=1}^N$, minimizing the online learning cumulative loss $\sum_{n=1}^N \iterloss_n(\pi_n)$ is equivalent to minimizing the cumulative imitation losses of $\pi_n$'s, i.e. $\sum_{n=1}^N L(\pi_n)$. 
Research efforts in online learning~\cite{shalev2011online} have focused on the design of algorithms that can output $\cbr{\pi_n}_{n=1}^N$ with static regret $\SReg_N(\Bcal)$ or dynamic regret $\DReg_N(\Bcal)$ against some benchmark policy class $\Bcal$, formally:
%\EE\sbr{  }
%As is standard in online learning, 
%have proposed approaches that provide guarantees of
\begin{equation}
\SReg_N(\Bcal):= 
\sum_{n=1}^N \iterloss_n(\pi_n) -  \min_{\pi \in \Bcal}  \sum_{n=1}^N \iterloss_n(\pi), 
\quad 
\DReg_N(\Bcal) := 
%\EE\sbr{ 
\sum_{n=1}^N \rbr{ \iterloss_n(\pi_n) -  \min_{\pi \in \Bcal} \iterloss_n(\pi) } .
\label{eqn:regret}
\end{equation}

{
\makeatletter
\renewcommand{\ALG@name}{Protocol}
\makeatother
\begin{algorithm}[t]
\caption{Classification-based Online Imitation Learning (\oilcsc)}
%Classification-based Online Imitation learning (\oilcsc)
%The \DAgger reduction framework~\cite{ross2011reduction}
\begin{algorithmic}
\FOR{$n=1,\ldots,N$}
\STATE Learner outputs policy $\pi_n$.
\STATE Loss function $\iterloss_n(\pi) := \EE_{s \sim d_{\pi_n} } \EE_{a \sim \pi(\cdot \mid s)} \sbr{\feedback(s,a)}$.
\STATE Learner draws samples from $\incrdist_{\pi_n}$ to obtain information about loss $F_n$, via interacting with $\Mcal$ and querying the  expert for annotation $\feedback$.
%, and incurs loss
%$\iterloss_n(\pi_n)$
\ENDFOR
\STATE \textbf{Goal of learner:} minimize $\sum_{n=1}^N \iterloss_n(\pi_n) = \sum_{n=1}^N \imitloss(\pi_n)$.
\end{algorithmic}
\label{protocol:ilrm}
\end{algorithm}
}

Assuming that the learner chooses policies $\cbr{\pi_n}_{n=1}^N$ from another stationary policy class $\Bcal_0$ (which may or may not be $\Bcal$), 
the following proposition shows that static and dynamic regret guarantees in the induced online learning game can be converted to policy suboptimality guarantees:

\begin{proposition}[e.g. \cite{cheng2018convergence}]
\label{prop:reg-conversion}
For any $N \in \mathbb{N}^+$ and online learner that outputs $\cbr{\pi_n}_{n=1}^N \in \Bcal_0^N$, define $\bias(\Bcal,\Bcal_0,N) := \mathop{\max}\limits_{\{\upsilon_n\}_{n=1}^N \in \Bcal_0^N } \min\limits_{\pi \in \Bcal} \EE_{s \sim \bar{d}_{N}} \EE_{a \sim \pi(\cdot \mid s) }\sbr{I(a \neq \pi^E(s))}$, where $\bar{d}_{N} := \frac{1}{N} \sum_{n=1}^N d_{\upsilon_n}$. Then, choosing $\hat{\pi}$ uniformly at random from $\cbr{\pi_n}_{n=1}^N$ has guarantee:
%\in \Bcal^N
%\frac{1}{N} \sum_{n=1}^N J(\pi_n) - J(\pi^E)
%^2\cdot
\[
\EE\sbr{ J(\hat{\pi}) - J(\pi^E) }
\leq 
H \cdot
\min\cbr{ \mu\cdot \bias(\Bcal,\Bcal_0,N)  + \frac{\EE[\SReg_N(\Bcal)]}{N}, \; \mu \cdot \bias(\Bcal,\Bcal_0,1)  + \frac{\EE[\DReg_N(\Bcal)]}{N} }.
\]
% \chicheng{todo: define $\eta$ as the recoverability constant for $\Mcal$ and $\pi^E$.}

\end{proposition} 
%and measures the fraction of mismatch between the best possible policy in $\Bcal$ and the expert policy $\pi^E$

%value of $\feedback(s, \pi(\cdot \mid s))$
% in general
% (potentially different from $\Bcal$)
%some policy class
In the above proposition, $\bias(\Bcal,\Bcal_0,N)$ takes the worst-case mixture of state occupancy distributions $\cbr{d_{\upsilon_n}}_{n=1}^N$ induced by $N$ policies from $\Bcal_0$, and measures the expected disagreement between $\pi^E$ and its best  approximating policy in $\Bcal$. Informally, it measures the ``approximation error'' of benchmark class $\Bcal$: it is always nonnegative, and when $\pi^E \in \Bcal$ (which we call the \emph{realizable case)}, $\bias(\Bcal,\Bcal_0,N) = 0$. 
Proposition~\ref{prop:reg-conversion} gives two ways to obtain a competitive imitation policy: (1) choose $(\Bcal,\Bcal_0)$ with a small $\bias(\Bcal,\Bcal_0,N)$ and achieve a low static regret; (2) choose $(\Bcal,\Bcal_0)$ with a small $\bias(\Bcal,\Bcal_0,1)$ and achieve a low dynamic regret. Although achieving low dynamic regret can be significantly more challenging than achieving low static regret, minimizing dynamic regret has the advantage that its approximation error term $\bias(\Bcal,\Bcal_0,1)$ is smaller than the static regret formulation's counterpart $\bias(\Bcal,\Bcal_0,N)$.
%, the bias term in the static regret formulation.

\paragraph{Classification-based Online Imitation Learning (\oilcsc).} As we consider a finite action space $\Acal$, a policy $\pi$ can be equivalently viewed as a (possibly randomized) multiclass classifier. 
A cost-sensitive classification (CSC) example is defined to be a pair $(x,\vec{c})$, where $x \in \Scal$ is its feature part, and $\vec{c} \in \RR^{A}$ is its cost part. 
Under the \DAgger reduction framework, the loss at iteration $n$,  
$\iterloss_n(\pi) = \EE_{s \sim d_{\pi_n}} \EE_{a \sim \pi(\cdot|s)} \sbr{\feedback(s,a)}$, can be viewed as the expected cost of policy $\pi$ on a distribution of cost-sensitive examples $\incrdist_{\pi_n}$ (formally, $\EE_{(s, \vec{c}) \sim \incrdist_{\pi_n}} \sbr{ \vec{c}(h(s)) }$), where a sample $(s,\vec{c})$ is drawn from $\incrdist_\pi$ by first rolling out $\pi$ and drawing $s \sim d_{\pi}$, and query the expert on $s$ to obtain $(\zeta_E(s,a))_{a \in \Acal}$ as its associated $\vec{c}$. 
The learner can obtain a finite-sample approximation to $F_n(\pi)$ by interacting with $\Mcal$ and the expert to draw samples from $\incrdist_{\pi_n}$. 
We will focus on designing efficient regret-minimizing algorithms in the \oilcsc setting; see Protocol~\ref{protocol:ilrm} for a summary.

%by executing policy $\pi_n$ multiple times, and request for expert feedback on a set of states incurred
% (which is equivalently a multiclass classifier)
%first, sample 

%Previous work left open the question of designing sublinear-regret algorithms for general policy classes in this setting.  
%, which can be equivalently viewed as a class of multiclass classifiers. \label{def: baseclass}

In addition to data efficiency, we also consider the design of imitation learning algorithms with computational efficiency guarantees. 
To this end,
following a sequence of empirically and theoretically successful works on oracle-efficient learning~\cite{syrgkanis2016efficient,dudik2011efficient,agarwal2014taming,dann2018oracle},
we assume access to the benchmark policy class $\Bcal$, which is a collection of $B$ stationary deterministic policies $h: \Scal \rightarrow \Acal$. 
% \chicheng{better to put this sentence after the definition?}
Throughout this paper, we assume access to the following computational oracle for class $\Bcal$ and measure an algorithm's computational efficiency by its number of calls to this oracle.

\begin{definition}[CSC oracle]
A CSC oracle $\Ocal$ for policy class $\Bcal$ is such that: given any input multiset of cost-sensitive examples 
% $D = \{(x_1, \vec{c}_1), \ldots, (x_n, \vec{c}_n)\} \in (\Scal \times \RR^{A})^n$
$D = \{(x_1, \vec{c}_1), \ldots, (x_K, \vec{c}_K)\} \in (\Scal \times \RR^{A})^K$, it outputs the policy in $\Bcal$ that has the smallest empirical cost, formally,
\[
\Ocal(D) := \argmin_{h \in \Bcal} \EE_{(x,\vec{c}) \sim D} \sbr{ \vec{c}(h(x)) }, 
\]
where we slightly abuse the notation and
% \edit{of $D$ to denote the uniform distribution on set $D$.}{
and use $D$ to also denote the uniform distribution over it.
% denote $\EE_{D}$ as the empirical average over $D$, i.e. $\EE_{(x,\vec{c}) \sim D} \sbr{ \vec{c}(h(x)) } := \frac{1}{K}\sum_{(x,\vec{c}) \in D}\vec{c}(h(x))$. 
% \chicheng{abuse notations to use $D$ to denote..}
\label{def:oracle}
\end{definition}

\section{\algreduct: reducing \oilcsc to online linear optimization}
\label{sec:improper-learning}

In this section, we introduce our main algorithmic framework, \algreduct(an abbreviation for Linear lOss aGGrEgation) 
% (an abbreviation for online Linear-Optimization-based data aGGrEgation) 
for designing regret-minimizing algorithms in the \oilcsc setting. Section~\ref{sec:proper} motivates our approach by showing that natural proper learning-based approaches fail to achieve sublinear regret in general; Section~\ref{sec:mixture-class} introduces the main idea of performing improper learning using a carefully-defined mixture policy class, via a reduction to online linear optimization.
%specifically 

\subsection{Can we achieve sublinear regret using proper learning?}
\label{sec:proper}

%the benchmark policy class
A natural idea for minimizing regret is {\em proper learning}: at round $n$, the learner chooses some policy $\pi_n$ (possibly at random) from $\Bcal_0 = \Bcal$, our benchmark policy class, using some online learning algorithm, based on all information collected in the first $n-1$ rounds; the learner collects information on $\iterloss_n$ via rollouts of $\pi_n$ and expert annotations, and continue to the next iteration. 

While this approach has demonstrated sharp online regret guarantees in classical online cost-sensitive classification settings~\cite{littlestone1994weighted,freund1997decision}, perhaps subtly, we show in the following theorem that,  this approach is insufficient to guarantee sublinear regret in the \oilcsc setting.

\begin{theorem}
\label{thm:proper-linear-reg}
Suppose the expert's feedback $\zeta_E(s,a)$ is of the form $\mu \cdot I(a \neq \pi^E(s))$ or $A^E(s,a)$. Then, for any $H \geq 3$,
there exists an MDP $\Mcal$ of episode length $H$, a deterministic expert policy $\pi^E$,
% , such that $(\Mcal, \pi^E)$ have a recoverability constant $\mu = 1$,
a benchmark policy class $\Bcal$, such that for any learner that sequentially and possibly at random generates a sequence of policies $\{\pi_n\}_{n=1}^N \in \Bcal^N$, its static regret satisfies $
\SReg_N(\Bcal) = \Omega(N).
$
\end{theorem}
% \chicheng{todo: modify the proof in light of the new theorem statement}

%online imitation learning
%Its key insight is that, 
The proof of the theorem can be found at Appendix~\ref{sec:mixed}.
Its key insight is that,
distinct from the classical online CSC setting, in \oilcsc, the loss at round $n$, $\iterloss_n$, \emph{depends on the policy chosen at that round $\pi_n$}, making standard regret minimization results in online classification~\cite{littlestone1994weighted,freund1997decision} inapplicable. In more detail, we construct MDPs that act ``adversarially'' to policies in $\Bcal$, such that any $\pi_n \in \Bcal$ has $\iterloss_n(\pi_n)\geq \frac{N(H-1)}{H}$, whereas $\min_{\pi \in \Bcal}\sum_n^N\iterloss_n(\pi) \leq \frac{N}{2}$.
Our theorem is similar in spirit to Cover's impossibility result in online classification~\cite{cover1966behavior}, which shows that an adversary that adapts to the randomness of the learner at each round can force the learner to suffer linear regret.

%, and its computational challenges
\subsection{A new hypothesis class and the \algreduct algorithmic framework}
\label{sec:mixture-class}

%The impossibility theorem
To sidestep the impossibility result in Theorem~\ref{thm:proper-linear-reg}, we apply the ``convexification by randomization'' technique in online convex optimization \cite{shalev2011online} by  improper learning on a mixed policy class, defined below.
%Theorem~\ref{thm:proper-linear-reg}  suggests that improper learning is necessary to achieve sublinear regret. This motivates us to consider choosing policies from the following class that contains $\Bcal$:

\begin{definition}[Mixed policy class]
\label{def:mixed-class}
Given policy class $\Bcal$, define its induced mixed policy class 
\[
\Pi_\Bcal := \cbr{\pi_u(\cdot|s) := \sum_{h \in \Bcal} u\coord{h} \cdot h(\cdot|s): u \in \Delta(\Bcal) }
.
\]
We slightly abuse notation and use $h(\cdot \mid s) \in \Delta(\Acal)$ to denote the delta mass on $h(s) \in \Acal$.\footnote{\cite{sun2017deeply}[Theorem 5.3] propose to perform no-regret learning using another definition of nonstationary mixed policy class; we identify an issue with this approach, and defer detailed discussions to Appendix~\ref{sec:mixture-class-deferred}.}
\end{definition}

\begin{algorithm}[t]
\caption{\algreduct: reducing \oilcsc to online linear optimization}
%\algreduct 
% \chicheng{Need a good name for this algorithm that pronounces well - we gotta be creative here. Something like Olobi (Online Linear Optimization-Based Imitation), which is also a city in Nigeria?}}
\begin{algorithmic}[1]
%\REQUIRE
\STATE \textbf{Input:} MDP $\Mcal$, Expert feedback $\feedback$, sample size per iteration $K$, Online linear optimization algorithm $\oloa$ with decision set $\Delta(\Bcal)$.

\FOR{$n=1,2,\ldots,N$}
%\STATE Let $u_n \gets \Acal( \{ \est_i \}_{i=1}^{n-1} )$
\STATE Choose $u_n \gets \oloa ( \{ \est_i \}_{i=1}^{n-1} )$, which induces policy $\pi_n := \pi_{u_n}$.
\label{line:compute_un}

%\STATE Roll out $\pi_n$ in $\Mcal$ for $K$ times, and obtain $\incrstate_n = \cbr{ s_{n,k} }_{k=1}^K$ drawn iid from $d_{\pi_n}$

%\STATE Ask for expert annotations $\feedback$ on $\Ucal_n$, obtaining cost-sensitive examples 

%cost-sensitive 
\STATE Draw $K$ examples
$\incrdata_n = \cbr{ (s, \vec{c} ) }$ iid from $\incrdist_{\pi_n}$, via interaction with $\Mcal$ and expert $\feedback$.
\label{line:ksamples}
%\feedback(s, a)_{a \in \Acal}
%: s \in \incrstate_n
%_{n,k}
%$\cbr{s_{n,k}}_{k=1}^K$, obtaining cost-sensitive examples $\incrdata_n = \cbr{ (s_{n,k}, c_{n,k}) }_{k=1}^K$, where $c_{n,k}(a) = \feedback(s_{n,k}, a)$.

%\feedback(s_{n,k}, h(s_{n,k}))
\STATE $\incrdata_n$ induces $\est_n = \rbr{ \EE_{(s,\vec{c}) \sim \incrdata_n} \EE_{a \sim h(\cdot \mid s)} \sbr{ \vec{c}(a) } }_{h \in \Bcal}$, an unbiased estimate of $\theta(u_n)$.
\label{line:estimator}
%Alternatively, use $\frac1K \sum_{k=1}^K c_{n,k}(h(s_{n,k}))$

% \blue{\STATE Suffer linear loss $u \mapsto \inner{\est_n}{u}$}

%$\linloss_n(u) := \inner{\theta(u_n)}{u}$
\ENDFOR
\end{algorithmic}
\label{alg:reduce-olo}
\end{algorithm}

At a cursory glance, choosing a policy from $\Pi_\Bcal$ seems equivalent to choosing some policy at random from $\Bcal$, which also falls into the failure mode of proper learning (Theorem~\ref{thm:proper-linear-reg}). We remark that this is not true: rolling out a policy from $\pi_u \in \Pi_\Bcal$ is equivalent to drawing new policies in the i.i.d. manner from $\Bcal$ \emph{at every step of the episode} instead. As we will see next, 
using $\Pi_\Bcal$ enables the design of algorithms with sublinear regret.

The key observation is with the learner outputting policies from the mixed policy class $\Pi_\Bcal$, online regret minimization in IL becomes an {\em online linear optimization} problem. Recall that in online IL, the loss at round $n$ is $\iterloss_n(\pi) = \EE_{s \sim d_{\pi_n}} \EE_{a \sim \pi(\cdot \mid s)} \sbr{ \feedback(s,a) }$. By choosing $\pi_n = \pi_{u_n} \in \Pi_\Bcal$ and $\pi = \pi_u\in \Pi_\Bcal$ for $u_n,u \in \Delta(\Bcal)$, $\iterloss_n(\pi_u)$ can be viewed as a linear function of $u$:
\[ 
\begin{aligned}
\iterloss_n(\pi_u) 
%&= 
%\EE_{s \sim d_{\pi_v}} \EE_{a \sim \pi_u(\cdot \mid s)} \sbr{ \feedback(s,a) }\\
%&
=
\sum_{h \in \Bcal} u\coord{h} \cdot \EE_{s \sim d_{\pi_n}} \EE_{a \sim h(\cdot \mid s)}  \sbr{ \feedback(s,a) }
%\\
%&
=
\inner{\theta(u_n)}{u},
\end{aligned}
\]
%\chicheng{use $\coord{}$}
where $\theta(v) := \rbr{ \EE_{s \sim d_{\pi_v}}\EE_{a \sim h(\cdot \mid s)}  \sbr{ \zeta_E(s,a) } }_{h \in \Bcal}$.
% \EE_{a \sim h(\cdot \mid s)} 
We have
\[
\sum_{n=1}^N \inner{\theta(u_n)}{u_n} = \sum_{n=1}^N \iterloss_n(\pi_n), \quad 
\min_{u \in \Delta(\Bcal)} \sum_{n=1}^N \inner{\theta(u_n)}{u} = 
\min_{u \in \Delta(\Bcal)}
\sum_{n=1}^N\iterloss_n(\pi_u) 
= 
\min_{\pi \in \Bcal}
\sum_{n=1}^N\iterloss_n(\pi),
\]
and therefore, minimzing the static regret $\SReg_N(\Bcal)$ is equivalent to minimizing the regret in the online linear optimization problem with losses $\cbr{u \mapsto \inner{\theta(u_n)}{u}}_{n=1}^N$. 

This motivates \algreduct (Algorithm~\ref{alg:reduce-olo}), our main algorithmic framework.
Given input an online linear optimization algorithm \oloa and sample size $K$, \algreduct outputs policy sequence $\{\pi_n\}_{n=1}^N$. Specifically, at round $n$, \algreduct calls \oloa to perform online linear optimization with respect to linear losses $\{ u \mapsto \inner{\est_i}{u}  \}_{i=1}^{n-1}$
%  \oloa with history loss sequence  $\{\est_i\}_{i=1}^{n-1}$
and obtains $u_n \in \Delta(\Bcal)$, which corresponds to a policy $\pi_n \in \Pi_\Bcal$ (line~\ref{line:compute_un}); here for every $i$, $g_i$ is an unbiased estimator of $\theta(u_i)$. 
It then rolls out $\pi_n$ in $\Mcal$ for $K$ times to obtain $K$ samples iid from $d_{\pi_n}$, queries the expert on each sample $s$ to obtain $(\zeta_E(s,a))_{a \in \Acal}$ as its associated $\vec{c}$, and constructs dataset $\incrdata_n =  \cbr{ (s, \vec{c} ) }$ (line~\ref{line:ksamples}). Finally \algreduct computes the empirical average loss on $\incrdata_n$, i.e.  $\est_n = \rbr{ \EE_{(s,\vec{c}) \sim \incrdata_n} \EE_{a \sim h(\cdot \mid s)} \sbr{ \vec{c}(a) } }_{h \in \Bcal}$ (line~\ref{line:estimator}). 

\paragraph{Comparison to prior works.} \cite{cheng2020online} considers a general online convex optimization formulation for online IL, dubbed ``continuous online learning''; our loss function $\inner{\theta(u)}{\cdot}$ can be viewed as a concrete instantiation of the loss function $f_u( \cdot )$ therein. However, their regret minimization results assume that $f_u( \cdot )$ is strongly convex, which do not cover our \oilcsc setting where $f_u( \cdot )$ is linear.

Define $\SLReg_N := \sum_{n=1}^N \inner{\est_n}{u_n} - \min_{u \in \Delta(\Bcal)} \sum_{n=1}^N \inner{\est_n}{u}$ as \oloa's  static regret with respect to $\{ \inner{\est_n}{\cdot} \}_{n=1}^N$. 
We have the following proposition that links $\SLReg_N$ to $\SReg_N(\Bcal)$, the static regret of $\cbr{\pi_n}_{n=1}^N$ in the online IL problem.

%Algorithm~\ref{alg:reduce-olo}
\begin{proposition}
\label{prop:il-olo}
For any $\delta \in (0,1]$, if \algreduct uses some $\oloa$ that outputs $\cbr{u_n}_{n=1}^N \subset \Delta(\Bcal)^{N}$ such that with probability at least  $1-\delta/3$, 
% \edit{$\sum_{n=1}^N \inner{\est_n}{u_n} - \min_{u \in \Delta(\Bcal)} \sum_{n=1}^N \inner{\est_n}{u} \leq \mathrm{Reg}(N)$}{
% its online linear optimization regret 
$\SLReg_N \leq \mathrm{Reg}(N)$. Then, with probability at least $1-\delta$, its output policies $\cbr{\pi_n}_{n=1}^N$ satisfy
$\SReg_N(\Bcal) \leq \mathrm{Reg}(N) + O\rbr{\mu\sqrt{\frac{N\ln(B/\delta)}{K}}}$.
\end{proposition}

%the design of 
Proposition~\ref{prop:il-olo} shows that Algorithm~\ref{alg:reduce-olo} is a regret-preserving reduction from online IL to online linear optimization over $\Delta(\Bcal)$, a $B$-dimensional probability simplex. The latter is well-known as the ``prediction with expert advice'' problem~\cite{freund1997decision} (abbrev. expert problem), where algorithms with different guarantees abound, such as Follow the Regularized Leader (FTRL), Hedge~\cite{freund1997decision} and its adaptive and optimistic variants~\cite{steinhardt2014adaptivity}[Section 1], many of which have optimal worst-case regret bounds $\Reg(N) = O\rbr{ \sqrt{N \ln (B)} }$. Instantiating Algorithm~\ref{alg:reduce-olo} with $\oloa$ set as these algorithms, we obtain a family of online IL algorithms with expected regret of order $O(\sqrt{N})$. 

%\SLReg_N
%Follow-the-Regularized-Leader
%steinhardt2014adaptivity
%(FTRL), Hedge, $p$-norm
%standard algorithms for the expert problem, when applied to online imitation learning
%that is usually dense and
Although satisfying from a statistical efficiency perspective, such online IL algorithms suffers from computational inefficiency: they require explicit calculation of $\est_n$ and maintenance of $u_n$, which are $B$-dimensional (dense) vectors whose entries need to be updated separately. For instance, when Hedge is set as $\oloa$, $u_n\coord{h} \propto \exp(-\eta \sum_{i=1}^{n-1} \est_n\coord{h}))$ for all $h \in \base$, which naively requires $O(B)$ time per round to maintain.   
To address this computational efficiency issue, in the next section, we exploit the cost-sensitive classification nature of the COIL problem 
to design sublinear-regret algorithms that 
% \edit{keep $\est_n$ implicitly in $\incrdata_n$}{
use implicit representations of $\est_n$'s, i.e. $\incrdata_n$'s, that enjoy oracle-efficiency guarantees. 

%the structure of the underlying CSC problem, and

\section{Efficient algorithms with static regret guarantees}
\label{sec:static-reg}

Using the \algreduct framework, in this section, we propose two oracle-efficient \oilcsc algorithms that have sublinear static regret guarantees against policy class $\Bcal$, in Subsections~\ref{sec:reg-sqrt-n} and~\ref{sec:reg-const} respectively.
% , with different sample complexity and interaction round complexity tradeoffs.

\subsection{\algreductm: an efficient algorithm with $O(\sqrt N)$ static regret}
\label{sec:reg-sqrt-n}

%suggests the design 
%in the previous subsection
%, namely Algorithm~\ref{alg:reduce-olo},

%\chicheng{Emphasizing that $g_i$'s are only kept \emph{implicitly} in \algm.}

The \algreduct reduction framework calls for a computationally and statistically efficient $\oloa$, which, if devised, yields an computationally and statistically efficient online imitation learner.
However, when viewed as a general adversarial online learning problem, computational hardness results~\cite{hazan2016computational} suggest that, even with access to classification oracle $\Ocal$, a prohibitive $\Omega(\sqrt{B})$ time complexity is necessary for sublinear regret. Therefore, in subsequent sections, we adopt an assumption on $\Bcal$ in~\cite{syrgkanis2016efficient}, which, to the best of our knowledge, is the state-of-the-art weakest assumption that allows the design of oracle-efficient online CSC algorithms in the adversarial setting:  

% first proposed by
%As a running time polynomial in $B$ is prohibitive in practical applications,
%In view of this observation,
\begin{assumption}[Small separator set]
There exists a set $\Xcal \subset \Scal$ (called the separator set) such that, 
for every pair of distinct policies $h, h' \in \Bcal$, $\exists x \in \Xcal$, such that $h(x) \neq h'(x)$. Denote by $X := \abr{ \Xcal }$.
\label{def:small-sep}
\end{assumption}

\paragraph{Technical challenges.} Even under the small separator set assumption, the design of low-regret oracle-efficient algorithms for imitation learning still remains nontrivial. 
A naive application of existing oracle-efficient online CSC algorithms, such as Contextual Follow the Perturbed Leader (CFTPL)~\cite{syrgkanis2016efficient,dudik2020oracle}, still falls into the failure mode of proper learning (Theorem~\ref{thm:proper-linear-reg} in Section~\ref{sec:proper}), where an $\Omega(N)$ regret lower bound is unavoidable in the worst case.
This is in sharp contrast to the classical online CSC setting, where CFTPL enjoys a $O(\sqrt{N})$ regret~\cite{syrgkanis2016efficient,dudik2020oracle}. 
To recap, at round $n$, CFTPL first constructs a random set of ``hallucinated'' cost-sensitive examples $Z$ based on the separator set $\Xcal$; it subsequently calls the CSC oracle $\Ocal$ on the union of $Z$ and the accumulated dataset $\cup_{i=1}^{n-1} \incrdata_i$ to obtain policy $\pi_n \in \Bcal$. 
CFTPL achieves computational efficiency by
% \edit{keeping $\est_n$ implicitly in $\incrdata_n$}{
operating on $D_n$'s, an implicit representation of $\est_n$'s, the linear losses of the underlying OLO problem.

\paragraph{Our approach.} The above difficulty motivates the need of a new algorithmic approach for efficient classification-based online IL. In view of Section~\ref{sec:mixture-class}'s observation that FTRL approaches enjoy a sublinear regret, we ask the question: is it possible to perform FTRL in an oracle-efficient manner? A positive answer will simultaneously address the computational and statistical challenges of \oilcsc. 
% \chicheng{Just for our note: we should look at Akshay's reference ``Disagreement-Based Combinatorial Pure Exploration: Sample Complexity Bounds and an Efficient Algorithm'' which he also used this idea. Perhaps discuss it in the next round of revision.}

%to this question
%in an manner
%Fortunately, this answer has been answered in 
%$\Phi$ 
% distribution of the noisy fake cost vectors 
% applied to each policy
We answer this question in the affirmative, by utilizing a connection between FTRL and FTPL first observed in~\cite{abernethy2014online}: an in-expectation version of FTPL can be viewed as an FTRL algorithm. Using this observation, we design Algorithm~\ref{alg:mftpl}, namely Mixed CFTPL (abbrev. \algm), which mimics FTRL by approximating the in-expectation version of CFTPL in an oracle-efficient manner. Similar to CFTPL, \algm keeps $\est_n$ implicitly in $\incrdata_n$, and calls the CSC oracle. Different from CFTPL,   \algm runs the oracle-call step in CFTPL for $\Sp$ times and outputs the uniform mixture of the $\Sp$ policies. We refer to $\Sp$ as Algorithm~\ref{alg:mftpl}'s {\em sparsification parameter}, due to the algorithm's resemblance to Maurey's sparsification~\cite{pisier1981remarques}. 

%\footnote{Due to the resemblance between approximating $u_n^*$ using $u_n$ and Maurey's sparsification~\cite{pisier1981remarques},.}.

Specifically, at each iteration $j \in [T]$, \algm first draws $(\ell_{x,j}(a))_{a \in \Acal} \sim \mathcal{N}(0,I_A)$ iid for each $x $ in the separator set $ \Xcal$, where $I_A$ is the identity matrix of dimension $A$ (line~\ref{line:drawperturbation}). It then constructs a perturbation set of cost-sensitive examples $Z_{j} = \cbr{ (x, \frac{K}{\eta} \ell_{x,j} ): x \in \Xcal }$ that contains each $x$ within the separator set and $ \frac{K}{\eta} \ell_{x,j}$ as associated $\vec{c}$, where $K$ accounts for the adjustment on dataset size and $\eta$ accounts for FTRL's learning rate (line~\ref{line:perturbset}). By calling the oracle $\Ocal$ with the so far accumulated datasets $\cup_{i=1}^{n-1} \incrdata_i$ together with perturbation set $Z_j$, the $h\in \Bcal$ that achieves the smallest empirical cost on $ (\cup_{i=1}^{n-1} \incrdata_i) \cup Z_j$ is returned from the oracle and represented by $u_{n,j}$, which is a one-hot vector in $\Delta(\Bcal)$ that has weight $1$ on the $h$ returned and $0$ elsewhere (line~\ref{line:calloracle}). Finally, after $\Sp$ iterations, \algm returns the mean value  $u_n =\frac1\Sp \sum_{j=1}^{\Sp} u_{n,j}$ (line~\ref{line:average}). MFTPL  guarantees that:

\begin{algorithm}[t]
\caption{\algm: an oracle-efficient approximation of FTRL}
\begin{algorithmic}[1]
\STATE \textbf{Input:} Linear losses $\{\est_i\}_{i=1}^{n-1}$
represented by datasets $\cbr{\incrdata_i}_{i=1}^{n-1}$
each of size $K$
(s.t. $\est_i\coord{h} = \EE_{(s,\vec{c}) \sim \incrdata_i}\sbr{\vec{c}(h(s))}$ for all $h \in \Bcal$) 
, separator set $\Xcal$, learning rate $\eta$, sparsification parameter $\Sp$.
% \chicheng{Make a similar change for MFTPL-EG}
%\red{sample size per iteration $K$.}
%, cost-sensitive classification oracle $\Ocal$.
\FOR{$j=1,2,\ldots,\Sp$}
\STATE Draw $(\ell_{x,j}(a))_{a \in \Acal} \sim \mathcal{N}(0,I_A)$ iid for each $x \in \Xcal$.
\label{line:drawperturbation}
% \chicheng{Many $i$'s here need to be changed to $j$ - Yichen, can you make a pass?}
\STATE Define $Z_{j} = \cbr{ (x, \frac{K}{\eta} \ell_{x,j} ): x \in \Xcal }$. \label{line:perturbset}
\STATE Compute $u_{n,j} \gets \Onehot(\Ocal( (\cup_{i=1}^{n-1} \incrdata_i) \cup Z_j), \Bcal)$. \label{line:calloracle}
\ENDFOR
\RETURN $u_n \gets \frac1\Sp \sum_{j=1}^{\Sp} u_{n,j}$. \label{line:average}
\end{algorithmic}
\label{alg:mftpl}
\end{algorithm}

%(Algorithm~\ref{\ref{alg:ecftpl}})
\begin{lemma}
\label{lem: ftpl_approximation-main}
There exists some strongly convex function $R: \Delta(\Bcal) \to \RR$, such that the following holds.
Suppose \algm receives datasets $\cbr{\incrdata_i}_{i=1}^{n-1}$, separator set $\Xcal$, learning rate $\eta$, sparsification parameter $\Sp$. Then, $\forall \delta \in (0,1]$, with probability at least $1-\delta$, \algm makes $\Sp$ calls to the cost-sensitive oracle $\Ocal$, and 
outputs $u_n \in \Delta(\Bcal)$ such that 
$$
\forall s \in \Scal,
\| \pi_{u_n}(\cdot|s) -  \pi_{u^*_n}(\cdot|s) \|_1 \leq \sqrt{\frac{2A \rbr{ \ln(S)+\ln(\frac{2}{\delta} )} }{\Sp}},
$$
with $u_n^* := \argmin_{u \in \Delta(\Bcal)} \rbr{  \langle{\eta \sum_{i=1}^{n-1}\est_i},{u} \rangle + R(u) }$.
\end{lemma}

Therefore,  by setting $\Sp=\Omega(A \ln(S))$, the policy $\pi_n$ induced by the MFTPL's output $u_n$ closely mimics $\pi_{u_n^*}$, a policy induced by the FTRL output $u_n^*$. The mild $\ln(\cdot)$ dependence on $S$ makes the lemma useful in large-state-space settings.
Specifically,
a naive attempt to show Lemma~\ref{lem: ftpl_approximation-main} is to establish the $\|\cdot\|_1$ closeness of $u_n$ and $u_n^*$, given that $T\cdot u_n \sim \multinomial( T, u_n^* )$. This unavoidably carries an impractical concentration factor of $O(\sqrt{B / \Sp})$, as the bound requires $\Sp$ to be $\Omega(B)$ to be non-vacuous. We get around this challenge by directly showing the closeness of the action distributions $\pi_{u_n}$ and $\pi_{u_n^*}$ for all states. We defer the full version of the lemma, including an explicit form of $R$, to Appendix~\ref{sec:reg-sqrt-n-deferred}.

\begin{lemma}
\label{lem:mftpl-reg}
For any $\delta \in (0,1]$, \algm, if called for $N$ rounds, with input
learning rate $\eta =  \frac{1}{\mu \sqrt{NA} }  \rbr{ \frac{\ln (B)}{X} }^{\frac14}$ and sparsification parameter $\Sp = \frac{ N \ln(2NS/\delta)}{\sqrt{X^3 \ln(B)}}$,
outputs a sequence $\cbr{u_n}_{n=1}^N$, such that with probability at least    $1-\delta$:
\[
%R(N )
%=
\SLReg_N
% \sum_{n=1}^N 
% %\linloss_n(u_n)
% \langle {\est_n}, {u_n} \rangle 
% - 
% \min_{u \in \Delta(\Bcal)} \sum_{n=1}^N %\linloss_n(u)
% \langle {\est_n}, {u} \rangle 
\leq 
   O\rbr{ \mu\sqrt{ NA}\rbr{X^3\ln(B)}^{\frac{1}{4}} }.
\]
% \chicheng{Should be $\tilde{O}$ - here $\tilde{O}$ hides dependence on $\ln\frac1\delta$}
\end{lemma}

% \chicheng{Now that we have the general \algreduct framework , I propose to use the following theorem statement to replace the theorem statement below. The theorem statement below is very good (and the proofs are very useful), but it is hard to explain, so we need to re-organize it a bit here. Yichen, can you move Theorem~\ref{thm: ftrl-noextragrad} along with its proof to the corresponding ``deferred material'' section in the appendix?}
% as the underlying $\oloa$ algorithm
%, termed as

%, that has both oracle efficiency and regret efficiency guarantees
Composing \algreduct with \algm, we obtain an efficient online IL algorithm,  \algreductm. 
Its regret guarantees immediately follow from  combining Lemma~\ref{lem:mftpl-reg} with Proposition~\ref{prop:il-olo}:

%general 
%the online linear optimization guarantee of \algm,
%as the underlying online linear optimization algorithm,
 
\begin{theorem}
\label{thm:logger-m-main}
For any $\delta \in (0,1]$, \algreductm, with $K=1$ and $\algm$ setting its parameters as in Lemma~\ref{lem:mftpl-reg}, is such that: 
(1) with probability at least $1-\delta$, its output $\cbr{\pi_n}_{n=1}^N$ satisfies:
% it outputs a sequence of policies $\cbr{\pi_n}_{n=1}^N$ \red{that satisfies}
% whose static regret satisfies that 
$\SReg_N(\Bcal) \leq O\rbr{  \mu \sqrt{  NA  \ln(1/ \delta)} (X^3 \ln (B))^{\frac14} }$; 
(2) it queries $N$  annotations from the expert;
(3) it calls the CSC oracle $\Ocal$ for $\frac{ N^2 \ln(6NS/\delta)}{\sqrt{X^3 \ln (B)}}$ times.
\end{theorem}

\subsection{\algreductmp: an efficient algorithm with $O\rbr{ 1 }$ static regret}
\label{sec:reg-const}

%Algorithm~\ref{alg:ftrl-noextragrad}
Although $\algreductm$ is  oracle-efficient, it is unclear whether its $O(\sqrt{N})$ regret guarantee is optimal.
As a lower regret can translate to lower sample and interaction round complexity guarantees, it is of importance to design algorithms with regret as low as possible.
%as low order regret as possible

%, in the sense of
%not a complete adversarial online learning problem, but rather
%, which most often occurs
% in  scenarios
A key observation from prior works~\cite{cheng2018convergence,cheng2020online,lee2021dynamic} is that, online IL is a \emph{predictable} online learning problem~\cite{chiang2012online,rakhlin2013online}, and is thus not completely adversarial. 
This opens up possibilities to bypass the $O(\sqrt{N})$ worst-case regret barrier.  
Specifically, in the \algreduct framework, the coefficient of the linear loss at round $n$, $\theta(u_n)$, depends \emph{continuously} on $u_n$; more concretely, we can show:

%The following lemma formalizes this intuition in the context of online imitation learning with cost-sensitive classification.

\begin{lemma}
\label{lem:dist-cont}
For $u,v \in \Delta(\Bcal)$,
$
\| \theta(u) - \theta(v) \|_\infty 
%\substack{ *1\\ \leq }
\leq
\mu H \cdot \underset{s \in \Scal}{\max} {\| \pi_{u}(\cdot|s) -  \pi_{v}(\cdot|s) \|_1}
\leq 
%\substack{ *2\\ \leq }
\mu H \| u - v \|_1
$.
\end{lemma}
%~\cite{cheng2018convergence,cheng2019accelerating,cheng2019predictor,cheng2020online}
This property and its variants, termed {\em distributional continuity}, has been utilized in many online IL algorithms to achieve sharper regret guarantees. These works additionally exploit the strong convexity on the loss functions $\iterloss_n(\pi)$~\cite{cheng2018convergence,cheng2020online,lee2021dynamic}, or use some external predictive model that can predict $\nabla \iterloss_n(\pi)$ well~\cite{cheng2019accelerating,cheng2019predictor}. Unfortunately, in our \oilcsc setting, neither is the loss function $F_n(\pi_u) = \langle \theta(u_n), u \rangle$ strongly convex in the policy parameter $u$, nor do we have access to an external predictive model, rendering these approaches inapplicable. 

We get around these challenges and design an oracle-efficient algorithm, namely \algmp (where EG stands for extra-gradient), with $O(1)$ regret for the online linear optimization problem, which, when composed with the \algreduct framework, yields the \algreductmp algorithm with $O(1)$ regret in the \oilcsc setting.
\algmp is largely inspired by the predictor-corrector framework for policy optimization~\cite{cheng2019predictor} and extragradient methods in smooth optimization~\cite{nemirovski2004prox,juditsky2011solving}; its details can be found in Appendix~\ref{sec:reg-const-deferred}.
Its key insight is that, although we do not have a predictive model for $\theta(u_n)$, we can use an extra round of interaction with $\Mcal$ and expert annotations to obtain a good estimate of it. Based on this, we derive an online linear optimization regret guarantee of \algmp, deferred to Appendix~\ref{sec:reg-const-deferred}. This immediately implies the following guarantee of \algreductmp:

\begin{theorem}
\label{thm:logger-me-main}
For any $\delta \in (0,1]$, \algreductmp, with $K$ and \algmp's parameters set appropriately, is such that: (1) with probability at least $1-\delta$, its output $\cbr{\pi_n}_{n=1}^N$ satisfies:
$\SReg_N(\Bcal)
\leq 
 O(\mu H A \sqrt{X^3\ln(B)})$;
(2) it queries $O \rbr {\frac{N^2 \ln(NB/\delta)}{H^2A\sqrt{X^3\ln(B)}}}$ annotations from the expert; 
(3) it calls the CSC oracle $\Ocal$ for $O \rbr{  \frac{ N^3\ln(NS/\delta)}{\mu H AX^3\ln(B)} }$ times.
% Specifically,\algreductmp achieves $\frac{\SReg_N(\Bcal)}{N}<\epsilon$  with probability $1-\delta$ in
% $N =  O(\frac{H^2 A\sqrt{X^3\ln(B)}}{\epsilon})$ rounds, with $\tilde{O}\rbr{\frac{H^2 A\sqrt{X^3\ln(B)}}{\epsilon^2}\ln(\frac{B}{\delta}) }$ annotations and $\tilde{O}\rbr{\frac{A^2 H^4\sqrt{X^3\ln(B)}}{\epsilon^3}\ln(\frac{S}{\delta}) }$ oracle calls.
\end{theorem}

%%\blue{a few baselines}.
%Specifically, 

\paragraph{Discussion and comparison.} We now compare the guarantees of \algreductm, \algreductmp, and the baseline of behavior cloning, where the learner simply draws iid examples from $\incrdist_{\pi^E}$ and perform empirical risk minimization over $\Bcal$ to learn a policy $\hat{\pi}$. 
All algorithms' output policy suboptimality guarantees have the following decomposition:
\begin{equation}
\EE\sbr{ J(\hat{\pi}) - J(\pi^E) }
\leq 
\mathrm{ApproxErr} + \mathrm{EstimErr}, 
\label{eqn:error-decomp}
\end{equation}
%(\Bcal)
% $\mathrm{bias}$ 
%$\mathrm{complexity}$
where $\mathrm{ApproxErr}$ measures the approximation error of the policy class $\base$ to the expert policy $\pi^E$, and $\mathrm{EstimErr}$ is an estimation error term that vanishes with the number of expert annotation examples and iterations increasing. By Proposition~\ref{prop:reg-conversion},  
for \algreductm and \algreductmp, their $\mathrm{ApproxErr}$ terms are both $\mu H \cdot \bias(\Bcal,\Pi_\Bcal,N) $. 
For their $\mathrm{EstimErr} = \frac{H \cdot \SReg_N(\Bcal)}{N}$, we use Theorems~\ref{thm:logger-m-main} and~\ref{thm:logger-me-main} to calculate the minimum total numbers of interaction rounds $I(\epsilon)$, expert annotations $A(\epsilon)$, and oracle calls $C(\epsilon)$, 
so that $\mathrm{EstimErr}$ is at most $\epsilon$ with probability at least $1-\delta$.

 As presented in Table~\ref{tab:comparison}, \algreductmp has the same sample complexity order as \algreductm, but obtains a much lower interaction round complexity ($\mu H^2 / \epsilon$ vs. $\mu^2 H^2 / \epsilon^2$).

%, in terms of $\mu, H$ and $\epsilon$. 

On the other hand, by standard ERM analysis~\cite{shalev2014understanding} and conversion from supervised learning to imitation learning guarantees (\cite{syed2010reduction,Kakade2002ApproximatelyOA}), behavior cloning on $\Bcal$ using $K$ samples outputs a policy $\hat{\pi}$, such that
Equation~\eqref{eqn:error-decomp} holds with 
$\mathrm{ApproxErr} = H^2 \cdot \bias(\Bcal,\{\pi^E\},1)$, and $\mathrm{EstimErr} = H^2 \sqrt{2\ln( 2B/\delta) / K}$ (see Appendix~\ref{sec:comparison-bc} for a detailed derivation), where $\bias(\Bcal,\{\pi^E\},1) =  \min_{h \in \Bcal} \EE_{s \sim d_{\pi^E}} \sbr{I(h(s) \neq \pi^E(s))}$.
% , the smallest disagreement rate between policies in $\Bcal$ and $\pi^E$ on expert's state occupancy distribution.
We also summarize behavior cloning's performance guarantees in Table~\ref{tab:comparison}. 
Compared with the two interactive IL algorithms above, behavior cloning requires only one interaction round and one call to the oracle $\Ocal$, however its $\approxerr$ has a larger coefficient on the optimal classification loss ($H^2$ vs. $\mu H$), and needs more expert annotations ($H^4 / \epsilon^2$ vs. $H^2 \mu^2 / \epsilon^2$) to achieve approximation error smaller than $\epsilon$ with probability at least $1-\delta$.

\begin{table}[]
\renewcommand\arraystretch{1.5}
%  \caption{A comparison between our algorithms and the behavior cloning baseline}
    \centering
    \begin{tabular}{ccccc}
     \toprule
       Algorithm  &  $\mathrm{ApproxErr}$ & $I(\epsilon)$ & $A(\epsilon)$ & $C(\epsilon)$
       \\
      \midrule
        \algreductm & $\mu H \cdot \bias(\Bcal,\Pi_\Bcal,N) $ &
        %A\sqrt{X^3\ln(B)}
        $\tilde{O}(\frac{\mu^2 H^2}{\epsilon^2})$ & $\tilde{O}(\frac{\mu^2H^2}{\epsilon^2} )$ & $\tilde{O}(\frac{\mu^4H^4}{\epsilon^4} )$
        \\
%         \hline
        \algreductmp & $\mu H\cdot \bias(\Bcal,\Pi_\Bcal,N)$ & $\tilde{O}(\frac{\mu H^2}{\epsilon})$ & $\tilde{O}(\frac{\mu^2 H^2 }{\epsilon^2} )$ &  $\tilde{O}(\frac{\mu^2 H^5 }{\epsilon^3} )$
        \\
%         \hline 
        Behavior cloning & $H^2 \cdot \bias(\Bcal,\{\pi^E\},1)$ & 1 & $\tilde{O} (\frac{H^4}{\epsilon^2} )$ & 1
        \\
        \bottomrule
        \\
    \end{tabular}
    % \yichen{actually we are missing $H$s and we have to replace current $\epsilon$ by $\frac{\epsilon}{H}$. }
    % \chicheng{todo: double check this table - the dependence on $\mu$ and $H$ is a bit suspicious}
    \\
    \caption{A comparison between our algorithms and the behavior cloning baseline, in terms of approximation error, and numbers of interaction rounds $I(\epsilon)$, expert annotations $A(\epsilon)$, oracle calls $C(\epsilon)$ needed for estimation error to be at most $\epsilon$ with probability $1-\delta$.
    Here $\tilde{O}(\cdot)$ hides dependences on $\ln(\mu H/\epsilon), X, A, \ln( S)$, $\ln (B), \ln (\mu),  \ln(1/\delta)$.
    See Appendix~\ref{sec:comparison-bc} for the full version of the table. }
    \label{tab:comparison}
\end{table}

\section{Computational hardness of sublinear dynamic regret guarantees}
\label{sec:dynamic-reg}

%in online IL via cost-sensitive classification
%turn to the
% now
%~\cite{cheng2018convergence,lee2021dynamic,cheng2020online}
Finally, we study dynamic regret minimization for \oilcsc in the \algreduct framework. Although in the abstract continuous online learning setup, dynamic regret minimization has recently been shown to be  computationally hard~\cite{cheng2020online}, given the peculiar linear loss structure of the \algreduct framework,
the computational tractability of dynamic regret minimization within this framework still remains open.

We fill this gap by showing that, under a standard complexity-theoretic assumption (that $\mathrm{PPAD}$-complete problems do not admit randomized polynomial-time algorithms), there do not exist polynomial-time algorithms that achieve sublinear dynamic regret in \oilcsc. Specifically, we have:

\begin{theorem}
\label{thm:hardness-dreg}
Fix $\const>0$, if there exist a \oilcsc algorithm such that for any $\Mcal$ and expert $\pi^E$, it interacts with $\Mcal$, CSC oracle $\Ocal$, expert feedback $\feedback(s,a) = A^E(s,a)$, and outputs a sequence of $\cbr{ \pi_n }_{n=1}^N \in \Pi_\Bcal^N$ s.t. with probability at least $1/2$,
\[
\DReg_N(\Bcal)
% ( \cbr{ \pi_n }_{n=1}^N )
\leq
O(\poly(S, A, B) \cdot N^{1-\const}),
\]
in $\poly(N, S, A, B)$ time, then all problems in $\mathrm{PPAD}$ are solvable in randomized polynomial time.
\end{theorem}

The key insight behind our proof of Theorem~\ref{thm:hardness-dreg} is that,  achieving a sublinear dynamic regret in the \oilcsc setup is at least as hard as finding an approximate Nash equilibrium in a two-player general-sum game, a well-known $\mathrm{PPAD}$-complete problem~\cite{chen2006settling}. To establish a reduction from a two-player general-sum game to a \oilcsc problem, we carefully construct a tree-structured MDP whose leave states can be divided to two major groups: one group has costs encoding the two players' payoffs, and the other group has a large constant cost, ensuring that any policy in $\Pi_\Bcal$ with small dynamic regret encodes near-optimal strategies of both players. We refer the readers to Appendix~\ref{sec:dynamic-reg-proof} for details.

\section{Conclusion}
In this work, we investigate the fundamental statistical and computational limits of classification-based online imitation learning (\oilcsc).
On the positive side, we propose the \algreduct framework that enables the design of oracle and regret efficient \oilcsc algorithms with different sample and interaction round complexity tradeoffs, outperforming the behavior cloning baseline. 
On the negative side, we establish impossibility results for sublinear static regret using proper learning in the \oilcsc setting, a subtle but important observation overlooked by prior works. We also show the computational hardness of sublinear dynamic regret guarantees in the \algreduct framework.

Looking forward, it would be interesting to investigate the optimality of our sample complexity and interaction round complexity guarantees; we 
% \edit{are also eager to see the possibility of relaxing}{
also speculate that it is possible to relax the small separator set assumption 
% \edit{for broader applications}{
on $\Bcal$ by utilizing very recent results on smoothed online learning~\cite{block2022smoothed,haghtalab2022oracle}.
Finally, we are also interested in empirically evaluating our algorithms.

\paragraph{Acknowledgments and Disclosure of Funding.} 
We thank the anonymous reviewers for their constructive comments.
We thank Kwang-Sung Jun, Ryn Gray, Jason Pacheco, and members of the University of Arizona machine learning reading group for helpful discussions. We thank Wen Sun for helpful communications regarding~\cite{sun2017deeply}[Theorem 5.3] and pointing us to an updated version~\cite{ajks20}[Section 15.5]. We thank Weijing Wang for helping with illustrative figures. This work is supported by a startup funding by the University of Arizona.

%\yichen{To do: Add disclosure of funding. }
% \chicheng{Yichen, feel free to add other acknowledgements.}

%It is expected that our research will open up a new area of research on online 

%%%%%%%%%%%%%%%%%%%%%%%%%%%%%%%%%%%%%%%%%%%%%%%%%%%%%%%%%%%%

%\bibliographystyle{plain}
%\bibliography{reference} 

\newpage

\appendix
\addcontentsline{toc}{section}{Appendix} % Add the appendix text to the document TOC
\part{Appendix} % Start the appendix part
\parttoc

\section{Additional related works}
\label{sec:add-rel-work}

\subsection{IL via reduction to offline learning}

An algorithm is said to reduce IL to offline learning, if it interacts with the MDP and the expert to create a series of offline learning tasks, and outputs a policy whose suboptimality depends on the quality of solving the offline learning tasks.
One representative example is the Behavior Cloning algorithm, where the learner learns a policy by performing offline supervised learning on a dataset drawn from the expert's state-action occupancy distribution. By \cite{ross2010efficient}[Theorem 2.1] (see also earlier work of~\cite{syed2010reduction}), Behavior Cloning's output policy's suboptimality is bounded by $H^2$ times the classification loss with respect to the expert's state-action occupancy distribution. 
Another example is the Forward Training algorithm of~\cite{ross2010efficient}, where a non-stationary policy is trained incrementally. For every step of the MDP, it trains a policy by performing offline supervised learning on a dataset drawn from the state occupancy distribution at this step induced by the nonstationary policy trained for all previous steps. The output policy's suboptimality of Forward Training is bounded in terms of the averaged 0-1 loss of intermediate offline classification problems at all steps.
The same paper~\cite{ross2010efficient} also proposed the SMILe algorithm, where the learned stationary policy is defined as a mixture of policies trained in the past, as well as the expert policy whose weight diminishes in the number of learning rounds. 
At each round, the learner trains a new policy component under the state distribution induced by the learned policy and uses it to update the learned policy.
In the worst case, the output policy's suboptimality guarantee is bounded in terms of the weighted average of 0-1 losses of intermediate offline classification problems. 
As made explicit by~\cite{ross2010efficient}, the SEARN~\cite{daume2009search} algorithm can be applied to imitation learning and its suboptimality guarantee can be bounded in terms of the averaged 0-1 loss of intermediate offline classification problems. Later, \cite{le2016smooth} extends SEARN to continuous-action regime under the setting of exogenous input, following the similar reduction as  in~\cite{daume2009search}.

\paragraph{IL via reduction to offline surrogate loss minimization.} A few works study IL via offline learning that performs minimization over surrogate losses of 0-1 loss. ~\cite{xu2020error},~\cite{ajks20}[Theorem 15.3] show that if the average KL divergence between a policy's action distribution and the expert's action distribution is bounded, the policy's suboptimality can in turn be bounded. In addition, under a realizable setting, maximum likelihood estimation (log loss minimization) ensures that the above KL divergence goes to zero as the training sample size grows to infinity. 
~\cite{xu2020error} also shows policy suboptimality bounds of generative adversarial imitation learning~\cite{ho2016generative} that depends on the approximation power of the policy class, and an estimation error term that depends on the sample size and the expressivity of the discriminator class.

\subsection{IL via reduction to online learning}
    
%\paragraph{IL via reduction to online learning:}
%\paragraph{IL via reduction to no-regret online classification:} 
%%Dagger~\cite{ross2011reduction} and friends~\cite{ross2014reinforcement}. 
%By this reduction, the policy's suboptimality of Dagger is bounded by the online classification loss.
A major line of research  \cite{ross2010efficient,ross2011reduction,ross2014reinforcement,sun2017deeply,cheng2018convergence,cheng2019accelerating,cheng2019predictor} reduces interactive IL to an online learning problem, where a sequence of online losses are carefully constructed so that the cumulative online loss of a sequence of policies corresponds to the policy sequence's cumulative imitation losses.
In the discrete action setting, where policies can be viewed as classifiers, early works such as \DAgger~\cite{ross2010efficient} do not directly provide an explicit algorithm for online cost-sensitive classification loss regret minimization, and instead perform regret minimization over convex surrogates of the classification losses.
The convex surrogate minimization approach is well-known to be statistically inconsistent even in supervised learning, a special case of imitation learning~\cite{ben2012minimizing}.
Subsequent works~\cite{ross2014reinforcement} reduces online cost-sensitive classification in imitation learning to online least squares regression.
% ; however this approach requires a realizability assumption on the expert's action-value function.} \chicheng{Yichen can you help check whether I am interpreting their results correctly?}

In contrast to these works, we study the original regret minimization problem (induced by CSC losses) in online IL without relaxations, in the nonrealizable setting. Although Sun et al. \cite{sun2017deeply}[Theorem 5.2] implicitly  designs \oilcsc algorithms
for general policy classes
by performing online linear optimization over the convex hull of benchmark policies, we identify a subtle technical issue in their approach; we discuss it in detail in Section~\ref{sec:mixture-class-deferred}.

\paragraph{Online IL as predictable online learning:}
In the above reduction from interactive IL to online learning,
a key observation from prior works~\cite{cheng2018convergence,cheng2020online,lee2021dynamic} is that, online IL is a \emph{predictable} online learning problem~\cite{chiang2012online,rakhlin2013online}.
This observation has enabled the design of more sample efficient~\cite{cheng2020online,lee2021dynamic,cheng2019accelerating,cheng2019predictor} and convergent~\cite{cheng2018convergence} imitation learning algorithms.
However, these works either assume access to an external predictive model~\cite{cheng2019accelerating,cheng2019predictor} or assume strong convexity of the losses~\cite{cheng2020online,lee2021dynamic,cheng2018convergence}, neither of which is satisfied in the \oilcsc setting.
Our \algmp algorithm achieves $O(1)$ static regret without the strong convexity assumption on the losses, and is largely inspired by the predictor-corrector framework of  \cite{cheng2019predictor}, the Mirror-Prox algorithm and extragradient methods in smooth optimization~\cite{nemirovski2004prox,juditsky2011solving}. 

\paragraph{IL with dynamic regret guarantees:}
Prior works in imitation learning \cite{lee2021dynamic, cheng2020online} have designed algorithms that achieve sublinear dynamic regret, under the assumption that the imitation losses are strongly convex in the policy parameters. While strong convexity of the losses naturally occurs in 
settings such as continuous control (e.g. square losses), in our \oilcsc setting with mixed policies, the learner's loss functions do not have strong convexity.

Cheng et al.~\cite{cheng2020online} show that, under an abstract continuous online learning setup, dynamic regret minimization is $\mathrm{PPAD}$-hard, by a reduction from Brouwer's fixed point problem. 
Our computational hardness result Theorem~\ref{thm:hardness-dreg} can be viewed as a strengthening of theirs, in that we show a concrete dynamic regret minimization problem induced by imitation learning in MDPs is $\mathrm{PPAD}$-hard. Our reduction is also significantly different from~\cite{cheng2020online}'s, in that it reduces from the 2-player mixed Nash equilibrium problem: specifically, the reduction constructs a 3-layer MDP based on the payoff matrices of the two players.

In the general online learning setting, \cite{besbes2015non, jadbabaie2015online} design efficient gradient-based algorithms with sublinear dynamic regret guarantees, under the assumption that the sequence of online loss functions have bounded variations. While these results appear to be promising for designing efficient sublinear dynamic regret algorithms in \oilcsc settings, our computational hardness result strongly suggests that additional structural assumptions on the \oilcsc problem are necessary for such guarantees.

\subsection{Other important related works}

\paragraph{Fundamental sample complexity limits of IL:}
Recent works of \cite{rajaraman2020toward,rajaraman2021value}
study minimax sample complexities of realizable imitation learning in the tabular or linear policy class settings, and shows that in general, allowing the learning agent to interact with the environment does not improve the minimax sample complexity. In contrast, in settings where the (MDP, expert policy) pair has a low recoverability constant, interactivity helps reduce the minimax sample complexity.
They also show that knowing the transition probability of the MDP helps reduce the minimax sample complexity. 
Different from their work, our work focuses on the general function approximation setting without realizability assumptions.

\paragraph{Oracle-efficient online and imitation learning:} A line of works \cite{dudik2011efficient,agarwal2014taming,syrgkanis2016efficient,dudik2020oracle,rakhlin2016bistro,syrgkanis2016improved} design oracle-efficient online learning algorithms for online classification and online contextual bandit learning. 
Most of these works either assume that the context distributions are iid, or the contexts are observed ahead of time (i.e. the transductive setting), which 
are inapplicable in the \oilcsc setting.
The only exceptions we are aware of are~\cite{syrgkanis2016efficient,dudik2020oracle}, which utilize a small separator set assumption of the benchmark policy class. However, as we have seen, a direct application of~\cite{syrgkanis2016efficient,dudik2020oracle} to the online IL setting results in a linear regret (Theorem~\ref{thm:proper-linear-reg}), which motivates our design of \algm and \algmp algorithms.  
\cite{sun2019provably} designs oracle-efficient imitation learning algorithms from experts' state observations alone (without seeing experts' actions). Different from ours, their work makes a (strong) realizability assumption: the learner is given access to a policy class and a value function class, that contain the expert's policy and value function, respectively. Also, their algorithm requires regularized CSC oracle, for running FTRL. 
%, which is essentially running FTRL.

\paragraph{Connections between FTRL and FTPL:}
In online linear optimization, \cite{abernethy2014online} first observe that an in-expectation version of FTPL is equivalent to FTRL, where the regularizer depends on the distribution of the noise perturbation.
This viewpoint yields a productive line of work that designs new 
bandit and online learning algorithms~\cite{abernethy2015fighting,abernethy2019online}. 
Our work utilizes this connection to design oracle-efficient online imitation learning algorithms with static regret guarantees.

\section{Recap of notations and additional notations used in the proofs}
We provide a brief recap of the notations introduced outside Section~\ref{sec:prelims}. 

In Section~\ref{sec:improper-learning}, we introduce mixed policy class $\Pi_\Bcal := \cbr{\pi_u(\cdot|s) := \sum_{h \in \Bcal} u\coord{h} \cdot h(\cdot|s): u \in \Delta(\Bcal) }$, and cost vector $\theta(v) := \rbr{ \EE_{s \sim d_{\pi_v}} \sbr{ \zeta_E(s,h(s)) } }_{h \in \Bcal}$, which is induced by the distribution occupancy of $\pi_v \in \Pi_{\Bcal}$, expert feedback $\feedback$, and $\Bcal$. Also, we define cost vector induced by CSC dataset $\incrdata$ and $\Bcal$ as $\est := \rbr{ \EE_{(s,\vec{c}) \sim \incrdata} \sbr{ \vec{c}(h(s)) } }_{h \in \Bcal}$. 

In Section~\ref{sec:static-reg}, we introduce algorithm \algm and separator set $\Xcal$. Given a deterministic stationary benchmark policy class $\Bcal$, its separator $\Xcal$ satisfies $\forall h, h' \in \Bcal$, $\exists x \in \Xcal$, s.t. $h(x) \neq h'(x)$. Define sample based perturbation loss variables $\ell_x \sim  \mathcal{N}(0,I_A)$ for each $x \in \Xcal$. Denote $\ell = \rbr{\ell_x}_{x \in \Xcal} \sim \mathcal{N}(0,I_{XA})$, and the induced perturbation vector $q(\ell) := (\sum_{x\in \Xcal}\ell_x(h(x)))_{h\in \Bcal} $, where $\ell_x(a)$ denotes the $a$-th term of $\ell_x$. 
When it is clear from context, we abbreviate $q(\ell)$ as $q$.
% (resp. $q(\ell_j)$)(resp. $q_j$).
Define perturbation samples set 
$Z = \cbr{ (x, \frac{K}{\eta} \cdot \ell_{x} )}_{x \in \Xcal }$, where $K$ is the sample budget per round and $\eta$ is the learning rate. Additionally, we use 
$Z_j = \cbr{ (x, \frac{K}{\eta} \cdot \ell_{x,j} )}_{x \in \Xcal }$, $j =1,\ldots,\Sp$ to index $\Sp$ perturbation sets induced by $\Sp$ draws of $\ell_j = \rbr{\ell_{x,j}}_{x \in \Xcal} \sim \mathcal{N}(0,I_{XA})$. Similarly, we abbreviate $q(\ell_j)$ as $q_j$.
% \chicheng{Needs to be revised, as only $Z_j$'s (and other quantities like $\ell_{x,j}$) are used)?}  where $\ell_{x,j} \sim  \mathcal{N}(0,I_A)$.

We denote by $\Prob(E)$  the probability of event $E$ happening.  For function $f$, we say
\begin{enumerate}
    \item $f(n) = O(\poly(n))$ if $\exists C > 0$ s.t. $f(n) = O(n^C)$;
    \item $f(n_1, \ldots, n_k) = O(\poly(n_1, \ldots, n_k))$ if $f(n_1, \ldots, n_k) = O(\poly(n_1 \times \ldots \times n_k))$;
    \item $f(n_1,\ldots, n_k) = O(\polylog(n_1,\ldots,n_k))$ if $f(n_1, \ldots, n_k) = O( \poly (\ln n_1, \ldots, \ln n_k) )$.
\end{enumerate}
%recapitulate
We summarize frequently-used definitions in the main paper and the appendix in Table~\ref{notation-table}.
\begin{table}[!htb]
    \caption{A review of  notations in this paper.} %\chicheng{Some of these explanations are not definitions (e.g. $\incrdist_{\pi}$). It is better to replace them with concise definitions here.}
  \label{notation-table}
   
      \centering
      \makebox[0pt]{
        \begin{tabular}{llll}
         \toprule
    % \cmidrule(r){1-2}
    Name     & Description & Name     & Description    \\
    \midrule
    $\MDP$ & Markov decision process & $(x,\vec{c})$& CSC example\\ 
    $H$ & Episode length & $\incrdist_{\pi}$ & $(x,\vec{c})$ distribution induced by $\pi$, $\Mcal$ and $\pi^E$ \\
    $t$ & Time step in $\MDP$ & $ \Ocal$  & CSC oracle\\
    $\Scal$ & State space & $\Pi_\Bcal$ &  Mixed policy class\\
     $S$ & State space size &$u$ & Mixed policy probability weight\\
    $s$ & State &$\pi_u$ & Mixed policy induced by $u$\\
     $\step(s)$ & Time step of state $s$ &$\theta(u)$ & Linear loss vector induced by $\pi_u$\\
    $\Acal$ & Action space & $K$  & Sample budget per round \\
    $A$ & Action space size & $k$ & Sample iteration index\\
    $a$ & Action  &$\incrdata_n$ & Set of CSC examples at iteration $n$\\
    $\rho$ &  Initial distribution  & $\EE_D$ & Empirical average over set $D$\\
    $P$ & Transition dynamics  & $\est_n$ & Estimator for $\theta(u_n)$ by $\incrdata_n$\\
    $c$ & Cost function & $\ff_n(\pi)$ & Estimator for $\iterloss_n(\pi)$ by $\incrdata_n$\\
    $\pi$ & Policy  & $\Xcal$ &  Separator set for $\Bcal$\\
    $\EE_{\pi}$ & Expectation wrt $\pi$ & $X$ &  Separator set size\\
    $\PP_{\pi}$ & Probability wrt $\pi$ & $\Sp$   &  Sparsification parameter\\
    $d_{\pi}^t$ & State occupancy distribution &  $j$ &  Sparsification iteration number\\
    $d_{\pi}$  & State occupancy distribution &$Z_j$ &  Perturbation example set  \\
    $\tau$  & Trajectory   &  $\mathcal{N}$ & Gaussian distribution\\
    $J(\pi)$ & Expected cumulative cost &$I_A$ &Identity matrix of dimension $A$\\
    $Q_\pi$ &  Action value function & $I_{XA}$ & Identity matrix of dimension $X \cdot A$\\
    $V_\pi$ &  State value function  &  $\ell_x$ & Perturbation vector drawn from $\mathcal{N}(0,I_A)$\\
    $A_\pi$ & Advantage function &$\ell$ & $\{\ell_x\}_{x \in \Xcal} \sim \mathcal{N}(0,I_{XA})$\\
    $\mu$ & Recoverability for $\pi^E$ in $\Mcal$ &$q(\ell)$ & Perturbation vector in $\RR^B$ induced by $\ell$\\
    $\pi^E$ & Expert policy &$\eta$ & Learning rate\\
    $A^E$ & Expert advantage function &$R$ & Closed and strongly convex function\\
    $\feedback$ & Expert feedback function  &$\dom(R)$ & Effective domain of $R$\\
    $\imitloss(\pi)$ & Imitation loss of $\pi$  &$R^*$ & Fenchel conjugate of $R$\\
    $N$ & Number of learning rounds   & $D_{R^*}$ & Bregman divergence of $R^*$\\
    $n$  & Learning round number & $\Phi_{\mathcal{N}}$ &Expected CFTPL objective function \\
    $i$ &  Learning round index  &$\nabla\Phi_{\mathcal{N}}$ & Gradient of $\Phi_{\mathcal{N}}$\\
    $\iterloss_n(\pi)$ & Online loss function  &$R_\Ncal$ & $\Phi^*_{\mathcal{N}}$ (Fenchel conjugate of $\Phi_{\mathcal{N}}$)\\
    $\Bcal$ &  Benchmark policy class & $u^*_n$ & Expectation of output from MFTPL\\
    $B$ &  Benchmark policy class size & $u_n$ & Output from MFTPL\\
    $h$ &  Policy in $\Bcal$  & $\auxest_n$ & Optimistic estimation for $\theta(u_n)$\\
    $\SReg_N(\Bcal)$  &  Online static regret & $\sbr{N}$ &Set $\{1,2,\cdots,N\}$\\
    $\DReg_N(\Bcal)$ & Online dynamic regret  &$I(\cdot)$ & Indicator function\\
    $\SLReg_N$ & Linear optimization regret  & $\Delta(W)$  & All probability distributions over  $W$\\
    $ \bias(\Bcal,\Bcal_0,N)$  & Approximation error & $\Onehot(w,W)$ & Delta mass (one-hot vector) on $w\in W$\\
    $\delta$ & Failure probability  &$u[w]$ &  $w$-th term of $u\in \mathbb{R}^{|W|}$\\
    $\Prob(E)$ & Probability of event $E$& $\Theta$  & $\RR^d$ or $\RR^B$ vector\\
    
       \bottomrule
    \end{tabular}
    }
\end{table}

\section{Deferred materials from Section~\ref{sec:prelims}}

\begin{proposition}[Restatement of Proposition~\ref{prop:reg-conversion}]
\label{prop:reg-conversion-restated}
For any $N \in \mathbb{N}^+$ and online learner that outputs $\cbr{\pi_n}_{n=1}^N \in \Bcal_0^N$, define $\bias(\Bcal,\Bcal_0,N) := \mathop{\max}\limits_{\{\upsilon_n\}_{n=1}^N \in \Bcal_0^N } \min\limits_{\pi \in \Bcal} \EE_{s \sim \bar{d}_{N}} \EE_{a \sim \pi(\cdot \mid s) }\sbr{I(a \neq \pi^E(s))}$, where $\bar{d}_{N} := \frac{1}{N} \sum_{n=1}^N d_{\upsilon_n}$. Then, choosing $\hat{\pi}$ uniformly at random from $\cbr{\pi_n}_{n=1}^N$ has guarantee:
%^2\cdot
\[
\EE\sbr{ J(\hat{\pi}) - J(\pi^E) }
\leq 
H \cdot
\min\cbr{ \mu\cdot \bias(\Bcal,\Bcal_0,N) + \frac{\EE[\SReg_N(\Bcal)]}{N}, \; \mu \cdot \bias(\Bcal,\Bcal_0,1) + \frac{\EE[\DReg_N(\Bcal)]}{N} }.
\]
% \chicheng{todo: define $\eta$ as the recoverability constant for $\Mcal$ and $\pi^E$.}

\end{proposition}

% \end{proposition} 
\begin{proof}[Proof of Proposition~\ref{prop:reg-conversion-restated}]

%(Equation~\eqref{def:imitloss})

% \chicheng{The  labels such as \ref{def:iterloss} cannot be used this way - if we do this, only the section number will be shown; check e.g. \url{https://en.wikibooks.org/wiki/LaTeX/Labels_and_Cross-referencing}}

By the performance difference lemma (Lemma~\ref{lem: performance difference lemma}), the definitions of $\imitloss$ and $\iterloss_n$, and the assumption that $A^E(s,a) \leq \feedback(s,a)$, we have
\[
\frac1N \sum_{n=1}^N J(\pi_n) - J(\pi^E) 
=H \cdot \frac1N  \sum_{n=1}^N \EE_{s \sim d_{\pi_n}}\EE_{a \sim \pi_n(\cdot|s)} \sbr{A^{E}(s,a)}
\leq 
\frac H N \sum_{n=1}^N   \imitloss (\pi_n)
= 
\frac H N \sum_{n=1}^N   \iterloss_n(\pi_n) .
\]
Following the definition of $\SReg_N(\Bcal)$ and $\DReg_N(\Bcal)$ in Equation~\eqref{eqn:regret},
% \chicheng{The references here (and other similar places) are broken.}
% \ref{def:regret},
\begin{equation}
\frac 1 N \sum_{n=1}^N   \iterloss_n(\pi_n) 
=
\frac 1 N \min_{\pi \in \Bcal}\sum_{n=1}^N \iterloss_n(\pi)
+
 \frac{\SReg_N(\Bcal)}{N}
=
\frac 1 N \sum_{n=1}^N \min_{\pi \in \Bcal}\iterloss_n(\pi)
+
 \frac{\DReg_N(\Bcal)}{N} .
 \label{eq:sregdreg}
\end{equation}
% \chicheng{I corrected the equation labeling here. We need to use explicit creately an named equation for this using ``begin equation.. end equation''}

Since $\bias(\Bcal,\Bcal_0,N) = \mathop{\max}\limits_{\{\upsilon_n\}_{n=1}^N \in \Bcal_0^N } \min\limits_{\pi \in \Bcal} \EE_{s \sim \bar{d}_{N}} \EE_{a \sim \pi(\cdot \mid s) }\sbr{I(a \neq \pi^E(s))}$,  $\cbr{\pi_n}_{n=1}^N \in \Bcal_0^N$, and our assumption that $\feedback(s,a) \leq \mu I(a \neq \pi^E(s))$, the static regret benchmark is bounded by:
\[
\begin{aligned}
\frac 1 N \min_{\pi \in \Bcal}\sum_{n=1}^N \iterloss_n(\pi) 
&=
 \min_{\pi \in \Bcal}\frac 1 N\sum_{n=1}^N \EE_{s \sim d_{\pi_n}} \EE_{a \sim \pi(\cdot \mid s) } \sbr{\feedback(s,a)} \\
& =
  \min_{\pi \in \Bcal} \EE_{s \sim \bar{d}_N} \EE_{a \sim \pi(\cdot \mid s) } \sbr{\feedback(s,a) }\\
 & \leq
  \min_{\pi \in \Bcal} \EE_{s \sim \bar{d}_N }\EE_{a \sim \pi(\cdot \mid s) } \sbr{\mu \cdot I(a \neq \pi^E(s))}\\
  &\leq 
  \mu \cdot \mathop{\max}\limits_{\{\upsilon_n\}_{n=1}^N \in \Bcal_0^N } \min\limits_{\pi \in \Bcal} \EE_{s \sim \bar{d}_{N}} \EE_{a \sim \pi(\cdot \mid s) }\sbr{I(a \neq \pi^E(s))}\\
&  =
\mu \cdot  \bias(\Bcal,\Bcal_0,N) .
\end{aligned}
\]
Similarly,  $\forall n$, 
\[
\begin{aligned}
\min_{\pi \in \Bcal}\iterloss_n(\pi) 
& =
  \min_{\pi \in \Bcal} \EE_{s \sim d_{\pi_n}} \EE_{a \sim \pi(\cdot \mid s) } \sbr{\feedback(s,a) }\\
 & \leq
  \min_{\pi \in \Bcal} \EE_{s \sim d_{\pi_n}} \EE_{a \sim \pi(\cdot \mid s) } \sbr{\mu \cdot I(a \neq \pi^E(s))}\\
  &\leq 
  \mu \cdot \mathop{\max}\limits_{\upsilon \in \Bcal_0 } \min_{\pi \in \Bcal} \EE_{s \sim d_{\upsilon}} \EE_{a \sim \pi(\cdot \mid s) }\sbr{I(a \neq \pi^E(s))}\\
&  =
\mu \cdot\bias(\Bcal,\Bcal_0,1) .
\end{aligned}
\]

By bringing our observations back to 
Equation~\eqref{eq:sregdreg}, we obtain
\[
\begin{aligned}
\frac1N \sum_{n=1}^N J(\pi_n) - J(\pi^E) 
\leq 
\frac 1 N \sum_{n=1}^N   \iterloss_n(\pi_n) 
 \leq 
 \mu H \cdot \bias(\Bcal,\Bcal_0,N)
 + 
 \frac{H}{N}\SReg_N(\Bcal) ,\\
 \frac 1 N \sum_{n=1}^N J(\pi_n) - J(\pi^E) 
\leq 
 \frac H N \sum_{n=1}^N   \iterloss_n(\pi_n) 
 \leq 
  \mu H \cdot\bias(\Bcal,\Bcal_0,1)
 + 
 \frac{H}{N}\DReg_N(\Bcal).
\end{aligned}
\]
Notice that $J(\pi^E)$, $ \bias(\Bcal,\Bcal_0,N)$, and $\bias(\Bcal,\Bcal_0,1)$ are constants, we apply the fact that given fixed sequence $\{\pi_n\}_{n=1}^{N}$ ,  $\EE\sbr{J(\hat{\pi})|\{\pi_n\}_{n=1}^{N}} = \frac1N \sum_{n=1}^N J(\pi_n)$ and the law of total expectation,
% \chicheng{$\pi^E$ should be changed to  $\pi^E$ throughout the paper?}
\[
\begin{aligned}
\EE\sbr{ J(\hat{\pi}) - J(\pi^E) }
=&
\EE_{\{\pi_n\}_{n=1}^{N}}\sbr{\EE\sbr{J(\hat{\pi})|\{\pi_n\}_{n=1}^{N}}} - J(\pi^E) \\
=&
\EE_{\{\pi_n\}_{n=1}^{N}}\sbr{\frac1N \sum_{n=1}^N J(\pi_n)} - J(\pi^E) \\
 \leq &
 H \cdot
\min\cbr{ \mu\cdot \bias(\Bcal,\Bcal_0,N) + \frac{\EE[\SReg_N(\Bcal)]}{N}, \; \mu \cdot \bias(\Bcal,\Bcal_0,1) + \frac{\EE[\DReg_N(\Bcal)]}{N} } ,
\end{aligned}
% _{\{\pi_n\}_{n=1}^{N}}
\]
which concludes the proof.
\end{proof}

\section{Deferred materials from Section~\ref{sec:improper-learning}}

\subsection{Deferred materials from Section~\ref{sec:proper}}
\label{sec:mixed}

\begin{theorem}[Restatement of Theorem~\ref{thm:proper-linear-reg}]
\label{thm:proper-linear-reg-proof}
% Suppose the expert's feedback is of the form $\feedback(s,a) = A^E(s,a)$. Then, for any $H \geq 2$,
% there exists an MDP $\Mcal$ of episode length $H$, a deterministic expert policy $\pi^E$,
% a benchmark policy class $\Bcal$, such that for any learner that sequentially and possibly at random generates a sequence of policies $\{\pi_n\}_{n=1}^N \in \Bcal^N$, its static regret satisfies $
% \SReg_N(\Bcal) = \Omega(N).
% $
Suppose the expert's feedback is either of the form $\feedback(s,a) = \mu \cdot I(a \neq \pi^E(s))$ or $\feedback(s,a) = A^E(s,a)$. Then, for any $H \geq 3$,
there exists an MDP $\Mcal$ of episode length $H$, a deterministic expert policy $\pi^E$,
% , such that $(\Mcal, \pi^E)$ have a recoverability constant $\mu = 1$,
a benchmark policy class $\Bcal$, such that for any learner that sequentially and possibly at random generates a sequence of policies $\{\pi_n\}_{n=1}^N \in \Bcal^N$, its static regret satisfies $
\SReg_N(\Bcal) = \Omega(N).$
\end{theorem}

\begin{figure}[h]
  \centering
    \includegraphics[scale=0.3]{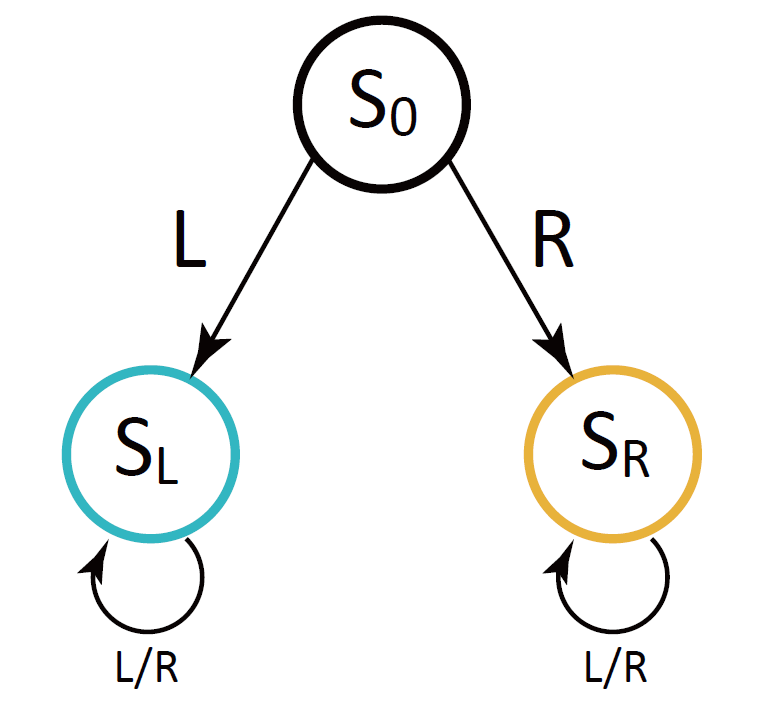}
  \caption{The MDP construction used in the proof of Theorem~\ref{thm:proper-linear-reg-proof}.}
    \label{fig:covers impossibility in imitation learning}
\end{figure}
\begin{proof}

Define MDP $\Mcal$ with:
\begin{itemize}
    \item State space $\Scal = \cbr{S_0,S_L,S_R}$ and action space $\Acal = \{L,R\}$.
    \item Initial state distribution  $\rho(S_0) = 1$
    \item Transition dynamics:  $P_1(S_L|S_0,L)=1$, $P_1(S_R|S_0,R)=1$, i.e. playing $L$ at $S_0$ transitions to $S_L$ deterministically, while playing $R$ at $S_0$ transitions to $S_R$. Also, $\forall t \in [H-1]$, $\forall a \in \Acal$, $P_t(S_L|S_L,a)=1$, $P_t(S_R|S_R,a)=1$, i.e. $S_L$ and $S_R$ have transition dynamics that are self-absorbing before termination. See Figure~\ref{fig:covers impossibility in imitation learning} for an illustration.
    \item Cost function $c(S_0,L) = c(S_0,R) =  c(S_L,R) = c(S_R,L)=0$, $c(S_L,L) = c(S_R,R)=1$.
\end{itemize}  

Meanwhile, let:
\begin{itemize}
\item Benchmark policy class $\Bcal = \{h_L,h_R\}$, where $\forall s$, $h_L(s) = L$ and $h_R(s)=R$. 
\item Deterministic expert $\pi^E$ such that $\pi^E(S_0) =L$, $\pi^E(S_L)=  R$ and $\pi^E(S_R)=  L$. 
\end{itemize}

Notice that $\forall s \in \Scal$, $c(s, \pi^E(s)) = 0$, and therefore $V_{\pi^E}(s) = 0$ for all $s$.  Also, by observing $A^E(s,a) = Q_{\pi^E}(s,a) - V_{\pi^E}(s) = c(s,a) $,
it can be seen that $A^E(S_0,L) =A^E(S_0,R) =  A^E(S_L,R) = A^E(S_R,L)=0$, $A^E(S_L,L) = A^E(S_R,R)=1$. 

By the transition dynamics, rolling out $h_L$ in $\Mcal$ incurs trajectory $\tau_{h_L} = \rbr{S_0,L,S_L,L,\cdots,S_L,L}$ with probability 1, where $A^E(S_0,L) =  I(L \neq \pi^E(S_0))  = 0 $ and $A^E(S_L,L) = I(L \neq \pi^E(S_L)) = 1$.  Similarly, the trajectory induced by $h_R$ is $\tau_{h_R} = \rbr{S_0,R,S_R,R,\cdots,S_R,R}$, where $A^E(S_0,R) = 0 $, $I(R \neq \pi^E(S_0)) = 1$ and $ A^E(S_R,R) = I(R \neq \pi^E(S_R)) = 1$.

\paragraph{For the direct expert annotation feedback $\feedback(s,a) = \mu \cdot I(a \neq \pi^E(s)).$}

To begin with, it can be seen from the advantage function values that $(\Mcal,\pi^E)$ is 1-recoverable (i.e. $\forall s\in \Scal, a\in \Acal$, $\abr{ A^E(s,a) }\leq 1$). Therefore, the feedback is of the form $\zeta_E(s,a) = I(a \neq \pi^E(s))$. Recall that $F_n(\pi) = \EE_{s \sim d_{\pi_n}}\EE_{a \sim \pi(\cdot|s)} \sbr{\zeta_E(s,a)}$, we follow the trajectories of $h_L, h_R$ and obtain:
\begin{itemize}
    \item when $\pi_n =h_L$, 
$F_n(h_L) = \frac{H-1}{H}$, $F_n(h_R) = \frac{1}{H}$.
    \item when $\pi_n =h_R$,
    $F_n(h_R) = 1$, $F_n(h_L) = 0$.
\end{itemize}  

With this, we conclude that, $\forall \{\pi_n\}_{n=1}^N \in \{h_L,h_R\}^N$,  
$ \sum_{n=1}^N F_n(\pi_n) \geq  \frac{ N (H-1)}{H}$.

On the other hand, $\forall \{\pi_n\}_{n=1}^N \in \{h_L,h_R\}^N$, define $P_L := \sum_{n=1}^N \EE_{s \sim d_{\pi_n}}I(L = \pi^E(s))$, and $P_R := \sum_{n=1}^N \EE_{s \sim d_{\pi_n}}I(R = \pi^E(s))$. 
For the benchmark term, we have 
\[
\min_{\pi \in \Bcal} \sum_{n=1}^N F_n(\pi)
=
\min\rbr{ \sum_{n=1}^N F_n(\pi_L), \sum_{n=1}^N F_n(\pi_R) }
= 
\min\rbr{ P_L , P_R}
\leq 
\frac{N}{2},
\]
where the inequality uses the observation that $P_L + P_R = N$ and therefore $\min(P_L, P_R) \leq \frac{N}{2}$.

Together we obtain 
$SReg_N(\Bcal) = \sum_{n=1}^N F_n(\pi_n) - \min_{\pi \in \Bcal}  \sum_{n=1}^N F_n(\pi) \geq
\frac{N (H-1)}{H}-\frac{N}{2}
=
\frac{N(H-2)}{2H}$, which is linear in $N$ when $H \geq 3$.

\paragraph{For the feedback of the form $\feedback(s,a) = A^E(s,a).$} 

% \chicheng{Yichen - please make another pass by try to add some of the details like what I added. Thanks!}
By bringing in $\feedback(s,a) = A^E(s,a)$, we obtain 
% By the transition dynamics, rolling out $h_L$ in $\Mcal$ incurs trajectory $\tau_L = \rbr{S_0,L,S_L,L,\cdots,S_L,L}$ with probability 1, where we recall that $A^E(S_0,L) = 0 $ and $A^E(S_L,L) = 1$.  Similarly, the trajectory induced by $h_R$ is $\tau_R = \rbr{S_0,R,S_R,R,\cdots,S_R,R}$, where $A^E(S_0,R) = 0 $ and $ A^E(S_R,R)  = 1 $. Recall the definition of  
$\iterloss_n(\pi) = \EE_{s \sim d_{\pi_n}}\EE_{a \sim \pi(\cdot|s)} \sbr{A^E(s,a)}$, following the trajectories of $h_L, h_R$, it can be seen that:
\begin{itemize}
    \item when $\pi_n =h_L$,  
$\iterloss_n(h_L) = \frac{H-1}{H}$,$\iterloss_n(h_R) =0 $.
    \item  when $\pi_n =h_R$, $\iterloss_n(h_R) = \frac{H-1}{H}$, $\iterloss_n(h_L) = 0$.
\end{itemize}  
This implies $\forall \{\pi_n\}_{n=1}^N \in \{h_L,h_R\}^N$,  
$ \sum_{n=1}^N \iterloss_n(\pi_n) =  \frac{ N (H-1)}{H}.$

On the other hand,  $\forall \{\pi_n\}_{n=1}^N \in \{h_L,h_R\}^N$, define $C_L :=  \sum_{n=1}^NI(\pi_n = h_L)$, and  $C_R :=  \sum_{n=1}^NI(\pi_n = h_R)$, where $C_L + C_R = N$. 

For the benchmark term, we have 
\[
\min_{\pi \in \Bcal} \sum_{n=1}^N F_n(\pi)
=
\min\rbr{ \sum_{n=1}^N F_n(\pi_L), \sum_{n=1}^N F_n(\pi_R) }
= 
\min\rbr{C_L \frac{H-1}{H}, C_R \frac{H-1}{H}}
\leq 
\frac{N}{2} \cdot \frac{H-1}{H},
\]
where the inequality uses the observation that $C_L + C_R = N$ and therefore $\min(C_L, C_R) \leq \frac{N}{2}$.

In conclusion, any online proper learning algorithm satisfies 
$\SReg_N(\Bcal) = \sum_{n=1}^N \iterloss_n(\pi_n) - \min_{\pi \in \Bcal}  \sum_{n=1}^N \iterloss_n(\pi) \geq
\frac{ N (H-1)}{H}-\frac{ N(H-1)}{2H}
=
\frac{ N(H-1)}{2H},$
i.e., it suffers regret $ \SReg_N(\Bcal) = \Omega(N)$ which is linear in $N$ when $H\geq 2$.
\end{proof}

\subsection{Deferred materials from Section~\ref{sec:mixture-class}}
\label{sec:mixture-class-deferred}

%interpretation
\paragraph{An alternative mixed policy class and its issues.} Prior work \cite{sun2017deeply}[Theorem 5.3]
% ,~\cite[Section 15.5]{ajks20} 
% \chicheng{On a second thought, I think it may be more fair to just quote the first citation.}
propose to use an alternative definition of mixed policy class
\[
\tilde{\Pi}_\Bcal = \cbr{ \ampolicy_u  : u \in \Delta(B) },
\]
where policy $\ampolicy_u$ is executed in an an episode of an MDP by: draw $h \sim u$ at the beginning of the episode, and execute policy $h$ throughout the episode. 
% Note that $\tilde{\Pi}_\Bcal$ is not a standard stationary policy as in our Definition~\ref{def:mixed-class}. 
% \edit{given $\sigma_u \in \tilde{\Pi}_\Bcal$,the action distribution at step $t$, $\tilde{\pi}(\cdot \mid s_t)$, depends not only on the state $s_t$, but also on $(s_1, \cdots, s_{t-1})$, which is correlated to the  $h$ drawn at the beginning of the episode}{
Importantly, $\ampolicy_u$ is not a stationary policy; as a result, $\{a_t\}_{t=1}^H$ are {\em dependent} conditioned on $\{s_t\}_{t=1}^H$; $\{a_t\}_{t=1}^H$ are only conditionally independent given $\{s_t\}_{t=1}^H$ and $h$.

By the definition of $\ampolicy_u$, $J(\ampolicy_u)$ is a weighted combination of $J(h)$ over $h \in \Bcal$, which can be written as $J(\ampolicy_u) = \sum_{h \in \Bcal} u[h] \cdot J(h)$. 
\cite{sun2017deeply}[Theorem 5.3] propose to perform online  optimization over the following losses, 
$\tilde{F}_n(\ampolicy_u) :=  \EE_{s \sim d_{\ampolicy_n}} \sbr{ \sum_{h \in \Bcal}u[h]\cdot A^E(s,h(s)) }$, 
% \chicheng{Why not using $\tilde{F}_n(u) :=  \EE_{s \sim d_{\ampolicy_n}} \sbr{ \sum_{h \in \Bcal}u[h]\cdot A^E(s,h(s)) }$ and present 
% \[
% \sum_{n=1}^N \tilde{F}_n( u_n )
% - 
% \min_{u \in \Delta(\Bcal)} \sum_{n=1}^N \tilde{F}_n( u ) \leq \mu\sqrt{N\log(B)}.
% \]
% ? }
where $\ampolicy_n$ denotes $\ampolicy_{u_n}$; specifically, they output a sequence of $\cbr{u_n}_{n=1}^N \subset \Delta(\Bcal)$,
\[
\sum_{n=1}^N \tilde{F}_n(\ampolicy_{n})
- 
\min_{u \in \Delta(\Bcal)} \sum_{n=1}^N \tilde{F}_n(\ampolicy_{u}) \leq \mu\sqrt{N\log(B)}.
\]
where $\mathop{\sup}\limits_{s,a}|A^E(s,a) |\leq \mu$. 

We show by our MDP example in Figure~\ref{fig:covers impossibility in imitation learning} above, in general,
$
J(\ampolicy_n) - J(\pi^E) 
\neq 
H \cdot \tilde{F}_n(\ampolicy_n) ,
$
which implies that, an online optimization guarantee for $\{ \tilde{F}_n(\ampolicy_u)\}_{n=1}^N$ cannot be converted to a policy suboptimality guarantee.
In contrast, in our \algreduct framework, with the setting of $\feedback = A^E$, we always have that $J(\pi_n) - J(\pi^E) = H\cdot F_n(\pi_n)$, which guarantees the conversion.\footnote{Note that the performance difference lemma (Lemma~\ref{lem: performance difference lemma}) requires the two policies in comparison to be stationary.}

Consider the MDP in the proof of Theorem~\ref{thm:proper-linear-reg-proof} and $u_n = (0.5,0.5)$. 
Here, policy $\ampolicy_n$ is executed by picking $h \in \{h_L,h_R\}$ uniformly at random and executing $h$ through the whole episode. 
Since $J(h_L) = J(h_R)= H-1$ and $J(\pi^E) =0$, we obtain 
$J(\ampolicy_n) - J(\pi^E)= \sum_{h \in \{h_L,h_R\}} u_n[h] \cdot \rbr{ J(h) -  J(\pi^E)} = H-1$.
On the other side, it can be shown that $d_{\ampolicy_n}$ distributes on $\{S_0,S_L,S_R\}$ with probability weight $(\frac{1}{H}, \frac{H-1}{2H}, \frac{H-1}{2H})$, where 
\[ \sum_{h \in \{h_L,h_R\}}u_n[h]\cdot A^E(S_0,h(s)) = 0,
\]
and 
\[ \sum_{h \in \{h_L,h_R\}}u_n[h]\cdot A^E(S_L,h(s)) = \sum_{h \in \{h_L,h_R\}}u_n[h]\cdot A^E(S_R,h(s)) =\frac{1}{2}.
\]
% \chicheng{General suggestion: if a part of the proof looks too dense, we can make some in-line equations become single-line displays}
Thus it can be verified that $\tilde{F}_n(\ampolicy_n) = \EE_{s \sim d_{\ampolicy_n}} \sbr{ \sum_{h \in \{h_L,h_R\}}u_n[h]\cdot A^E(s,h(s)) } = 2\cdot\frac{H-1}{2H} \cdot \frac{1}{2}
=\frac{H-1}{2H} $. By this we conclude $
J(\ampolicy_n) - J(\pi^E) 
\neq 
H \cdot \tilde{F}(\ampolicy_n) 
$.

\paragraph{Proof of Proposition~\ref{prop:il-olo}.} We begin by stating a more precise version of Proposition~\ref{prop:il-olo}.
%$\SReg_N(\Bcal) \leq \mathrm{Reg}(N) 

\begin{proposition}[Restatement of Proposition~\ref{prop:il-olo}]
\label{prop:il-olo-restated}
For any $\delta \in (0,1]$, if Algorithm~\ref{alg:reduce-olo} uses online linear optimization algorithm $\oloa$ that outputs $\cbr{u_n}_{n=1}^N \subset \Delta(\Bcal)^N$ s.t. with probability at least $1-\delta/3$, 
\[ 
\SLReg_N = \sum_{n=1}^N \langle \est_n, u_n \rangle - \min_{u \in \Delta(\Bcal)} \sum_{n=1}^N \langle \est_n, u \rangle \leq \mathrm{Reg}(N).
\]
Then, with probability at least $1-\delta$, its output policies $\cbr{\pi_n}_{n=1}^N$ satisfy
\[
\SReg_N(\Bcal)
\leq 
\mathrm{Reg}(N)
+
2\mu\sqrt{\frac{2N \ln(\frac{6}{\delta})}{K}}
+
2\mu\sqrt{2N\frac{\ln (B)+\ln(\frac{6}{\delta})}{K}}
=
\mathrm{Reg}(N) + O\rbr{ \mu \sqrt{\frac{N\ln(B/\delta)}{K}} }.
\]
\end{proposition}

%[Proof of Proposition~\ref{prop:reg-conversion}]
\begin{proof}
% \red{Throughout the proof, we abbreviate $B$ as $M$.}
Recall that in online IL, the loss at round $n$ is $\iterloss_n(\pi) = \EE_{s \sim d_{\pi_n}} \EE_{a \sim \pi(\cdot \mid s)} \sbr{ \feedback(s,a) }$; 
For $\pi_u\in \Pi_\Bcal$,
\[
\iterloss_n(\pi_u) 
=
\sum_{h \in \Bcal} u\coord{h} \cdot \EE_{s \sim d_{\pi_v}} \EE_{a \sim h(\cdot \mid s)}  \sbr{ \feedback(s,a) }
=
\inner{\theta(u_n)}{u},
\]
where $\theta(v) := \rbr{ \EE_{s \sim d_{\pi_n}}\EE_{a \sim h(\cdot \mid s)}  \sbr{ \zeta_E(s,a) } }_{h \in \Bcal}$.

In \algreduct, $\est_n = \rbr{ \EE_{(s,\vec{c}) \sim \incrdata_n}\EE_{a \sim h(\cdot \mid s)} \sbr{ \vec{c}(a) } }_{h \in \Bcal}$ is our unbiased estimator for $\theta(u_n)$. By defining 
$\ff_n(\pi) :=\EE_{(s,\vec{c}) \sim \incrdata_n}\EE_{a \sim \pi(\cdot|s)}\sbr{\vec{c}(a) }$, it can be seen that
$\ff_n(\pi_u) = \sum_{h \in \Bcal} u\coord{h} \cdot \EE_{(s,\vec{c}) \sim \incrdata_n} \EE_{a \sim h(\cdot \mid s)} \sbr{ \vec{c}(a) }= \langle{\est_n},{u}\rangle$. 

%it can be seen that for any $\pi$, $\ff_n(\pi)$ is an unbiased estimate of $\iterloss_n(\pi) = \EE_{s \sim d_{\pi_n}} \EE_{a \sim \pi(\cdot|s)} \sbr{\feedback(s,a)}$ 
Since the static regret is defined as $\SReg_N(\Bcal)  
= \sum_{n=1}^N \iterloss_n(\pi_n) - \min_{h\in \Bcal} \sum_{n=1}^N \iterloss_n(h) $, where 
\[
\min_{h \in \Bcal}
\sum_{n=1}^N\iterloss_n(h)
=
\min_{u \in \Delta(\Bcal)} \sum_{n=1}^N \inner{\theta(u_n)}{u} 
= 
\min_{u \in \Delta(\Bcal)}
\sum_{n=1}^N\iterloss_n(\pi_u) 
=
\min_{\pi \in \Pi_{\Bcal}}
\sum_{n=1}^N\iterloss_n(\pi).
\]
We write the static regret as
\[
\begin{aligned}
&\SReg_N(\Bcal)  
% \edit{= \sum_{n=1}^N \iterloss_n(\pi_n) - \min_{\pi\in \Bcal} \sum_{n=1}^N \iterloss_n(\pi)}{} \\
 =\sum_{n=1}^N \iterloss_n(\pi_n) - \min_{\pi \in \Pi_{\Bcal}} \sum_{n=1}^N \iterloss_n(\pi) \\
& =
\underbrace{\sum_{n=1}^N (\iterloss_n(\pi_n) - \ff_n (\pi_n) )}_{(1)}
+
\underbrace{\sum_{n=1}^N \ff_n (\pi_n)  - \min_{\pi \in  \Pi_{\Bcal}} \sum_{n=1}^N \ff_n(\pi)}_{(2)} 
+
\underbrace{\min_{\pi \in \Pi_{\Bcal}} \sum_{n=1}^N \ff_n(\pi) - 
\min_{\pi \in  \Pi_{\Bcal}} \sum_{n=1}^N \iterloss_n(\pi)}_{(3)}.
\end{aligned}
\]
We will bound each term respectively. First, for (2), we recognize that it equals to $\SLReg_N = \sum_{n=1}^N \langle \est_n, u_n \rangle - \min_{u \in \Delta(\Bcal)} \sum_{n=1}^N \langle \est_n, u \rangle$, which is at most $\mathrm{Reg}(N)$ with probability at least $1-\delta/3$ by the assumptions on $\oloa$.

We now bound the remaining two terms. Before going into details, we index each cost-sensitive examples as $(s_{n,k},\vec{c}_{n,k})$ for the $k$-th sample that drawn from the $k$-th rollout trajectory at the $n$-th round, where $k \in [K]$, $n \in [N]$ and $\vec{c}_{n,k} = \rbr{\feedback(s_{n,k},a)}_{a\in \Acal}$. With this notation, we can write cost-sensitive examples generated at round $n$ as $\incrdata_n = \rbr{(s_{n,k},\vec{c}_{n,k})}_{k = 1}^K$ and write $\ff_n(\pi)=\EE_{(s,\vec{c}) \sim \incrdata_n}\EE_{a \sim \pi(\cdot|s)}\sbr{\vec{c}(a) }=\frac 1 K \sum_{k=1}^K\EE_{a \sim \pi(\cdot \mid s_{n,k})}\sbr{\feedback(s_{n,k},a)}$. 

Also, denote $\EE_{n,k} \sbr{Y}$ as the conditional expectation of random variable $Y$ on all history before the $k$-th rollout of the $n$-th round. More precisely, denote by $\Ucal_{n,k} = \cbr{ s_{n',k'}: (n',k') \preceq (n, k) }$, where $\preceq$ denotes precedence in dictionary order, i.e., $(n_1, k_1) \preceq (n_2, k_2)$ if and only if $n_1 < n_2$, or $n_1 = n_2$ and $k_1 \leq k_2$; and 
$\EE_{n,k} \sbr{ \cdot } := \EE \sbr{ \cdot \mid \Ucal_{n,k-1} }$. As a convention, denote by $\Ucal_{n,0} := \Ucal_{n-1,K}$.
% \chicheng{in the ``additional notations'' section, given dataset $D$, we need to define $\EE_{D}$ as the empirical average over $D$}
By the assumption of $\forall s \in \Scal, \forall a \in \Acal$, $A^E(s,a)\leq \feedback(s,a)\leq \mu \cdot I(a \neq \pi^E(s))$ and $\abr{A^E(s,a)} \leq \mu $ (recall Section~\ref{sec:prelims}), we have that $|\feedback(s,a)| \leq \mu$ for all $s,a$.  

%\textbf{For the first term}
% $\sum_{n=1}^N (\iterloss_n(\pi_n) - \ff_n (\pi_n) )$,
% \chicheng{Note that in this new writing, we just view $\oloa$ as an black box, so we should not mention  details to some specific algorithms, such as $Z_{n,i}$'s. For this reason, I changed all mentioning of Algorithm~\ref{alg:ftrl-noextragrad} to mentioning Algorithm~\ref{alg:reduce-olo}.}

\paragraph{Term (1): $\sum_{n=1}^N (\iterloss_n(\pi_n) - \ff_n (\pi_n) )$.} We define 
$Y_{n,k} := \iterloss_n(\pi_n) -\EE_{a \sim \pi_n(\cdot \mid s_{n,k})} \sbr{\feedback(s_{n,k},a)} $ 
% \chicheng{Is appears that $c_{n,k}(a,s_{n,k})$ is undefined. Do you mean $\feedback(s_{n,k},a)$, or $c_{n,k}(a)$?}
where $s_{n,k} \sim d_{\pi_n}$. It can be seen from the representation of 
$\ff_n(\pi)=\frac 1 K \sum_{k=1}^K\EE_{a \sim \pi(\cdot \mid s_{n,k})}\sbr{\feedback(s_{n,k},a)}$ that 
\[ 
\iterloss_n(\pi_n)-\ff_n (\pi_n)
=
\frac 1 K \sum_{k=1}^K\rbr{\iterloss_n(\pi_n) -\EE_{a \sim \pi_n(\cdot \mid s_{n,k})}\sbr{\feedback(s_{n,k},a)}}
=
\frac{1}{K}\sum_{k=1}^K Y_{n,k} .
\]
Since $\pi_n$ only depends on history until $n-1$ round, and $s_{n,k}$ are iid drawn from $d_{\pi_n}$, we have  
\[
\begin{aligned}
\EE_{n,k} \sbr{Y_{n,k}} 
=&
\EE \sbr{ \iterloss_n(\pi_n) -\EE_{a \sim \pi_n(\cdot \mid s_{n,k})} \sbr{\feedback(s_{n,k},a)} \mid \Ucal_{n,k-1} }\\
=&
\EE \sbr{ \iterloss_n(\pi_n) -\EE_{s \sim d_{\pi_n}}\EE_{a \sim \pi_n(\cdot \mid s_{n,k})} \sbr{\feedback(s,a)} \mid \Ucal_{n-1,K} }\\
=&
\EE \sbr{ \iterloss_n(\pi_n) -\iterloss_n(\pi_n) \mid \Ucal_{n-1,K} }
= 0 .
\end{aligned}
\]

By applying  $\|\feedback(s_{n,k},\cdot)\|_{\infty} \leq \mu$, we have
% \chicheng{The tag here will show `2', the section number of the original label. Yichen, is this the behavior you intended?}

\[
\begin{aligned}
|Y_{n,k}|
&=
|\iterloss_n(\pi_n)
-
\EE_{a \sim \pi_n(\cdot \mid s_{n,k})}\sbr{\feedback(s_{n,k},a)}|\\
&=  
|\EE_{s \sim d_{\pi_n}}\EE_{a \sim \pi_n(\cdot|s) }\sbr{\feedback(s,a) } 
-
\inner{\feedback(s_{n,k},\cdot)}{\pi_n(\cdot \mid s_{n,k})} |\\
&\leq
|\EE_{s \sim d_{\pi_n}}\inner{\feedback(s,\cdot)}{\pi_n(\cdot \mid s)}|
+
\|\feedback(s_{n,k},\cdot)\|_\infty \\
&\leq
2\|\feedback(s_{n,k},\cdot)\|_\infty
\leq
2\mu.
\end{aligned}
\]
This implies the sequence of random variables $\{Y_{1,1},Y_{1,2}, \cdots, Y_{1,K}, Y_{2,1}, \cdots, Y_{N,K} \}$ form a martingale difference sequence. Applying Azuma-Hoeffding's  inequality, we get with probability at least $1-\delta/3$, 
\[
\abr{\sum_{n=1}^N (\iterloss_n(\pi_n) - \ff_n (\pi_n) )}
= 
\frac{1}{K}\abr{\sum_{n=1}^N\sum_{k = 1}^{K} Y_{n,k} }
\leq
2\mu\sqrt{\frac{2N {\ln}(\frac{6}{\delta})}{K}} .
\]
% \chicheng{The $ln$'s throughout the paper should be $\ln$'s, and perhaps there are a few other math symbols such as $exp$, $log$, $max$, $min$, etc, that need to be modified this way. }

\paragraph{Term (3): $\min_{\pi \in  \Pi_{\Bcal}} \sum_{n=1}^N \ff_n(\pi) - 
\min_{\pi \in  \Pi_{\Bcal}} \sum_{n=1}^N \iterloss_n(\pi)$.} Similar to term (1), for any $h \in \base$, we define $\hat{Y}_{n,k}(h) := \iterloss_n(h)- c_{n,k}(h(s_{n,k}),s_{n,k})$ where $s_{n,k} \sim d_{\pi_n}$. Also, we have that $\ff_n(\pi)=\frac 1 K \sum_{k=1}^K\EE_{a \sim \pi(\cdot \mid s_{n,k})}\sbr{\feedback(s_{n,k},a)}$ , which implies $\iterloss_n(h) - \ff_n(h)  = \frac{1}{K}\sum_{k=1}^K \hat{Y}_{n,k}(h)$. 
% We denote $\EE_{n,k} \sbr{\hat{Y}_{n,k}(h)} = \EE \sbr{ \hat{Y}_{n,k}(h) \mid \Ucal_{n,k-1} }$  as the conditional expectation on all history before the $(k-1)$-th rollout of the $n$ epoch. Then, for any given $h\in \Bcal$, we have  $\EE_{n,k} \sbr{\hat{Y}_{n,k}(h)} = 0$. 
Following the same analysis shown in term (1), it can be shown that $\EE_{n,k} \sbr{\hat{Y}_{n,k}(h)} = 0$ and  $|\hat{Y}_{n,k}(h)| \leq  2\mu$. By applying Azuma-Hoeffding's  inequality, we get for any given $h \in \Bcal$, with probability at least $1-\frac \delta {3B}$ (recall that $B =|\Bcal|$), 
\[
\abr{\sum_{n=1}^N (\iterloss_n(h) - \ff_n (h) )}
= 
\frac{1}{K}\abr{\sum_{n=1}^N\sum_{k = 1}^{K} \hat{Y}_{n,k} (h)}
\leq
2\mu\sqrt{2N\frac{\ln (B) + \ln(\frac{6}{\delta})}{K}} .
\]
By applying union bound over all $h \in \Bcal$, we get with probability at least $1-\frac{\delta}{3}$,
$
\sum_{n=1}^N (\iterloss_n(h) - \ff_n (h) )
\leq
2\mu\sqrt{2N\frac{\ln(B)+\ln(\frac{6}{\delta})}{K}}
$, $\forall h \in \Bcal$. Also, by the fact that $\ff_n(\pi) =\langle{\est_n},{u}\rangle $, it can be shown that
\[
\min_{\pi \in \Pi_\Bcal} \sum_{n=1}^N \ff_n(\pi) 
=
\min_{u \in \Delta(\Bcal)} \sum_{n=1}^N \langle{\est_n},{u}\rangle 
=
\min_{h \in \Bcal} \sum_{n=1}^N \ff_n(h) .
\]
Since $\min_{\pi \in \Pi_\Bcal} \sum_{n=1}^N \iterloss_n(\pi)  =\min_{h \in \Bcal} \sum_{n=1}^N \iterloss_n(h)$, by denoting $h^*\in \argmin_{h \in \Bcal} \sum_{n=1}^N \iterloss_n(h)$, we conclude with probability at least $1-\delta/3$ ,
\[
\begin{aligned}
\min_{\pi \in \Pi_\Bcal} \sum_{n=1}^N \ff_n(\pi) - 
\min_{\pi \in \Pi_\Bcal} \sum_{n=1}^N \iterloss_n(\pi)
= &
\min_{h \in \Bcal} \sum_{n=1}^N \ff_n(h) - 
\min_{h \in \Bcal} \sum_{n=1}^N \iterloss_n(h)\\
= &
\min_{h \in \Bcal} \sum_{n=1}^N \ff_n(h) 
- 
\sum_{n=1}^N \ff_n(h^*)
+
\sum_{n=1}^N (\ff_n(h^*) - \iterloss_n (h^*) )\\
\leq &
0 + 2\mu\sqrt{2N\frac{\ln (B)+\ln(\frac{6}{\delta})}{K}} .
\end{aligned}
\]

Finally, by combining our high probability bounds on terms (1),(2), and (3), applying union bound,
we conclude that with probability at least $1-\delta$, 
\[
\begin{aligned}
\SReg_N(\Bcal)  
\leq &
\mathrm{Reg}(N)
   +
   2\mu\sqrt{\frac{2N \ln(\frac{6}{\delta})}{K}}
   +
   2\mu\sqrt{2N\frac{\ln (B)+\ln(\frac{6}{\delta})}{K}}
   =
\mathrm{Reg}(N) + O\rbr{ \mu \sqrt{\frac{N\ln(B/\delta)}{K}} }.
   \end{aligned}
    \qedhere
\]
% where $\tilde{O}$ hides the $\ln(\frac{1}{\delta})$.
% \chicheng{use ``qedhere'' to avoid empty line at the end of a proof (when the proof ends with an equation)}
\end{proof}

\section{Deferred materials from Section~\ref{sec:static-reg}}
\label{sec:static-reg-deferred}

\subsection{Deferred materials from Section~\ref{sec:reg-sqrt-n}}
\label{sec:reg-sqrt-n-deferred}

%% \chicheng{Proposal: let's define $R$ explicitly before the lemma, so the readers see what is going on earlier. Note that I used a different definition of $\Phi_\Ncal$ below. This is something we should discuss in our meeting.}

\paragraph{A more precise version of Lemma~\ref{lem: ftpl_approximation-main}.}
Denote by $\ell = (\ell_x)_{x \in \Xcal}$ and $q(\ell) = (\sum_{x\in \Xcal}\ell_x(h(x)))_{h\in \Bcal}$.
Define $\Phi_\mathcal{N}: \RR^B \to \RR$ as:
\begin{equation}
\Phi_\mathcal{N}(\Theta) = \EE_{ \ell \sim \mathcal{N}(0,I_{XA})} \sbr{ \mathop{\max} \limits_{u \in \Delta(\Bcal)} \inner{\Theta + q(\ell)}{u}}.
\label{eq:phi}
\end{equation}
%-full
%When it is clear from context, we will abbreviate $q(\ell)$
% \begin{equation}
% \Phi_\mathcal{N}(\Theta) = \EE_{\ell_x\sim \mathcal{N}(0,I_A)} \sbr{ \mathop{\max} \limits_{u \in \Delta(\Bcal)} \inner{\Theta+\frac{1}{\eta}q}{u}}.
% \label{eq:phi}
% \end{equation}

% \begin{equation}
% \Phi_\mathcal{N}(\Theta) = \EE_{\ell_x\sim \mathcal{N}(0,I_A)} \sbr{ \mathop{\max} \limits_{u \in \Delta(\Bcal)} \inner{\Theta+q}{u}}.
% \label{eq:phi}
% \end{equation}
%  where $q = (\sum_{x\in \Xcal}\ell_x(h(x)))_{h\in \Bcal}$.

Also, define $R_{\Ncal}: \RR^B \to \RR \cup \cbr{+\infty}$ as $\Phi_\Ncal$'s Fenchel conjugate: 
% \begin{equation}
% R(u) 
% = 
% \eta
% \cdot
% \Phi_\mathcal{N}^*(u)
% =
% \eta \rbr{ \sup_{\tilde{\Theta}\in \mathbb{R}^B} \inner{\tilde{\Theta}}{u} - \Phi_\mathcal{N}(\tilde{\Theta})}.
% \label{eqn:r}
% \end{equation}
\begin{equation}
R_{\Ncal}(u) 
= 
\Phi_\mathcal{N}^*(u)
=
\sup_{\tilde{\Theta}\in \mathbb{R}^B} \inner{\tilde{\Theta}}{u} - \Phi_\mathcal{N}(\tilde{\Theta}).
\label{eqn:r}
\end{equation}

We will need the following two lemmas that establish properties of $\Phi_{\Ncal}$ and $R$ useful in the proof of Lemma~\ref{lem: ftpl_approximation-main}; for their proofs, please refer to Section ~\ref{sec:ftrl-ex-reg}.

\begin{lemma}
\label{lem: differentiable function}
$\Phi_\mathcal{N}(\Theta)$ is differentiable and $\nabla\Phi_{\mathcal{N}}(\Theta) = \EE_{\ell \sim \mathcal{N}(0,I_{XA})}\sbr{\mathop{\argmax}\limits_{u\in \Delta(\Bcal)} \left\langle  \Theta + q(\ell) , u\right\rangle}$. 
\end{lemma}

\begin{lemma}
\label{lem:sc_dual}
$R_{\Ncal}(u)$ is $\sqrt{\frac{\pi}{8}}{\frac{1}{ AX}}$-strongly convex with respect to $\| \cdot \|_1$. 
\end{lemma}
%(Please refer to Section~\ref{sec:gbpa}.)

\begin{lemma}
\label{lem:ftrl-ftpl}
$ 
 \mathop{\argmin} \limits_{u \in \Delta(\Bcal)} \rbr{  \langle{\Theta},{u} \rangle + R_{\Ncal}(u)}
  =
  \nabla \Phi_\mathcal{N}(-\Theta)
$.
\end{lemma}

We are now ready to present a more precise version of Lemma~\ref{lem: ftpl_approximation-main}.

%$\beta = \sqrt{\frac{8}{\pi}}\frac{1}{XA\eta}$

\begin{lemma}[A more precise version of Lemma~\ref{lem: ftpl_approximation-main}]
\label{lem: ftpl_approximation-restated}
% \edit{There exists some strongly convex function $R: \Delta(\Bcal) \to \RR$, such that the following holds.}{}
Suppose Mixed CFTPL receives datasets $\cbr{\incrdata_i}_{i=1}^{n-1}$, separator set $\Xcal$, learning rate $\eta$, sparsification parameter $\Sp$. Then, $\forall \delta \in (0,1]$, with probability at least $1-\delta$, Mixed CFTPL makes $\Sp$ calls to the cost-sensitive oracle $\Ocal$, and 
outputs $u_n \in \Delta(\Bcal)$ s.t.
\begin{equation}
\| \pi_{u_n}(\cdot|s) -  \pi_{u^*_n}(\cdot|s) \|_1 \leq \sqrt{\frac{2A \rbr{ \ln(S)+\ln(\frac{2}{\delta} )} }{\Sp}}, \quad \forall s \in \Scal,
\label{eqn:pi_u_n-pi_u_n_star}
\end{equation}
with 
\begin{equation}
   u_n^* := \argmin_{u \in \Delta(\Bcal)} \rbr{  \inner{\eta \sum_{i=1}^{n-1}\est_i}{u} + R_{\Ncal}(u) }
   =
   \nabla\Phi_{\mathcal{N}}\rbr{ -\eta\sum_{i=1}^{n-1} \est_i }
   ,
   \label{eqn:u-n-star}
\end{equation}
for $\Phi_\Ncal$ and $R$ defined in Equations~\eqref{eq:phi} and~\eqref{eqn:r}. 
% \chicheng{Need to be consistent on using $u_n^*$ vs. $u_{n+1}^*$. My sense is that the proof focuses on the latter.}

%A
%where $\pi_u(\cdot|s) = \sum_{h \in \Bcal} u\coord{h} h(\cdot|s)$, $u^* = \EE_{\ell_x\sim \mathcal{N}(0,I_A)} \sbr{ \mathop{\argmin} \limits_{u \in \Delta(\Bcal)} \inner{u}{l_D + \eta'  q}}$,
%$l_D = ( \sum_{(x,\vec{c}) \in D} \vec{c}(h(x)) )_{h \in \Bcal}$ , and $q = (\sum_{x\in \Xcal}\ell_x(h(x)))_{h\in \Bcal}$.
%Moreover, it calls the cost-sensitive classification oracle $\Ocal$ for $\Sp$ times.
\label{lem: ftpl_approximation}
\end{lemma}

\begin{proof}

In the proof, we refer to results from online linear optimization, which can be checked in Section~\ref{sec:ftrl}.
% Given $\est_n = \rbr{ \EE_{(s,\vec{c}) \sim \incrdata_n} \sbr{ \vec{c}(h(s)) } }_{h \in \Bcal}$, 
% \paragraph{Finally, we show $u^* = u_{n+1}^*$.} 
By Lemma \ref{lem:ftrl-ftpl} and Lemma~\ref{lem: differentiable function}, for $R_\Ncal$ defined in Equation~\eqref{eqn:r} and $\Phi_\mathcal{N}$ defined in Equation~\eqref{eq:phi},
% \[
% u^*
% =
% \EE_{\ell_x\sim \mathcal{N}(0,I_A)} \sbr{ \mathop{\argmax} \limits_{u \in \Delta(\Bcal)} \inner{-\sum_{i=1}^{n} \est_i+  \frac{1}{\eta} q}{u}}
% =
% \argmin_{u \in \Delta(\Bcal)} \rbr{ \eta \sum_{i=1}^{n} \langle{\est_i},{u} \rangle + R(u) } = u^*_{n+1},
% \]
\begin{equation}
\begin{aligned}
u_n^* = &
\argmin_{u \in \Delta(\Bcal)} \rbr{  \inner{\eta \sum_{i=1}^{n-1}\est_i}{u}  + R_{\Ncal}(u) }\\
   = &
   \nabla\Phi_{\mathcal{N}} \rbr{ -\eta\sum_{i=1}^{n-1} \est_i }
   =
   \EE_{\ell \sim \mathcal{N}(0,I_{XA})} \sbr{ \mathop{\argmax} \limits_{u \in \Delta(\Bcal)} \inner{-\eta\sum_{i=1}^{n-1}  \est_i+ q(\ell)}{u}}.
 \end{aligned}
\label{eqn:u-n-star-equiv}
\end{equation}

We now turn to proving Equation~\eqref{eqn:pi_u_n-pi_u_n_star}.
Recall the definition of $u_n = \frac1\Sp \sum_{j=1}^{\Sp} u_{n,j}$ in Algorithm  \ref{alg:mftpl}, where 
$u_{n,j} = \Onehot(\Ocal((\cup_{i=1}^{n-1} \incrdata_i) \cup \cup Z_j), \Bcal)$ and $Z_j$ is induced by Gaussian random variables $\ell_{x,j}\sim \mathcal{N}(0,I_A)$ for each $x \in \Xcal$. Denote by $\ell_j = (\ell_{x,j})_{x \in \Xcal}$. 

Our proof consists of two steps: first, showing that $\EE_{\ell_j \sim \mathcal{N}(0,I_{XA})}  \sbr{u_{n,j}}=  u^*_n$; second, applying the concentration inequality for Multinoulli random variables on $\| \pi_{u_n}(\cdot|s) -  \pi_{u_n^*}(\cdot|s) \|_1$ for all $s \in \Scal$.

\paragraph{To begin with, we first prove $\EE_{\ell_{j}\sim \mathcal{N}(0,I_{XA})}  \sbr{u_{n,j}}=  u_n^*$.} Since each $\incrdata_i$ contains $K$ cost-sensitive examples, and $\Xcal$ has size $X$, we have  $(\cup_{i=1}^{n-1} \incrdata_i) \cup Z_j$ contains in total $(n-1)K+X$ examples. Since 
$\est_i[h] = \EE_{(s,\vec{c}) \sim \incrdata_i} \sbr{ \vec{c}(h(s)) } $ and $Z_{j} = \cbr{ (x, \frac{K}{\eta} \ell_{x,j} ): x \in \Xcal }$, by denoting $q_j\coord{h} := \sum_{x\in \Xcal}\ell_{x,j}(h(x))$, it can be seen that
\[
\begin{aligned}
& \EE_{(x,\vec{c}) \sim (\cup_{i=1}^{n-1} \incrdata_i)  \cup Z_j} \sbr{ \vec{c}(h(x))}  \\
 =& 
 \frac{1}{(n-1)K + X}\rbr{\sum_{i=1}^{n-1}\sum_{(x,\vec{c}) \in \incrdata_i} \vec{c}(h(x))+\sum_{(x,\vec{c}) \in Z_j} \vec{c}(h(x))}\\
  =& 
 \frac{K}{(n-1)K + X} \rbr{\sum_{i=1}^{n-1} \EE_{(x,\vec{c}) \sim \incrdata_i} \sbr{ \vec{c}(h(x))} }
 + \frac{1}{(n-1)K + X} \rbr{\sum_{x \in \Xcal}\frac{K}{\eta} \ell_{x,j}(h(x))}\\
  =&
 \frac{ K}{\eta((n-1)K + X)}\rbr{\eta\sum_{i=1}^{n-1}\est_i \coord{h} + q_j\coord{h}}.
  \end{aligned}
\]
By the definition of oracle $\Ocal$ Definition~\ref{def:oracle}
% \chicheng{Yichen, next time you can try autoref or cleverref package for this}
and $u_{n,j} = \Onehot(\Ocal( (\cup_{i=1}^{n-1} \incrdata_i)  \cup Z_j), \Bcal)$,
% \chicheng{$\argmin$ to $\argmin$ throughout the paper}
\[
\begin{aligned}
u_{n,j}
&= 
\Onehot\rbr{ \Ocal( (\cup_{i=1}^{n-1} \incrdata_i) \cup Z_j), \Bcal } \\
&= 
\Onehot\rbr{ \mathop{\argmin}\limits_{h \in \Bcal} \EE_{(x,\vec{c}) \sim  (\cup_{i=1}^{n-1} \incrdata_i)  \cup Z_j} \sbr{ \vec{c}(h(x))}, \Bcal }\\
&=
\Onehot\rbr{ \mathop{\argmin}\limits_{h \in \Bcal} \frac{K}{\eta((n-1)K + X)}(\eta\sum_{i=1}^{n-1} \est_i\coord{h}  + q_j\coord{h}) , \Bcal }\\
&=
\Onehot\rbr{ \mathop{\argmin}\limits_{h \in \Bcal}(\eta\sum_{i=1}^{n-1} \est_i\coord{h}  + q_j\coord{h}), \Bcal }\\
&=
\Onehot\rbr{ \mathop{\argmax}\limits_{h \in \Bcal}(-\eta\sum_{i=1}^{n-1} \est_i\coord{h}  - q_j\coord{h}), \Bcal }.
% &=
% \mathop{\argmax}\limits_{u\in \Delta(\Bcal)} \left\langle  u ,  -\sum_{i=1}^{n-1} \est_i -  q_j\right\rangle\\
\end{aligned}
\]
By this we obtain 
\[
\EE_{\ell_{j}\sim \mathcal{N}(0,I_{XA})}  \sbr{u_{n,j}}
=
\EE_{\ell_{j}\sim \mathcal{N}(0,I_{XA})} \sbr{\Onehot \rbr{ \mathop{\argmax}\limits_{h \in \Bcal}(-\eta\sum_{i=1}^{n-1} \est_i\coord{h}  - q_j\coord{h}), \Bcal } }.
\]
On the other hand, we have by Equation~\eqref{eqn:u-n-star-equiv} that $u_n^* = \EE_{\ell \sim \mathcal{N}(0,I_{XA})} \sbr{ \mathop{\argmax}\limits_{u\in \Delta(\Bcal)} \left\langle -\eta\sum_{i=1}^{n-1} \est_i + q(\ell) , u\right\rangle}$. By lemma \ref{lem: differentiable function_restated} in Section~\ref{sec:ftrl}, under the distribution of $\ell \sim \mathcal{N}(0,I_{XA})$,  $\mathop{\argmin}\limits_{u\in \Delta(\Bcal)} \left\langle  -\eta \sum_{i=1}^{n-1} \est_i +  q(\ell) , u\right\rangle$ is unique with probability $1$, which is a one-hot vector. With this observation, we can write 
\[
\begin{aligned}
u_n^* 
=&
\EE_{\ell \sim \mathcal{N}(0,I_{XA})} \sbr{ \argmax\limits_{u\in \Delta(\Bcal)} \left\langle   -\eta\sum_{i=1}^{n-1} \est_i +  q(\ell) , u\right\rangle}\\
= & 
\EE_{\ell \sim \mathcal{N}(0,I_{XA})} \sbr{ \mathop{\argmax}\limits_{u\in \Delta(\Bcal)} \left\langle -\eta\sum_{i=1}^{n-1} \est_i -  q(\ell) , u\right\rangle} \\
=&
\EE_{\ell \sim \mathcal{N}(0,I_{XA})} \sbr{\Onehot(\mathop{\argmax}\limits_{h \in \Bcal}(-\eta\sum_{i=1}^{n-1} \est_i\coord{h}  + q\coord{h}), \Bcal)},
\end{aligned}
\]
where the first equality uses the following observation: 
% \edit{$\ell_{x}$ is iid Gaussian variable following $ \mathcal{N}(0,I_A)$ for all $x \in \Xcal$}{
Since $\ell$ is iid from $\mathcal{N}(0,I_{XA})$, $\ell$  and $-\ell$ are equal in distribution.
% Since $q(\ell) = \rbr{\sum_{x\in \Xcal}\ell_{x}(h(x)))}_{h \in \Bcal}$ , and
By observing $-q(\ell) = -\rbr{\sum_{x\in \Xcal}\ell_{x}(h(x)))}_{h \in \Bcal}=q(-\ell)$, we have that $q(\ell)$ and $-q(\ell)$ are equal in distribution.
This concludes $\EE_{\ell_{j}} \sbr{u_{n,j}}= u_n^*$.

\paragraph{Next, we bound $\|\pi_{u_n}(\cdot|s) - \pi_{u^*_n}(\cdot|s) \|_1 $.} 
% Now that we have that $\EE_{\ell_{j}\sim \mathcal{N}(0,I_{XA})} \sbr{u_{n,j}}=  u_n^*$, \edit{what remains to be shown is }{we first show that}
% \edit{
%  which is proved by showing}{
 To this end, we show that for any $s \in \Scal$, $\pi_{u_{n,j}}(\cdot|s)$ is a Multinoulli random variable with expectation $\pi_{u^*_n}(\cdot|s)$, $\pi_{u}(\cdot|s) = \frac{1}{\Sp}\sum_{j = 1}^\Sp \pi_{u_{n,j}}(\cdot|s)$, and applying concentration inequality.
% Since $\ell_{x,j}\sim \mathcal{N}(0,I_A)$ is iid drawn for all $x \in \Xcal$ and $u_{n,j}$ is induced by $\{\ell_{x,j}\}_{x \in \Xcal}$, i.e.

Since $u_{n,j} =\Onehot(\Ocal( (\cup_{i=1}^{n-1} \incrdata_i) \cup Z_j), \Bcal) $, which is a one-hot vector, by denoting $h_{n,j} = \Ocal( (\cup_{i=1}^{n-1} \incrdata_i) \cup Z_j)$, it can be seen that $\pi_{u_{n,j}}(\cdot|s) =\sum_{h \in \Bcal} u_{n,j}\coord{h} h(\cdot|s) =  h_{n,j}(\cdot | s)$ is also a one-hot vector. Also, $\forall s \in \Scal$, 
\[
\begin{aligned}
\EE_{\ell_{j}\sim \mathcal{N}(0,I_{XA})}  \sbr{\pi_{u_{n,j}}(\cdot|s)} 
= \EE_{\ell_{j}\sim \mathcal{N}(0,I_{XA})} \sbr{ \sum_{h \in \Bcal} u_{n,j}\coord{h} h(\cdot|s)} 
=  \sum_{h \in B} u^*_n\coord{h} h(\cdot|s) 
= \pi_{u^*_n}(\cdot|s).
\end{aligned}
\]
% based on our deterministic benchmark policy class assumption $h: \Scal \rightarrow \Acal  $, $u_{n,j}(\cdot|s) = \sum_{h \in \Bcal} u_{n,j}\coord{h} h(\cdot|s)$ 
Thus, $\pi_{u_{n,j}}(\cdot|s)$ can also be seen as Multinoulli random variable on $\Delta(\Acal)$ with expectation  $\pi_{u^*_n}(\cdot|s)$, and $u_n$, the empirical average of $u_{n,j}$ satisfies
\[
\pi_{u_n}(\cdot|s) 
=
\sum_{h \in \Bcal} u_n\coord{h} h(\cdot|s) 
=
\sum_{h \in \Bcal}  \frac1\Sp \sum_{i=1}^{\Sp} u_{n,j}\coord{h}h(\cdot|s) 
=
\frac{1}{\Sp}\sum_{i = 1}^\Sp \pi_{u_{n,j}}(\cdot|s).
\]
Thus, we apply concentration inequality for Multinoulli random variables~\cite{qian2020concentration}
% \chicheng{Perhaps just cite the original reference?}
(originally~\cite{weissman2003inequalities}[Theorem 2.1]) on $\pi_{u_n}(\cdot|s) $  and obtain given any $s \in \Scal$, with probability at least $1-\delta/S$, 
\[
\|\pi_{u_n}(\cdot|s) - \pi_{u^*_n}(\cdot|s) \|_1 < \sqrt{ \frac{2A\rbr{\ln(S)+\ln(\frac{2}{\delta})}}{\Sp}}.
\]
% Denote event  $A(s) $ to be $ \|\pi_{u}(\cdot|s) - \pi_{u^*_n}(\cdot|s) \|_1>\frac{\epsilon}{A}$ and by the above analysis, $\Prob(A(s,a,\epsilon(\delta))) \leq \frac{\delta}{A}$.Then, for all $a \in \Acal$, we apply the union bound and obtain $\Prob(\cup_{a \in A} A(s,a,\epsilon(\delta))) \leq \delta$. And when the event $\cup_{a\in A} A(s,a,\epsilon(\delta))$does not happen, 
% $\| \pi_{u_n}(\cdot|s) -  \pi_{u^*_n}(\cdot|s) \|_1 < \epsilon(\delta)$. 

By applying union bound over all states in $\Scal$, we conclude that, with probability at least $1-\delta$,  
\[
\| \pi_{u_n}(\cdot|s) -  \pi_{u^*_n}(\cdot|s) \|_1 \leq\sqrt{ \frac{2A\rbr{\ln(S)+\ln(\frac{2}{\delta})}}{\Sp}}, \forall s \in \Scal .
\qedhere
\]
\end{proof}
% \edit{\paragraph{ Finally, we show the existence of strongly convex function.} By lemma \ref{lem:ftrl-ftpl},
% there exists some strongly convex function $R: \Delta(\Bcal) \to \RR$, such that }

\begin{lemma}[Restatement of Lemma~\ref{lem:mftpl-reg}]
\label{lem:mftpl-reg-restated}
For any $\delta \in (0,1]$, \algm, if being called for $N$ rounds, with input
learning rate $\eta =  \frac{1}{\mu \sqrt{NA} }  (\frac{\ln (B)}{X})^{\frac14}$ and sparsification parameter $\Sp = \frac{ N \ln(2NS/\delta)}{\sqrt{X^3 \ln(B)}}$,
outputs a sequence of $\cbr{u_n}_{n=1}^N$, such that with probability $1-\delta$:
\[
\SLReg_N
=
\sum_{n=1}^N \langle {\est_n}, {u_n} \rangle 
- 
\min_{u \in \Delta(\Bcal)} \sum_{n=1}^N \langle {\est_n}, {u} \rangle 
\leq 
   O\rbr{ \mu\sqrt{ NA}(X^3\ln(B))^{\frac{1}{4}} }.
\]
\end{lemma}

\begin{proof}
Denote $u_n^* = \nabla\Phi_{\mathcal{N}}(-\eta\sum_{i=1}^{n-1} \est_i)$ the same as Equation~\eqref{eqn:u-n-star}, where $\est_n = \rbr{ \EE_{(s,\vec{c}) \sim \incrdata_n} \sbr{ \vec{c}(h(s)) } }_{h \in \Bcal}$, we rewrite the regret as
% Notice this definition of $u^*_n$ can be seen as a equivalent definition for $u_{n}^* = \argmin_{u \in \Delta(\Bcal)} \rbr{ \eta \sum_{i=1}^{n-1} \langle{\est_i},{u} \rangle + R(u) }$ for $R$ defined in Equation~\eqref{eqn:r} by Lemma \ref{lem:ftrl-ftpl}.
% \chicheng{I changed the Lemma~\ref{lem: ftpl_approximation} statement a bit so that we don't have to refer to the proof of it.}

\[
\SLReg_N
% \sum_{n=1}^N \langle {\est_n}, {u_n} \rangle 
% - 
% \min_{u \in \Delta(\Bcal)} \sum_{n=1}^N \langle {\est_n}, {u} \rangle 
=
\underbrace{
\sum_{n=1}^N \langle {\est_n}, {u_n} \rangle 
-
\sum_{n=1}^N \langle {\est_n}, {u^*_n} \rangle }_{(1)} 
+
\underbrace{
\sum_{n=1}^N \langle {\est_n}, {u^*_n} \rangle 
- 
\min_{u \in \Delta(\Bcal)} \sum_{n=1}^N \langle {\est_n}, {u} \rangle}_{(2)} . \\
\]
\paragraph{Term (1): $\sum_{n=1}^N \langle {\est_n}, {u_n} \rangle -
\sum_{n=1}^N \langle {\est_n}, {u^*_n} \rangle$.}  For the first term, instead of bounding it naively by $\|u_n - u_n^*\|_1$, we expand its definition and use Lemma~\ref{lem: ftpl_approximation} to give a tighter bound. Denote $\pi_n^* = \pi_{u_n^*}$, 
% Since $u_n$ is the output of \algm with input dataset $\cbr{\incrdata_i}_{i=1}^{n-1}$, separator set $X$, and learning rate $\eta$ , while $u^*_n := \EE_{\ell_x\sim \mathcal{N}(0,I_A)} \sbr{ \mathop{\argmin} \limits_{u} \inner{u}{\sum_{i = 1}^{n-1}\est_i + \frac{1}{\eta}  q}}$, where $\sum_{i = 1}^{n-1}\est_i $ is induced by $\cbr{\incrdata_i}_{i=1}^{n-1}$. By observing $\Sp\est_n = \sum_{(x,\vec{c}) \in \incrdata_n} \vec{c}(h(x)) )_{h \in \Bcal}$ ,  
by Lemma \ref{lem: ftpl_approximation}, we guarantee that for any round $n$, with probability at least $1-\delta/N$,
\[
\begin{aligned}
 \langle {\est_n}, {u_n-u^*_n} \rangle 
=&
\EE_{(x,\vec{c}) \sim \incrdata_n}
\EE_{a \sim \pi_n(\cdot|s) }\sbr{ \vec{c}(a) }
-
\EE_{(x,\vec{c}) \sim \incrdata_n}
 \EE_{a \sim \pi_n^*(\cdot|s) }\sbr{ \vec{c}(a) }\\
=&
\EE_{(x,\vec{c}) \sim \incrdata_n}\sbr{ \inner{\vec{c} }{\pi_n(\cdot|x)-\pi_n^*(\cdot|x) }}\\
\leq&
 \EE_{(x,\vec{c}) \sim \incrdata_n}\sbr{\| \pi_n(\cdot|x)-\pi_n^*(\cdot|x) \|_1 \| \vec{c} \|_\infty}\\
\leq &
\mu\sqrt{ 2A\frac{\ln(NS)+\ln(\frac{2}{\delta})}{\Sp}},
\end{aligned}
\]
where the last line is form applying $\|\vec{c}\|_\infty = \|\feedback(x,\cdot)\|_\infty \leq \mu$, and with probability at least $1-\delta/N$, for all $s \in \Scal$, $\| \pi_{u_n}(\cdot|s) -  \pi_{u^*_n}(\cdot|s) \|_1 \leq \sqrt{2A\frac{\ln(NS)+\ln(\frac{2}{\delta})}{\Sp}}$. Then, by applying union bound for $N$ rounds, and sum over $n \in [N]$ we obtain that with probability at least $1-\delta$,
\[
\sum_{n=1}^N \langle {\est_n}, {u_n} \rangle -
\sum_{n=1}^N \langle {\est_n}, {u^*_n} \rangle
\leq
\mu N\sqrt{\frac{ 2A\rbr{\ln(NS)+\ln(\frac{2}{\delta})}}{\Sp}} .
\]
 
\paragraph{Term (2): $\sum_{n=1}^N \langle {\est_n}, {u^*_n} \rangle - 
\min_{u \in \Delta(\Bcal)} \sum_{n=1}^N \langle {\est_n}, {u} \rangle$.} By the definition of  $u^*_n = \nabla\Phi_{\mathcal{N}}(-\eta\sum_{i=1}^{n-1} \est_i)$,
% $u_n^* = \EE_{\ell_x\sim \mathcal{N}(0,I_A)} \sbr{ \mathop{\argmax} \limits_{u \in \Delta(\Bcal)} \inner{-\sum_{i=1}^{n} \est_i+  \frac{1}{\eta} q}{u}}$, 
$u_n^*$  follows exactly the same update rule of Algorithm \ref{alg: Optimistic FTRL}
% \chicheng{In the revision, refer to Algorithm~\ref{alg: Optimistic FTRL}}
with $\Phi_{\mathcal{N}}$ defined in Equation~\eqref{eq:phi}, on online
loss $\est_n$ and optimistic estimation $\hat{\est}_n$ set to be $0$ for all $n$. By Theorem~\ref{thm: Regret of  EFTPL  with Prediction}, the regret of $\cbr{u_n^*}_{n=1}^N$ is bounded by

\[
\begin{aligned}
\sum_{n=1}^N \inner{\est_n }{u^*_n}
-
\min_{u \in \Delta(\Bcal)}\sum_{n=1}^N \inner{ \est_n}{u}
\leq&
 \frac{1}{\eta}\sqrt{2X\ln(B)}
   +
   \sum_{n=1}^N  XA\eta \|\est_n \|_\infty ^2\\
  \leq&
   \frac{1}{\eta}\sqrt{2X\ln(B)}
   +
   \eta \mu^2NXA,
\end{aligned}
\]
% where \edit{it can be verified that $\|\est_n\|_\infty \leq \mu$}{
where we use the fact that $\est_n = \rbr{ \EE_{(s,\vec{c}) \sim \incrdata_n} \sbr{ \vec{c}(h(s)) } }_{h \in \Bcal}$ has $\|\est_n\|_\infty$ at most $\mu$ as $\| \feedback(s,\cdot) \|_\infty \leq \mu$ for all $s$. 

Finally, by combining the bounds on the first and second terms together, we obtain $\forall \delta \in (0,1]$, with probability at least $1-\delta$,
\[
\sum_{n=1}^N \langle {\est_n}, {u_n} \rangle 
- 
\min_{u \in \Delta(\Bcal)} \sum_{n=1}^N \langle {\est_n}, {u} \rangle 
\leq 
\frac{1}{\eta}\sqrt{2X\ln(B)}
   +
   \eta \mu^2NXA
   +
   \mu N\sqrt{ 2A\frac{\ln(NS)+\ln(\frac{2}{\delta})}{\Sp}}.
\]
By setting $\eta = \frac{1}{\mu \sqrt{NA} }  \rbr{ \frac{\ln (B)}{X} }^{\frac14}$ and $\Sp = \frac{ N \ln(2NS/\delta)}{\sqrt{X^3 \ln(B)}}$, we conclude with probability at least $1-\delta$,
\[
\begin{aligned}
\SLReg_N
% \sum_{n=1}^N \langle {\est_n}, {u_n} \rangle 
% - 
% \min_{u \in \Delta(\Bcal)} \sum_{n=1}^N \langle {\est_n}, {u} \rangle 
\leq 
\rbr{1+\frac{\sqrt{2}}{2}+1} \mu \sqrt{2NA}\rbr{ X^3\ln(B) }^{\frac{1}{4}}
%   +
%   \mu\sqrt{ 2NA} \rbr{ 2X^3\ln(B) }^{\frac{1}{4}}
   = 
   O\rbr{ \mu\sqrt{ NA} \rbr{ X^3\ln(B) }^{\frac{1}{4}} }.
   \end{aligned}
   \qedhere
\]
\end{proof}

\begin{theorem}[Restatement of Theorem~\ref{thm:logger-m-main}]
\label{thm:logger-m-restated}
For any $\delta \in (0,1]$, \algreductm, with $\algm$ setting its parameters as in Lemma~\ref{lem:mftpl-reg} and $K=1$, satisfies that: 
(1) with probability at least $1-\delta$, its output $\cbr{\pi_n}_{n=1}^N$ satisfies:
% it outputs a sequence of policies $\cbr{\pi_n}_{n=1}^N$ \red{that satisfies}
% whose static regret satisfies that
$\SReg(N) \leq O \rbr{  \mu \sqrt{  NA  \ln(1/ \delta)} (X^3 \ln (B))^{\frac14} }$; 
(2) it queries $N$  annotations from expert $\pi^E$; 
(3) it calls the CSC oracle $\Ocal$ for $\frac{ N^2 \ln(6NS/\delta)}{\sqrt{X^3 \ln (B)}}$ times.

Specifically, \algreductm achieves $\frac{H}{N}\SReg_N(\Bcal) \leq \epsilon$  with probability at least $1-\delta$ in
$N = O\rbr{ \frac{\mu^2 H^2 A  \ln(1/ \delta) \sqrt{X^3\ln(B)}}{\epsilon^2} }$ interaction rounds, with $O\rbr{\frac{\mu^2 H^2 A  \ln(1/ \delta) \sqrt{X^3\ln(B)}}{\epsilon^2} }$ expert annotations and $\tilde{O}\rbr{\frac{\mu^4 H^4A^2 \rbr{\ln(1/ \delta)}^2  \ln(S/\delta) \sqrt{X^3\ln(B)}}{\epsilon^4} }$ oracle calls.
\end{theorem}

\begin{proof}
Following the results in Lemma~\ref{lem:mftpl-reg}, \algm, if being called for $N$ rounds with 
% \edit{input learning rate $\eta = \frac{1}{\mu \sqrt{NA} }  \rbr{ \frac{\ln (B)}{X} }^{\frac14}$ and sparsification parameter $\Sp = \frac{ \mu \ln(NS/\delta)}{\sqrt{X^3 \ln(B)}}$}{
the prescribed input learning rate 
% \chicheng{do we have a consistent naming for $\eta$? learning rate or learning rate?}
$\eta$ and sparsification parameter $\Sp =\frac{ N \ln(6NS/\delta)}{\sqrt{X^3 \ln (B)}} $,
 generates a sequence of $\cbr{u_n}_{n=1}^N$, such that with probability at least $1-\frac{\delta}{3}$, 
\[
\SLReg_N
% \sum_{n=1}^N \langle {\est_n}, {u_n} \rangle 
% - 
% \min_{u \in \Delta(\Bcal)} \sum_{n=1}^N \langle {\est_n}, {u} \rangle 
\leq 
   O\rbr{ \mu \sqrt{ NA}(X^3\ln(B))^{\frac{1}{4}} }
.
\]
By Proposition~\ref{prop:il-olo},
\algreductm output policies $\cbr{\pi_n}_{n=1}^N$ such that with probability at least $1-\delta$, 
\[
\SReg_N(\Bcal)
\leq 
O\rbr{  \mu\sqrt{ NA}(X^3\ln(B))^{\frac{1}{4}} }
+
O\rbr{\mu\sqrt{\frac{N\ln(B/\delta)}{K}}} .
% 2\mu\sqrt{\frac{2N \ln(\frac{6}{\delta})}{K}}
% +
% 2\mu\sqrt{2N\frac{\ln (B)+\ln(\frac{6}{\delta})}{K}}.
\]
% \chicheng{Commit to seting $K=1$.}
% By setting  $K = \max\{1,\frac{\ln(B)+\ln(\frac{1}{\delta})}{A\sqrt{X^3\ln(B)}}\}$, 
% \chicheng{Let's just set $K=1$.}
By setting $K=1$, we obtain 
\[
\SReg_N(\Bcal)
\leq 
   O\rbr{  \mu\sqrt{ NA}(X^3\ln(B))^{\frac{1}{4}} }
+
   O\rbr{ \mu\sqrt{N(\ln(B/ \delta))} }
%    O( \mu\sqrt{ NA}(X^3\ln(B))^{\frac{1}{4}})
% \ln{(\frac{1}{\delta}}
   =
      O\rbr{ \mu\sqrt{ NA \ln(1/ \delta)}(X^3\ln(B))^{\frac{1}{4}} },
\]
where $O \rbr{ \mu\sqrt{N\ln(B/ \delta)} }$ is of lower order, by $X \geq \log_A(B)$ proven in Lemma \ref{lem:miniseparator}.
% it can be seen that $\frac{\ln(B)+\ln(\frac{1}{\delta})}{A\sqrt{X^3\ln(B)}} = O(\frac{\ln(\frac{1}{\delta})}{AX})$ $<1$.
% \chicheng{``extremely small'' is mathematically precise.}
% For simplicity, we set $K =1$. \\

Since at each round, Algorithm~\ref{alg:reduce-olo} queries $K=1$ annotation from the expert and calls \algm once, where \algm calls oracle $\Ocal$ $\Sp =\frac{ N \ln(6NS/\delta)}{\sqrt{X^3 \ln (B)}} $ times,  then for a total of $N$ rounds, it calls $N$  annotations and calls oracle $\Ocal$ for $ \frac{ N^2 \ln(6NS/\delta)}{\sqrt{X^3 \ln (B)}} $ times.

For the second part of the theorem, to guarantee $\frac{H}{N}\SReg_N(\Bcal) \leq \epsilon$, it suffices to let $N = O\rbr{  \frac{\mu^2H^2A \ln(1/ \delta) \sqrt{X^3\ln(B)}}{\epsilon^2} }$. The number of annotations and oracle calls follow from plugging this value of $N$ into their settings in the first part of the theorem.
% \chicheng{General writing suggestion: practice how to write technical proofs concisely without losing mathematical rigor. One suggestions is to keep polishing: make a few passes over your own writing.}
\end{proof}

\subsection{Deferred materials from Section~\ref{sec:reg-const}}
\label{sec:reg-const-deferred}

\subsubsection{The \algmp algorithm and its guarantees}
\label{sec:algmp-deferred}

%\chicheng{todo for CZ: describe the \algmp in more detail}

%but 

We present \algmp (Algorithm~\ref{alg:mftpl-eg}), an alternative to \algm (Algorithm~\ref{alg:mftpl}) for online linear optimization in COIL. Recall that in the \algreduct framework, for every $n$, the linear loss $\est_n$ at round $n$ is induced by $\incrdata_n$, in that $\est_n\coord{h} = \EE_{(s,c) \sim \incrdata_n}\sbr{\vec{c}(h(s))}$ for all $h \in \Bcal$.

Specifically, at round $n$, \algmp first computes $\auxu_n$, the output of $\algm$ on $\{ \est_i \}_{i=1}^{n-1}$ (line~\ref{line:auxun}); different from $\algm$, instead of using this as $u_n$, 
it rather uses this as a \emph{estimator} for $u_n$, which is still to be determined at this point. 
$\auxu_n$ induces policy
$\auxpi_n = \pi_{\auxu_n}$.
After rolling out $\auxpi_n$ in $\Mcal$ and requesting expert annotations (line~\ref{line:get-extra-gradient}), we obtain a dataset $\auxincrdata_n$, whose induced linear loss (denoted by $\auxest_n$), is 
an 
unbiased estimate of $\theta(\auxu_n)$, which by the distributional continuity property (Lemma~\ref{lem:dist-cont}), turns out to be a good estimator of $\theta(u_n)$. 
Finally, \algmp calls \algm on the linear losses $\cbr{g_i}_{i=1}^{n-1} \cup \cbr{ \auxest_n }$ (line~\ref{line:run-extra-gradient-descent}). 

%which closely mimics performing FTRL on  $\cbr{g_i}_{i=1}^{n-1} \cup \cbr{ g_n }$

\begin{algorithm}[t]
\caption{Mixed CFTPL with Extra Gradient (abbrev. \algmp)}
\begin{algorithmic}[1]
\REQUIRE MDP $\Mcal$, expert feedback $\feedback$, Linear losses $\{\est_i\}_{i=1}^{n-1}$
represented by datasets $\cbr{\incrdata_i}_{i=1}^{n-1}$
each of size $K$
(s.t. $\est_i\coord{h} = \EE_{(s,\vec{c}) \sim \incrdata_i}\sbr{\vec{c}(h(s))}$ for all $h \in \Bcal$), 
% Datasets $\cbr{\incrdata_i}_{i=1}^{n-1}$ that represent linear losses  $\{\est_i\}_{i=1}^{n-1}$  represented by (such that $\est_i\coord{h} = \EE_{(s,c) \sim \incrdata_i}\sbr{\vec{c}(h(s))}$ for all $h \in \Bcal$), 
separator set $\Xcal$, learning rate $\eta$, sparsification parameter $\Sp$.
%, cost-sensitive classification oracle $\Ocal$.
\STATE $\auxu_n \gets \algm \rbr{ \cbr{\incrdata_i}_{i=1}^{n-1},\Xcal, \eta, T }$,
and let $\auxpi_n \gets \pi_{\auxu_n}$.
\label{line:auxun}

%cost-sensitive 
\STATE Draw $K$ examples 
% \chicheng{Add a footnote saying that the sample size of dataset can be different from $K$}
$\auxincrdata_n = \cbr{ (s, \vec{c} ) }$ iid from $D_{\auxpi_n,E}$, via interaction with $\Mcal$ and expert $\feedback$.
\label{line:get-extra-gradient}

\RETURN $u_n \gets  \algm \rbr{ \cbr{\incrdata_i}_{i=1}^{n-1} \cup \auxincrdata_n, \Xcal, \eta, T }$.
\label{line:run-extra-gradient-descent}

\end{algorithmic}
\label{alg:mftpl-eg}
\end{algorithm}

%This unbiased estimate, denoted by 

To analyze \algmp,
We first restate and prove a distributional continuity property in \oilcsc problems.
\begin{lemma}[Restatement of Lemma~\ref{lem:dist-cont}]
\label{lem:dist-cont_restated}
For any $u,v \in \Delta(\Bcal)$,
\[
\| \theta(u) - \theta(v) \|_\infty 
\;
\substack{ (*1)\\ \leq } 
\;
\mu H\max_{s \in \Scal}\| \pi_{u}(\cdot|s) -  \pi_{v}(\cdot|s) \|_1
\;
\substack{ (*2)\\ \leq }
\;
\mu H \| u - v \|_1 .
\]
\end{lemma}
\begin{proof}
% \edit{The proof follows by first showing $\forall u, v \in \Delta(\Bcal)$,$\forall s \in \Scal$, $\| \pi_{u}(\cdot|s) -  \pi_{v}(\cdot|s) \|_1 \leq \|u - v\|_1$, and then proving 
% $ \|\theta(u) - \theta(v) \|_\infty \leq  \mu H\max_s\| \pi_{u}(\cdot|s) -  \pi_{v}(\cdot|s) \|_1$.}{
We show $(*1)$ and $(*2)$ respectively.

For $(*1)$,  recall the definition of trajectory as $\tau = \{s_1,a_1,\cdots,s_H,a_H\}$, we abuse $d_{\pi}(\cdot)$ to denote the distribution of trajectory
% \chicheng{I don't think there is a standard notion of ``trajectory occupancy distribution''? Perhaps just call it ``distribution of trajectory''}
induced by policy $\pi$. By denoting  $\theta_u\coord{h} = \EE_{s \sim d_{\pi_u}}\sbr{\feedback(s,h(s)) } $ and $\theta_v\coord{h} = \EE_{s \sim d_{\pi_v}}\sbr{\feedback(s,h(s)) } $, we have for any $h \in \Bcal$: 
\[
\begin{aligned}
\abr{ \theta_u\coord{h} -  \theta_v\coord{h} }
= &
\abr{ \EE_{s \sim d_{\pi_u}}\sbr{\feedback(s,h(s)) } - \EE_{s \sim d_{\pi_v}}\sbr{\feedback(s,h(s)) } }\\
= &
\abr{ \EE_{\tau\sim d_{\pi_u}(\cdot)}\sbr{\frac{1}{H}\sum_{s \in \tau}\feedback(s,h(s)) } - \EE_{\tau \sim d_{\pi_v}(\cdot)}\sbr{\frac{1}{H}\sum_{s \in \tau}\feedback(s,h(s)) } }\\
 = & \abr{ \sum_{\tau \in (\Scal \times \Acal)^H} \rbr{ (d_{\pi_u}(\tau)- d_{\pi_v}(\tau)) \cdot
 \frac{1}{H}\sum_{s \in \tau}\feedback(s,h(s))}}  \\
 \leq &
 \abr{ \sum_{\tau \in (\Scal \times \Acal)^H} \mu (d_{\pi_u}(\tau)- d_{\pi_v}(\tau)) } \\
 = & 
 \mu \cdot \| d_{\pi_u}(\cdot) -  d_{\pi_v}(\cdot)\|_1
 =
 2\mu \cdot D_{\mathrm{TV}}\rbr{d_{\pi_u}(\cdot), d_{\pi_v}(\cdot)}.
\end{aligned}
\]

where the last inequality is by $\forall s \in \Scal$, $\forall a \in \Acal$, $\abr{\feedback(s, a)} \leq \mu$. Here, $D_{\mathrm{TV}}(u,v) := \frac{1}{2}\|u-v\|_1$ denotes the Total Variance (TV) distance between two distributions. 

% \chicheng{let's not call their result distributional continuity to further confuse the readers..}
%  \edit{by the distributional continuity observed by Theorem 4~\cite{ke2020imitation}}
Now, by~\cite{ke2020imitation}[Theorem 4]:
\begin{equation}
D_{\mathrm{TV}}\rbr{d_{\pi_u}(\cdot), d_{\pi_v}(\cdot)}
\leq
 H \cdot \EE_{s \sim d_{\pi_u}}
\sbr{D_{\mathrm{TV}}
\rbr{\pi_u(\cdot \mid s), \pi_v(\cdot \mid s)}},
\label{eqn: continuity}
\end{equation}
% Since $\feedback(s,a) \in [-\mu,\mu]$ , 
we utilize Equation~\eqref{eqn: continuity} and conclude that for any $h \in \Bcal$,
\[
\begin{aligned}
\abr{ \theta_u\coord{h} -  \theta_v\coord{h} }
\leq&
2\mu \cdot D_{\mathrm{TV}}\rbr{d_{\pi_u}(\cdot), d_{\pi_v}(\cdot)}\\
\leq&
2\mu H \cdot \EE_{s \sim d_{\pi_u}}
\sbr{D_{\mathrm{TV}}
\rbr{\pi_u(\cdot \mid s), \pi_v(\cdot \mid s)}}\\
\leq &
2\mu H \max_{s \in \Scal}D_{\mathrm{TV}}
\rbr{\pi_u(\cdot \mid s), \pi_v(\cdot \mid s)}\\
=&
\mu H\max_{s \in \Scal}\| \pi_{u}(\cdot|s) -  \pi_{v}(\cdot|s) \|_1 .
\end{aligned}
\]
For $(*2)$, by the definition of $\pi_{u_n}(\cdot|s) = \sum_{h \in \Bcal} u\coord{h} h(\cdot|s)$, we have that $\forall s \in \Scal$,
\[
\begin{aligned}
\| \pi_{u}(\cdot|s) -  \pi_{v}(\cdot|s) \|_1 
&=
\|\sum_{h \in \Bcal} u\coord{h} h(\cdot|s)
-
\sum_{h \in \Bcal} v\coord{h} h(\cdot|s)\|_1 \\
&=
\sum_{a\in \Acal} \abr{ \sum_{h \in \Bcal} (u\coord{h}-v\coord{h}) I(h(s) = a) }\\
&\leq
\sum_{a\in \Acal} \sum_{h \in \Bcal} \abr{ u\coord{h}-v\coord{h} } I(h(s) = a) \\
&=
\sum_{h \in \Bcal} \sum_{a \in \Acal} \abr{ u\coord{h}-v\coord{h} } I(h(s) = a) \\
&=
\sum_{h \in \Bcal} \abr{ u\coord{h}-v\coord{h} }
=
\|u- v\|_1 .
\end{aligned}
\] 
% where we used the property of deterministic policy $h\in \Bcal$, and  $h(\cdot|s) = \Onehot( h(s),\Acal)$. 
% \chicheng{The notation $h(s)$ and $h(a|s)$ appears simultaneously, which is a bit confusing. Consider removing the second to last line.}

This lets us conclude
\[
\| \theta(u) - \theta(v) \|_\infty 
=
\max_{h \in \Bcal}|\theta_u\coord{h} -  \theta_v\coord{h}|
\leq 
\mu H\max_{s \in \Scal}\| \pi_{u}(\cdot|s) -  \pi_{v}(\cdot|s) \|_1
\leq 
\mu H \| u - v \|_1.
\qedhere
\]
\end{proof}

% \chicheng{Remove $\max(1,.)$ operator throughout.}
\begin{lemma}
\label{lem:mftpl-eg-reg} 
Let $N \geq \mu HA\sqrt{X^3\ln(B)}$.
For any $\delta \in (0,1]$, if 
\algmp is called for $N$ rounds, with input
learning rate $\eta = \frac{1}{5\mu HA X}$, sparsification parameter $\Sp = \frac{ N^2\ln(8NS/\delta)}{\mu H AX^3\ln(B)}$ and sample budget 
$K =\frac{N\ln(8NB / \delta)}{H^2A\sqrt{X^3\ln(B)}}$  , outputs a sequence $\cbr{u_n}_{n=1}^N$, such that with probability at least $1-\delta$:
\[
\SLReg_N
% \sum_{n=1}^N \langle {\est_n}, {u_n} \rangle 
% - 
% \min_{u \in \Delta(\Bcal)} \sum_{n=1}^N \langle {\est_n}, {u} \rangle 
\leq 
   O\rbr{ \mu H A\sqrt{X^3\ln(B)} }.
\]
\end{lemma}

\begin{proof}
We will follow a proof outline similar to that of Lemma~\ref{lem:mftpl-reg-restated}; intuitively, we can view \algmp as  approximating the execution of  Algorithm~\ref{alg: Optimistic FTRL}.
%with optimistic estimation,

At the $n$-th round in the execution of \algmp , the algorithm calls \algm with dataset $\cbr{\incrdata_i}_{i=1}^{n-1}$ to get $\auxu_n$ and gather extra data set $\auxincrdata_n$ by rolling out  $\auxpi_n = \pi_{\auxu_n}$ in $\Mcal$.  Then it outputs $u_n$  by running \algm on $\cbr{\incrdata_i}_{i=1}^{n-1} \cup \auxincrdata_n$.

In parallel to the definition that $\est_n = \rbr{ \EE_{(s,\vec{c}) \sim \incrdata_n} \sbr{ \vec{c}(h(s)) } }_{h \in \Bcal}$, we denote the loss vector induced by $\auxincrdata_n$ as $\auxest_n = \rbr{ \EE_{(s,\vec{c}) \sim \auxincrdata_n} \sbr{ \vec{c}(h(s)) } }_{h \in \Bcal}$. 
% $\auxest_n = ( \frac{1}{\Sp}\sum_{(x,\vec{c}) \in \auxincrdata_n} \vec{c}(h(x)) )_{h \in \Bcal}$. 
% We define their corresponding loss vector  $\est_n = ( \frac{1}{\Sp}\sum_{(x,\vec{c}) \in \incrdata_n} \vec{c}(h(x)) )_{h \in \Bcal}$ , and  $\auxest_n = ( \frac{1}{\Sp}\sum_{(x,\vec{c}) \in \auxincrdata_n} \vec{c}(h(x)) )_{h \in \Bcal}$. 
Following a similar definition as Equation~\eqref{eqn:u-n-star}, we denote $\auxu^*_n :=  \nabla\Phi_{\mathcal{N}}(-\eta  \sum_{i=1}^{n-1}  \est_i ) $ and $u^*_n := \nabla\Phi_{\mathcal{N}}(- \eta ( \sum_{i=1}^{n-1} \est_i + \auxest_n) )$.
We first rewrite the online linear optimization regret in the same way as the proof of Lemma~\ref{lem:mftpl-reg-restated} using $u_n^*$, 
\[
\SLReg_N
=
\underbrace{
\sum_{n=1}^N \langle {\est_n}, {u_n} \rangle 
-
\sum_{n=1}^N \langle {\est_n}, {u^*_n} \rangle }_{(1)} 
+
\underbrace{
\sum_{n=1}^N \langle {\est_n}, {u^*_n} \rangle 
- 
\min_{u \in \Delta(\Bcal)} \sum_{n=1}^N \langle {\est_n}, {u} \rangle}_{(2)},  
\]
and bound terms (1) and (2) respectively.

\paragraph{Term (1): $\sum_{n=1}^N \langle {\est_n}, {u_n} \rangle -
\sum_{n=1}^N \langle {\est_n}, {u^*_n} \rangle$.} Since $u_n$ is the output from \algm with input dataset $\cbr{\incrdata_i}_{i=1}^{n-1} \cup \auxincrdata_n$, separator set $\Xcal$, and learning rate $\eta$ , while 
% $u^*_n := \EE_{\ell_x\sim \mathcal{N}(0,I_A)} \sbr{ \mathop{\argmax} \limits_{u} \inner{\sum_{i = 1}^{n-1}-\est_i  -\auxest_n+ \frac{1}{\eta}  q}{u}}$, where $\sum_{i = 1}^{n-1}\est_i $ 
$u^*_n := \nabla\Phi_{\mathcal{N}}(-\eta ( \sum_{i=1}^{n-1} \est_i +\auxest_n)) $ is induced by $\cbr{\incrdata_i}_{i=1}^{n-1}$ and $\auxest_n$ is induced by $\auxincrdata_n$. 
% By observing $\Sp\est_n = \sum_{(x,\vec{c}) \in \incrdata_n} \vec{c}(h(x)) )_{h \in \Bcal}$ ,  and $\Sp\auxest_n = \sum_{(x,\vec{c}) \in \auxincrdata_n} \vec{c}(h(x)) )_{h \in \Bcal}$
We apply  Lemma \ref{lem: ftpl_approximation} and guarantee for any given round $n$, with probability at least $1-\frac{\delta}{4N}$, for all $s \in \Scal$, 
$\| \pi_{u_n}(\cdot|s) -  \pi_{u^*_n}(\cdot|s) \|_1 \leq \sqrt{2A\frac{\ln(NS)+\ln(\frac{8}{\delta})}{\Sp}}$.
By applying union bound over $N$ rounds, we obtain that event $E_1$ happens with probability at least $1-\frac \delta 4$, where $E_1$ is defined as 
\begin{equation}
E_1 : \| \pi_{u_n}(\cdot|s) -  \pi_{u^*_n}(\cdot|s) \|_1 \leq \sqrt{2A\frac{\ln(NS)+\ln(\frac{8}{\delta})}{\Sp}},
\forall n \in [N], \forall s \in \Scal.
\label{eq: event1}
\end{equation}
Thus, when $E_1$ happens, $\forall n \in [N]$,

\[
\begin{aligned}
 \langle {\est_n}, {u_n-u^*_n} \rangle 
=&
\EE_{(x,\vec{c}) \sim \incrdata_n}
\EE_{a \sim \pi_n(\cdot|s) }\sbr{ \vec{c}(a) }
-
\EE_{(x,\vec{c}) \sim \incrdata_n}
 \EE_{a \sim \pi_n^*(\cdot|s) }\sbr{ \vec{c}(a) }\\
=&
\EE_{(x,\vec{c}) \sim \incrdata_n}\sbr{ \inner{\vec{c} }{\pi_n(\cdot|x)-\pi_n^*(\cdot|x) }}\\
\leq&
 \EE_{(x,\vec{c}) \sim \incrdata_n}\sbr{\|\pi_n(\cdot|x)-\pi_n^*(\cdot|x)\|_1\|\vec{c} \|_\infty}\\
\leq &
\mu \sqrt{ 2A\frac{\ln(NS)+\ln(\frac{8}{\delta})}{\Sp}},
\end{aligned}
\]
which implies, 
\[
\sum_{n=1}^N \langle {\est_n}, {u_n} \rangle -
\sum_{n=1}^N \langle {\est_n}, {u^*_n} \rangle
\leq
\mu N\sqrt{ 2A\frac{\ln(NS)+\ln(\frac{8}{\delta})}{\Sp}}.
\]

\paragraph{Term (2): $\sum_{n=1}^N \langle {\est_n}, {u^*_n} \rangle - 
\min_{u \in \Delta(\Bcal)} \sum_{n=1}^N \langle {\est_n}, {u} \rangle$.} By the definition of  $u^*_n := \nabla\Phi_{\mathcal{N}}(-\eta (\sum_{i=1}^{n-1} \est_i +\auxest_n)) $,
% $u^*_n := \EE_{\ell_x\sim \mathcal{N}(0,I_A)} \sbr{ \mathop{\argmax} \limits_{u} \inner{\sum_{i = 1}^{n-1}-\est_i -\auxest_n + \frac{1}{\eta}  q}{u}}$ where,$q = (\sum_{x\in \Xcal}\ell_x(h(x)))_{h\in \Bcal}$, 
$u^*_n$ follows the update rule of Algorithm~\ref{alg: Optimistic FTRL}
% \chicheng{In the revision, refer to Algorithm~\ref{alg: Optimistic FTRL}}
with $\Phi_{\mathcal{N}}$ defined in Equation~\eqref{eq:phi}, online loss $\est_n$ and optimistic estimation $\auxest_n$. By Theorem~\ref{thm: Regret of EFTPL with Prediction}, the regret is bounded by

\[
\begin{aligned}
& \sum_{n=1}^N \inner{\est_n }{u^*_n}
-
\min_{u \in \Delta(\Bcal)}\sum_{n=1}^N \inner{ \est_n}{u} \\
\leq&
 \frac{1}{\eta}\sqrt{2X\ln(B)}
   +
   \sum_{n=1}^N \rbr{  \eta AX\|\est_n - \auxest_n \|_{\infty} ^2
   -
  \frac{1}{4\eta AX} \|u_n^* - \auxu_n^*\|_1^2 },
\end{aligned}
\]
where we applied the fact that $  \nabla\Phi_{\mathcal{N}}(-\eta (\sum_{i=1}^{n-1} \est_i +\auxest_n)) -  \nabla\Phi_{\mathcal{N}}(-\eta (\sum_{i=1}^{n-1} \est_i))  = u_n^* - \auxu_n^*$.

We now turn to bound $\| \est_n - \auxest_n \|_{\infty}^2$.
Intuitively, $\est_n $ and $ \auxest_n$ are approximators of $\theta(u_n^*)$ and  $\theta(\auxu_n^*)$, while by the inequality $(*2)$ of Lemma \ref{lem:dist-cont_restated},  $\|\theta(u_n^*) - \theta(\auxu_n^*)\|_\infty \leq \mu H\|u_n^* - \auxu_n^*\|_1$.
By \cite{cheng2019accelerating}[Lemma H.3], we bound  $\|\est_n - \auxest_n \|_{\infty} ^2$ by

\[
\begin{aligned}
\|\est_n - \auxest_n \|_{\infty} ^2
\leq &
%\left(
5\cdot (
\underbrace{\|\est_n -  \theta(u_n)\|_\infty^2}_{(a)} 
+\underbrace{\|\theta(u_n)- \theta(u_n^*)\|_\infty^2}_{(b)}
+\underbrace{\|\theta(u_n^*)-\theta(\auxu_n^*)\|_\infty^2}_{(c)} 
%\right. 
\\
%\left.
& +\underbrace{\|\theta(\auxu_n^*)-\theta(\auxu_n)\|_\infty^2}_{(\tilde{b})}
+\underbrace{\|\theta(\auxu_n)- \auxest_n\|_\infty^2}_{(\tilde{a})}).
%\right)
\end{aligned}
\]
We group the terms in three groups: $(a)(\tilde{a})$, $(b)(\tilde{b})$, and $(c)$, and apply different techniques to bound them. For the easiest $(c)$ term, we apply Lemma \ref{lem:dist-cont_restated} and bound it by $\|\theta(u_n^*) - \theta(\auxu_n^*)\|_\infty^2 \leq \mu^2 H^2\|u_n^* - \auxu_n^*\|^2_1$.

 For  $(b)$ and $(\tilde{b})$, we apply inequality $(*1)$ in Lemma \ref{lem:dist-cont_restated} and get
 \[
 \begin{aligned}
 \|\theta(u_n)- \theta(u_n^*)\|_\infty^2 
 \leq &
 \mu^2 H^2\max_{s \in \Scal}\| \pi_{u_n}(\cdot|s) -  \pi_{u_n^*}(\cdot|s) \|_1^2 ,\\
 \|\theta(\auxu_n^*)-\theta(\auxu_n)\|_\infty^2
 \leq &
  \mu^2 H^2\max_{s \in \Scal}\| \pi_{\auxu_n}(\cdot|s) -  \pi_{\auxu_n^*}(\cdot|s) \|_1^2 .
  \end{aligned}
 \]
%  \edit{Since $ \max_{s \in \Scal}\| \pi_{u_n}(\cdot|s) -  \pi_{u_n^*}(\cdot|s) \|_1^2$ corresponding to term $(b)$ is bounded when event $E_1$ as defined in Equation~\eqref{eq: event1} happens, we already obtain with probability at least $1-\frac \delta 4$,  $\forall n \in [N], $ $\max_{s \in \Scal}\| \pi_{u_n}(\cdot|s) -  \pi_{u^*_n}(\cdot|s) \|_1 ^2\leq 2A\frac{\ln(NS)+\ln(\frac{8}{\delta})}{\Sp}$.}{
 For term $(b)$, on event $E_1$, which happens with probability $1-\frac\delta4$, we have that $ \max_{s \in \Scal}\| \pi_{u_n}(\cdot|s) -  \pi_{u_n^*}(\cdot|s) \|_1^2 \leq 2A\frac{\ln(NS)+\ln(\frac{8}{\delta})}{\Sp}$ for all $n \in [N]$.
For term $(\tilde{b})$, the same analysis goes through for  $\auxu_n$ the output from \algm with input dataset $\cbr{\incrdata_i}_{i=1}^{n-1} $,  and $\auxu^*_n := \EE_{\ell \sim \mathcal{N}(0,I_{XA})} \sbr{ \mathop{\argmax} \limits_{u} \inner{\sum_{i = 1}^{n-1}-\est_i + \frac{1}{\eta}  q(\ell)}{u}}$. Again, 
% \edit{we apply  lemma \ref{lem: ftpl_approximation} and  guarantee  for any given round $n$, with probability at least $1-\frac{\delta}{4N}$, $\forall s \in \Scal$, $\| \pi_{\auxu_n}(\cdot|s) -  \pi_{\auxu^*_n}(\cdot|s) \|_1^2 \leq 2A\frac{\ln(NS)+\ln(\frac{8}{\delta})}{\Sp}$. Thus, by applying union bound over all $n \in [N]$,}{
applying Lemma~\ref{lem: ftpl_approximation} and union bound over $n \in [N]$, we guarantee that the following event $E_2$  happens with probability at least $1-\frac{\delta}{4}$, where 
\[
E_2: \max_{s \in \Scal}\| \pi_{\auxu_n}(\cdot|s) - \pi_{ \auxu^*_n}(\cdot|s) \|_1^2 \leq 2A\frac{\ln(NS)+\ln(\frac{8}{\delta})}{\Sp}, \forall n \in [N], \forall s \in \Scal.
\] 
In summary, for $(b)$ and $(\tilde{b})$, we conclude:
\begin{enumerate}
\item With probability at least $1-\frac{\delta}{4}$,  $\forall n \in \Scal$,  $ \|\theta(u_n)- \theta(u_n^*)\|_\infty^2  \leq 2 \mu^2 H^2A\frac{\ln(NS)+\ln(\frac{8}{\delta})}{\Sp}$. 
\item With probability at least $1-\frac{\delta}{4}$,  $\forall n \in \Scal$,  $  \|\theta(\auxu_n^*)-\theta(\auxu_n)\|_\infty^2  \leq 2 \mu^2 H^2A\frac{\ln(NS)+\ln(\frac{8}{\delta})}{\Sp}$. 
\end{enumerate}

 For $(a)$ and $(\tilde{a})$, we first introduce notation $\theta_{n}\coord{h} = \EE_{s \sim d_{\pi_n}}\sbr{\feedback(s, h(s))}$, $\auxtheta_{n}\coord{h} = \EE_{s \sim d_{\auxpi_n}}\sbr{\feedback(s, h(s))}$, where we recall that $\pi_n = \pi_{u_n}$, $\auxpi_n = \pi_{\auxu_n}$. Also, since $\est_n\coord{h} = \EE_{(s,\vec{c}) \sim \incrdata_n}\sbr{ \vec{c}(h(s))}$ , $\auxest_n\coord{h} = \EE_{(s,\vec{c}) \sim \auxincrdata_n}\sbr{\vec{c}(h(s))}$.
%  \chicheng{$d_{\incrdata_n}$ should just be $\incrdata_n$?}
 Notice $\forall n \in [N] $ and $\forall h \in \Bcal$,  $\EE\est_n\coord{h} =  \theta_n\coord{h}$ and $\EE\auxest_n\coord{h} =  \auxest_n\coord{h}$. Since $\theta_{n}\coord{h}$, $\auxtheta_{n}\coord{h}$, $\est_n\coord{h}$, $\auxest_n\coord{h}$ are all in $[-\mu,\mu]$. We have by Hoeffding's Inequality,  given any $n \in [N]$ and $h \in \Bcal$,
\begin{enumerate}
\item With probability at least $1-\frac{\delta}{4NB}$,  $|\est_n\coord{h} -\theta_n\coord{h}| \leq 2\mu\sqrt{\frac{\ln(NB) +\ln(\frac{8}{\delta})}{2K}}$.\\
\item With probability at least $1-\frac{\delta}{4NB}$,  $|\auxest_n\coord{h} -\auxtheta_n\coord{h}| \leq 2\mu\sqrt{\frac{\ln(NB) +\ln(\frac{8}{\delta})}{2K}}$.\\
 \end{enumerate}
 With union bound applied over $[N]$ and all $h \in \Bcal$, we obtain\\
 \begin{enumerate}
\item Event $E_3$ happens with probability at least $1-\frac{\delta}{4}$, where $E_3: \|\est_n -  \theta(u_n)\|_\infty^2 \leq 2\mu^2\frac{\ln(NB) +\ln(\frac{8}{\delta})}{K}$, $\forall n \in [N]$.\\
\item Event $E_4$  happens with probability at least $1-\frac{\delta}{4}$, where $E_4: \|\theta(\auxu_n)- \auxest_n\|_\infty^2\leq 2\mu^2\frac{\ln(NB) +\ln(\frac{8}{\delta})}{K}$, $\forall n \in [N]$.\\
\end{enumerate}
Finally, by the union bound, event $E  = E_1\cap E_2 \cap E_3 \cap E_4$ happens with probability at least $1-\delta$. By combining the bounds on all terms we have, we obtain that when event $E$ happens,
\[
\begin{aligned}
\SLReg_N
\leq & 
\mu N\sqrt{ 2A\frac{\ln(NS)+\ln(\frac{8}{\delta})}{\Sp}}
+
   \frac{1}{\eta}\sqrt{2X\ln(B)}\\
   &+
   \sum_{n=1}^N\rbr{ \eta AX \|\est_n - \auxest_n \|_{\infty} ^2
   -
  \frac{1}{4\eta AX} \|u_n^* - \auxu_n^*\|_1^2 }
 \\
  \leq&
  \mu N\sqrt{ 2A\frac{\ln(NS)+\ln(\frac{8}{\delta})}{\Sp}} 
  +
  \frac{1}{\eta}\sqrt{2X\ln(B)}\\
   &+
   \sum_{n=1}^N \rbr{ \eta  AX\cdot 5 \mu^2 H^2\|u_n^* - \auxu_n^*\|^2_1
   -
  \frac{1}{4\eta AX} \|u_n^* - \auxu_n^*\|_1^2 }\\ 
  &+
  \eta AX \rbr{
  20NA \mu^2 H^2\frac{\ln(NS)+\ln(\frac{8}{\delta})}{\Sp}
  +
  20\mu^2 N\frac{\ln(NB) +\ln(\frac{8}{\delta})}{K} }.
\end{aligned}
\]
% <  \sqrt{\frac{\pi}{40}}\frac{1}{\mu HAX }
By setting $\eta = \frac{1}{5\mu HAX }$ , $\Sp = \frac{ N^2\ln(8NS/\delta)}{\mu HAX^3\ln(B)}$ and 
$K = \frac{N \ln(8NB/\delta)}{H^2A\sqrt{X^3\ln(B)}}$, we cancel the terms related to $ \|u_n^* - \auxu_n^*\|_1^2$ and conclude with probability at least $1-\delta$,
\[
\begin{aligned}
\SLReg_N
\leq &
\mu A\sqrt{2\mu HX^3\ln(B)} 
+
(5+ 4 + 4) \cdot \mu HA \sqrt{2X^3\ln(B)}
   =
   O\rbr{ \mu H A \sqrt{X^3\ln(B)} },
   \end{aligned}
%   \qedhere
\]
where $\mu A\sqrt{2\mu HX^3\ln(B)}$ is of lower order since $\mu \leq H$, and 
$\eta AX \cdot
  20NA \mu^2 H^2\frac{\ln(NS)+\ln(\frac{8}{\delta})}{\Sp} 
  =
  \frac{4\mu^2H^2A^2X^3\ln(B)}{N}
  \leq
  4 \mu H A \sqrt{2X^3\ln(B)}
  $ is from $N \geq \mu H A \sqrt{X^3\ln(B)}$.  
\end{proof}

\begin{theorem}[Restatement of Theorem~\ref{thm:logger-me-main}]
\label{thm:logger-me-restated}
Let $N \geq \mu HA\sqrt{X^3\ln(B)}$. For any $\delta \in (0,1]$,
\algreductmp, with \algmp setting its parameters as in Lemma~\ref{lem:mftpl-eg-reg}, satisfies that: (1) with probability $1-\delta$, its output $\cbr{\pi_n}_{n=1}^N$ satisfies that:
$\SReg_N(\Bcal)
% \blue{= \sum_{n=1}^N \iterloss_n(\pi_n) - \min_{\pi \in \Bcal} \sum_{n=1}^N \iterloss_n(\pi)}
\leq 
 O\rbr{ \mu HA\sqrt{X^3\ln(B)} } $;
(2) it queries $O \rbr {\frac{N^2 \ln(NB/\delta)}{H^2A\sqrt{X^3\ln(B)}}}$ annotations from expert $\pi^E$; 
(3) it calls the CSC oracle $\Ocal$ for $O \rbr{  \frac{ N^3 \ln(NS/\delta)}{\mu HAX^3\ln(B)} }$ times.\\
Specifically, \algreductmp achieves $\frac{H}{N}\SReg_N(\Bcal)<\epsilon$  with probability $1-\delta$ in
$2N =  O\rbr{ \frac{\mu H^2 A\sqrt{X^3\ln(B)}}{\epsilon} }$ interaction rounds, with $\tilde{O}\rbr{\frac{\mu^2 H^2 A \ln(B/\delta) \sqrt{X^3\ln(B)}}{\epsilon^2} }$ expert annotations and $\tilde{O}\rbr{\frac{\mu^2 H^5 A^2 \ln(S/\delta) \sqrt{X^3\ln(B)}}{\epsilon^3}}$ oracle calls.

\end{theorem}

\begin{proof}
% \chicheng{Same comment as in the corresponding part of the proof of Theorem~\ref{thm:logger-m-restated}}
Following the results in Lemma~\ref{lem:mftpl-eg-reg}, for any $\delta \in (0,1]$, \algmp, with the prescribed input learning rate $\eta$, sparsification parameter $\Sp$, and sample budget $K$, outputs a sequence of $\cbr{u_n}_{n=1}^N$, such that with probability at least $1-\frac{\delta}{3}$, 
% with input
% learning rate 
% $\eta =  \sqrt{\frac{\pi}{40}}\frac{1}{\mu H A X}$ and sparsification parameter 
% $\Sp = \frac{ N^2\ln(NS/\delta)}{H^2AX^3\ln(B)}$ , and sample budget  
% $K = \frac{N}{H^2A\sqrt{X^3\ln(B)}}\ln(\frac{NB}{\delta})$ ,  outputs a sequence of $\cbr{u_n}_{n=1}^N$, such that with probability $1-\delta/3$, 
\[
\SLReg_N
% \sum_{n=1}^N \langle {\est_n}, {u_n} \rangle 
% - 
% \min_{u \in \Delta(\Bcal)} \sum_{n=1}^N \langle {\est_n}, {u} \rangle
\leq 
     O\rbr{\mu H A\sqrt{X^3\ln(B)}}
.\] 
By Proposition~\ref{prop:il-olo},\algreductmp with the prescribed sample budget $K = O\rbr{\frac{N \ln(NB/\delta)}{H^2A\sqrt{X^3\ln(B)}}}$
% If Algorithm~\ref{alg:reduce-olo} uses online linear optimization algorithm \algm that outputs $\cbr{u_n}_{n=1}^N \subset \Delta(\Bcal)^N$ such that with probability at least $1-\delta/3$, $\SLReg_N
% % \sum_{n=1}^N \langle \est_n, u_n \rangle - \min_{u \in \Delta(\Bcal)} \sum_{n=1}^N \langle \est_n, u \rangle 
% \leq 
%   O(\mu H A\sqrt{X^3\ln(B)})$. Then, with probability $1-\delta$, its 
outputs policies $\cbr{\pi_n}_{n=1}^N$ that satisfy with probability at least $1-\delta$, 
\[
\begin{aligned}
\SReg_N(\Bcal)
\leq 
   O\rbr{\mu H A\sqrt{X^3\ln(B)}}
+
O\rbr{\mu\sqrt{\frac{N\ln(B/ \delta)}{K}}}
% 2\mu\sqrt{\frac{2N \ln(\frac{6}{\delta})}{K}}
% +
% 2\mu\sqrt{2N\frac{\ln (B)+\ln(\frac{6}{\delta})}{K}} 
=
O\rbr{\mu H A\sqrt{X^3\ln(B)}},
\end{aligned}
\]
where $O\rbr{\mu\sqrt{\frac{N\ln(B/ \delta)}{K}}} = O\rbr{\mu H\sqrt{A}(X^3\ln(B))^{\frac{1}{4}}}$ is of lower order. 

Since at each round \algreductmp queries $K = O\rbr{\frac{N \ln(NB/\delta)}{H^2A\sqrt{X^3\ln(B)}}}$ annotations from the expert and calls \algmp once, where \algmp also queries $K$ annotations, together for $N$ rounds \algreductmp calls $O \rbr {\frac{N^2 \ln(NB/\delta)}{H^2A\sqrt{X^3\ln(B)}}}$ annotations from the expert. Also, since \algmp calls \algm twice, where \algm calls oracle $\Sp =  O\rbr{\frac{ N^2 \ln(NS/\delta)}{\mu HAX^3\ln(B)} }$ times,  then $N$ rounds together \algreductmp  calls oracle  
$O \rbr{ \frac{ N^3 \ln(NS/\delta)}{\mu HAX^3\ln(B)}  }$ times. 

For the second part of the theorem, to guarantee $\frac{H}{N}\SReg_N(\Bcal) \leq \epsilon$, it suffices to let $N = O \rbr{ \frac{ \mu H^2A\sqrt{X^3\ln(B)}}{\epsilon}}$. The number of annotations and oracle calls follow from plugging this value of $N$ into their settings in the first part of the theorem.
\end{proof}

\subsection{Detailed comparisons between \algreductm, \algreductmp and behavior cloning}
\label{sec:comparison-bc}

We first present a finite-sample analysis of behavior cloning by ERM on the benchmark policy class in the following proposition.

% \chicheng{todo: compare with~\cite{syed2010reduction}}
% \blue{They need a mapping from policy to feature, which is probably a MDP simulator in our setting.}

\begin{proposition}[Application of standard agnostic ERM~\cite{shalev2014understanding}]
\label{prop:bcloss}
For $\Bcal$ that contains finite (e.g. $B$) deterministic policies $h: \Scal \rightarrow\Acal$ and  deterministic expert policy $\pi^E$, recall 
$\bias(\Bcal,\{\pi^E\},1) =  \min_{\pi \in \Bcal} \EE_{s \sim d_{\pi^E}} \sbr{I(h(s) \neq \pi^E(s))}$. Consider dataset $\Dcal =\{ (s_k,\pi^E(s_k)) \}_{k = 1}^K$,
% \chicheng{To be fair, we should compare with $\Dcal =\{ (s_k,\feedback(s_k, \cdot)) \}_{k = 1}^K$?}
where $s_k \sim d_{\pi^E}$. The output $\hat{\pi}$ from running ERM on $\Dcal$ satisfy with probability $1- \delta$,
\[
J(\hat{\pi}) - J(\pi^E)
\leq 
H^2 \cdot \bias(\Bcal,\{\pi^E\},1) + H^2  \sqrt{\frac{2(\ln( B)+\ln (\frac{2}{\delta}))}{K}}.
\]
\end{proposition} 
\begin{proof}
By standard analysis of ERM for agnostic learning (e.g. \cite{shalev2014understanding}), we have with probability $1-\delta$,

\[
\begin{aligned}
 \EE_{s \sim d_{\pi^E}} \sbr{I(h(s) \neq \pi^E(s))}
\leq &
 \min_{\pi \in \Bcal} \EE_{s \sim d_{\pi^E}} \sbr{I(h(s) \neq \pi^E(s))}+
%  2 \sqrt{\frac{\ln(B)+\ln \frac{2}{\delta}}{2 K}}
 \sqrt{\frac{2(\ln( B)+\ln (\frac{2}{\delta}))}{K}}\\
 =&
\bias(\Bcal,\{\pi^E\},1) + \sqrt{\frac{2(\ln( B)+\ln (\frac{2}{\delta}))}{K}}.
\end{aligned}
\]
By the performance difference lemma \ref{lem: performance difference lemma}, we have 
\[
 J(\pi^E) -  J(\hat{\pi}) 
=
H \cdot \EE_{s \sim d_{\pi^E}}  \sbr{ A_{\hat{\pi}}(s,\pi^E(s)) },
\]
where $A_{\hat{\pi}}(s,\pi^E(s))=Q_{\hat{\pi}}(s,\pi^E(s)) -V_{\hat{\pi}}(s)$. By the fact that $A_{\hat{\pi}} $ is bounded in $[-H,H]$ and $A_{\hat{\pi}}(s,\hat{\pi}(s))=0$, we have 
\[
\begin{aligned}
 J(\hat{\pi}) - J(\pi^E) = &
H \cdot \EE_{s \sim d_{\pi^E}}  \sbr{ -A_{\hat{\pi}}(s,\pi^E(s)) } \\
\leq &
H \cdot \EE_{s \sim d_{\pi^E}} \sbr{H\cdot I(h(s) \neq \pi^E(s))}\\
\leq  &
H^2 \cdot \bias(\Bcal,\{\pi^E\},1)
+
H^2  \sqrt{\frac{2(\ln( B)+\ln (\frac{2}{\delta}))}{K}},
\end{aligned}
\]
which concludes the proof.
\end{proof}

We now summarize the policy suboptimality guarantees of \algreductm, \algreductmp and behavior cloning in Table~\ref{tab:compare_full}, based on
Theorems~\ref{thm:logger-m-restated},~\ref{thm:logger-me-restated}, and
Proposition~\ref{prop:bcloss}; this extends Table~\ref{tab:comparison} in the main text. In addition to the main observations made in Section~\ref{sec:static-reg}, we see that \algreductmp has a factor of $O(\ln(B))$ higher sample complexity $A(\epsilon)$ than \algreductm. In addition, in terms of the dependence on $X$  and $\ln(B)$,
the sample complexity of behavior cloning $A(\epsilon)$ has a $\ln(B)$ dependence, which is better than \algreductm's $\sqrt{X^3 \ln(B)}$ and \algreductmp's $\sqrt{X^3 (\ln(B))^3}$. It would be interesting to design interactive imitation learning algorithms with sample complexity that only has a $O(\ln(B))$ dependence, by relaxing the small separator set assumption on $\Bcal$.

\begin{table}
\renewcommand\arraystretch{1.5}
  \caption{A comparison between our algorithms and the behavior cloning. Here $\tilde{O}(N)$  
%   hides  $\polylog(\frac{1}{\epsilon},\frac{1}{\delta}, H, \mu, A, X, \ln(B))$ factors.
  denotes $ O(N \ln(N))$.}
  \label{tab:compare_full}
  \centering
  \begin{tabular}{ccc}
    \toprule
   Algorithm  &  $\mathrm{Bias}$ Term & \# Interaction Rounds $I(\epsilon)$ \\
    \midrule
     \algm & $\mu H \cdot \bias(\Bcal,\Pi_\Bcal,N) $ & $O\rbr{\frac{\mu^2 H^2 A \ln(1/\delta)\sqrt{X^3\ln(B)}}{\epsilon^2}}$     \\
    \algmp & $\mu H \cdot \bias(\Bcal,\Pi_\Bcal,N) $ & $O \rbr{\frac{\mu H^2 A\sqrt{X^3\ln(B)}}{\epsilon}}$       \\
     Behavior Cloning &  $ H^2 \cdot \bias(\Bcal,\{\pi^E\},1) $  & 1   \\
       \midrule
       Algorithm  &  \# Expert Annotations $A(\epsilon)$ & \# Oracle Calls $C(\epsilon)$\\
       \midrule
      \algm &   $O\rbr{\frac{\mu^2 H^2 A \ln(1/\delta)\sqrt{X^3\ln(B)}}{\epsilon^2} }$ & $\tilde{O}\rbr{\frac{\mu^4 H^4 A^2 \ln(S/ \delta)\rbr{\ln(1/\delta)}^2 \sqrt{X^3\ln(B)}}{\epsilon^4} }$ \\
      \algmp &   $\tilde{O}\rbr{\frac{\mu^2 H^2 A \ln(B/\delta) \sqrt{X^3\ln(B)}}{\epsilon^2} }$ &  $\tilde{O}\rbr{\frac{
      \mu^2 H^5 A^2 \ln(S/\delta)\sqrt{X^3\ln(B)}}{\epsilon^3} }$ \\
        Behavior Cloning &   $O\rbr{\frac{H^4 \ln(B/\delta)}{\epsilon^2}}$ & 1\\
    \bottomrule
  \end{tabular}
\end{table}

\section{Deferred materials from Section~\ref{sec:dynamic-reg}}
\label{sec:dynamic-reg-proof}

In this section, we present the proof of Theorem~\ref{thm:hardness-dreg}. To this end, 
we will show that obtaining a sublinear dynamic regret guarantee in the \algreduct framework is as hard as computing an approximate mixed Nash equilibrium of a two-player general-sum game. This is achieved by using polynomial time reduction, where we use $Y \leq_{p} X$ to denote that problem $Y$ is polynomial-time reducible to problem $X$.
To facilitate our discussions, we start with some problem definitions.

% \chicheng{Perhaps use subsection structure?}

\subsection{Preliminaries for two-player general-sum games}
\label{sec:two-player-game}
A two-player general-sum game \cite{lemke1964equilibrium}, also known as bimatrix game, is a non-cooperative game between two players where they can choose actions (or strategies) from set $\Acal_x$, $\Acal_y$ that contain $A_x$, $A_y$ choices and gain reward from payoff matrices $V,W \in \mathbb{R}^{A_x \times A_y}$, respectively. We say a bimatrix game $(V,W)$ is positivly normalized if $V,W \in [0,1]^{A_x \times A_y}$.
Note that we use $V_{ij},W_{ij}$  $(i \in [A_x], j \in [A_y] )$ to index each element in matrix $V$ and $W$. In the bimatrix game, if the first player plays the $i$-th action and the second player plays the $j$-th action, they receive payoffs $V_{i,j}$ and $W_{i,j}$ respectively. Define mixed strategies probability distribution on action sets. The two players are allowed to play $x \in \Delta(\Acal_x)$ and $y \in \Delta(\Acal_y)$ that corresponds to the  mixed strategies on  set $\Acal_x$ and $\Acal_y$, and their payoffs are $x^\top V y$ and $x^\top W y$ respectively. A Nash equilibrium of a bimatrix game $(V,W)$ is a pair $(x,y)$, where $ x \in \Delta(\Acal_x)$, $y \in \Delta(\Acal_y)$, and no player can gain more payoff by changing $x$ or $y$ alone. A relaxed notion, $\epsilon$-approximate 2-player Nash equilibrium is defined below.
% \textbf{Definition of 2-Player $\epsilon$-approximate Nash Equilibrium  ($\epsilon$-MNE)}:  
%  \edit{In the setting of 2-player $\epsilon$-approximate Mixed Nash Equilibrium problem 2-player $\epsilon$-approximate Mixed Nash Equilibrium, $x \in \Delta(\Acal_x)$ and $y \in \Delta(\Acal_y)$ corresponds to the mixed strategies of two payers, while 
%  $V,W \in \mathbb {R}^{A_x \times A_y}$  denotes the payoff matrices for 2 players.}{
\begin{definition}[$\epsilon$-approximate 2-player Nash equilibrium]
% \edit{Given $V, W \in \mathbb {R}^{A_x \times A_y}$,}{} 
For $\epsilon \geq 0$,
an $\epsilon$-approximate 2-player Nash equilibrium $(\hat x, \hat y)$ for a bimatrix game $(V,W)$ satisfies that for any $x \in \Delta(\Acal_x)$ and $y \in \Delta(\Acal_y)$,
\[
\left\{
\begin{aligned}
\hat{x}^\top V\hat{y} \geq x^\top V\hat{y}-\epsilon, \\
\hat{x}^\top W\hat{y} \geq \hat{x}^\top  W y-\epsilon,
\end{aligned}
\right. 
\]
where $\hat x \in \Delta(\Acal_x)$, $ \hat y \in \Delta(\Acal_y)$, and $V,W \in \mathbb {R}^{A_x \times A_y}$.
\end{definition}

% \chicheng{Better to use punctuations in equations}
In other words, at $(\hat x, \hat y)$, by changing a players' mixed strategies unilaterally, the increase of her payoff is smaller than $\epsilon$. We consider a search problem of finding an approximate 2-player Nash equilibrium:

\begin{center}
\fbox{\parbox{0.9\linewidth}{
\bimatrix:\\
Input: A positively normalized bimatrix game $(V,W)$, where $V, W \in \mathbb [0,1]^{m \times m}$, $m \in \mathbb{N}$.\\
Output: An $m^{-12}$-approximate 2-player Nash equilibrium of $(V,W)$.
}
}
\end{center}

It is well-known that 
% the $\epsilon$-MNE problem with $\epsilon = O((A_x+A_y)^{-12})$ 
\bimatrix is a total search problem (i.e. any \bimatrix instance has a solution).
Furthermore, it is
% \chicheng{The words ``complete'' or hard are usually just plain text, see standard reference e.g. \url{https://theory.cs.princeton.edu/complexity/book.pdf}}
$\mathrm{PPAD}$-complete \cite{chen2006settling,daskalakis2009complexity}, which means that for any problem $Y$ in $\mathrm{PPAD}$, $Y \leq_{p} \bimatrix$. $\mathrm{PPAD}$-complete is a computational complexity class believed to be computationally intractable. 

% Any problem $Y$ in $\mathrm{PPAD}$ is polynomal-time reducible to \bimatrix, denoted by $Y \leq_{p} \bimatrix$.
 
\subsection{A related variational inequality problem} 
% \textbf{Preliminaries for $\epsilon$-Variational Inequality and  $\epsilon$-VI-MDP}
% \chicheng{At this point $\epsilon$-VI-MDP is undefined, and moreover, is something we invented on our own. Perhaps replace it by ``a related variational inequality problem''?}
%  \edit{$\epsilon$-Variational Inequalities (VI) formulates a variety of equilibrium problems with}{
The Variational Inequality (VI) formulation serves as a tool to address equilibrium problems. In this subsection, we intoduce a VI problem, which is shown to bridge \bimatrix and achieving sublinear dynamic regret in the \algreduct framework in the following subsections.  

\begin{definition}[Variational inequality]
Given $\Omega \subset \mathbb{R}^d$ and a vector field $\Fcal: \Omega \rightarrow\mathbb{R}^d$,  define $\vi(\Omega, \Fcal)$ , the variational inequality problem induced by $(\Omega, \Fcal)$ as finding $u^* \in \Omega$, such that
\[
\forall u \in \Omega,
\inner{\Fcal(u^*)}{u-u^*} \geq 0.
\]
\end{definition}

\begin{definition}[$\epsilon$-approximate solution of variational inequality]
$u^* \in \Omega$ is said to be an $\epsilon$-approximate solution of $\vi(\Omega, \Fcal)$ if
\[
\forall u \in \Omega,
\inner{\Fcal(u^*)}{u-u^*} \geq -\epsilon.
\]
\end{definition}

% as defined in  is a  from $$
Given a discrete state-action episodic MDP $\Mcal$, expert feedback $\feedback$ and deterministic policy class $\Bcal$, where $|\Bcal| = B$, Section~\ref{sec:improper-learning} defines a vector field $\theta: \Delta(\Bcal) \to \RR^B$, where $\theta(u) := \rbr{ \EE_{s \sim d_{\pi_u}} \EE_{a \sim h(\cdot \mid s)} \sbr{\feedback(s,a) } }_{h \in \Bcal}$ and $\pi_u(\cdot|s) := \sum_{h \in \Bcal} u\coord{h} \cdot h(\cdot|s)$. 
Setting $\Omega = \Delta(\Bcal) \subset \mathbb{R}^B$ and $\Fcal = \theta$, we obtain a variational inequality problem, which, as we see next, is tightly connected with the dynamic regret minimization problem in the \algreduct framework.

% turns out to be
%\edit{define}{obtain the following variational inequality problem:}
Note that in the \algreduct framework, if an algorithm at some round $n$ outputs a $\pi_n :=\pi_{u_n} \in \Pi_\Bcal$ that achieves a low instant dynamic regret guarantee: $F_n(\pi_n) - \min_{\pi\in \Bcal} F_n(\pi) \leq \epsilon$, by the fact that $F_n(\pi_u) = \inner{\theta(u_n)}{u} $, we obtain  $\inner{\theta(u_n)}{u_n}  -\mathop{\min}\limits_{u\in \Delta(\Bcal)} \inner{\theta(u_n)}{u}\leq \epsilon$, which is equivalent to $\forall u \in \Delta(\Bcal)$, $\inner{\theta(u_n)}{u-u_n} \geq -\epsilon.$ This implies $u_n$ is a $\epsilon$-approximate solution of $\vi(\Delta(\Bcal), \theta)$.
%and connects $\epsilon$-VI$(\Delta(\Bcal),\theta)$ with achieving sublinear dynamic regret in the \algreduct framework.
%$\epsilon$-VI$(\Delta(\Bcal),\theta)$

% \edit{Consider a search problem of finding an $\epsilon$-VI$(\Delta(\Bcal),\theta)$ solution,}{
This motivates the following search problem of finding an $\epsilon$-approximate solution of $\vi(\Delta(\Bcal), \theta)$:
\begin{center}
\fbox{\parbox{0.9\linewidth}{
\vimdp:\\
Input: Discrete state-action episodic MDP $\Mcal = (\Scal,\Acal,H,c,\rho,P)$, expert feedback $\feedback: \Scal \times \Acal \to \RR$, deterministic policy class $\Bcal$.\\
Output: A $(S+A+B)^{-6}$-approximate solution of $\vi(\Delta(\Bcal),\theta)$, where $S = |\Scal|, A = |\Acal|, B = |\Bcal|$.
}
}
\end{center}
% \frac{2}{27}
For the remainder of this section, we first establish a reduction from \bimatrix to \vimdp.
% $\frac{2}{27}\epsilon$-approximate VI-MDP( $\epsilon$-VI-MDP). 
% This implies \red{solving $\epsilon$-VI-MDP is }at least as hard as solving a $\epsilon$-MNE, which known to be  $\mathrm{PPAD}$-$\mathrm{complete}$ (see e.g. \cite{daskalakis2009complexity} ). 
Then, we show that an efficient algorithm that achieves sublinear dynamic regret in the \algreduct framework yields an efficient procedure for solving \vimdp,
and thus, all $\mathrm{PPAD}$ problems are solvable in randomized polynomial time. 
%% $\epsilon$-MNE with $\epsilon = O((A_x+A_y)^{-12})$ 
%, which is believed to be computationally intractable.

Before diving into the reduction, we first prove that any \vimdp instance has a solution.

\begin{lemma}
\label{lem: Existence of  VI Solution}
% Any \vimdp has a solution.
\vimdp is a total search problem, i.e., any \vimdp problem instance has a solution.
\end{lemma}
\begin{proof}
Theorem 3.1 in \cite{facchinei2007finite} says that, for any nonempty, convex and compact subset $\Omega \subset \mathbb{R}^{n}$ and continuous mapping $\Fcal: \Omega \rightarrow \mathbb{R}^{n}$, there exists an exact solution to the $\vi(\Omega, \Fcal)$. 
% To guarantee a solution for $\epsilon$-approximate VI-MDP $(\epsilon \geq 0)$, 
Then, it suffices to show that $\Omega=\Delta(\Bcal)$ and $\Fcal = \theta$ satisfy these requirements.
% requires the function in VI to be a continuous mapping and the space to be a nonempty, compact and convex subset of $\mathbb{R}^B$. 
% If the requirements are satisfied, then we can  where $\epsilon = 0$. 

First of all, by lemma \ref{lem:dist-cont}, for any $u,v \in \Delta(\Bcal)$,
$\| \theta(u) - \theta(v) \|_\infty \leq \mu H \| u - v \|_1,$
which implies $\theta(\cdot)$ is a continuous mapping. Secondly, since it is easy to verify that $\Delta(\Bcal)$ is a convex set, it remains to show the compactness of $\Delta(\Bcal)$.
By the Heine–Borel theorem, a subset in $\RR^B$ is compact if and only if it is closed and bounded. Then, it suffices to show $\Delta(\Bcal)$ is closed and bounded. It can be seen that, $\Delta(\Bcal):=\left\{u \in \mathbb{R}^{B} \mid u \succeq 0, \sum_{h \in \Bcal} u[h] = 1 \right\}$ is closed, being the intersection of closed sets, namely the orthant $\mathbb{R}_{+}^{B}$ and the hyperplane $\left\{u \in \mathbb{R}^{B} \mid \sum_{h \in \Bcal} u[h] = 1\right\}$. Also, $\Delta(\Bcal)$ is a subset of the hypercube $[0,1]^{B}$, and is therefore bounded.   
% \edit{Since the compactness on $\RR^B$  is being closed and bounded}{
% Since any probability simplex is a convex set,
% \edit{why convexity matters?}
Combining the above, we conclude the proof.
\end{proof}

\subsection{\vimdp is $\mathrm{PPAD}$-hard}

\begin{theorem}
\label{thm:vi-mdp-ppad}
\vimdp is $\mathrm{PPAD}$-hard.
\end{theorem}
% \chicheng{Change to$\mathrm{PPAD}$-$\mathrm{hard}$throughout}
% There exists a polynomial-time reduction from the $\epsilon$-MNE problem to the $\frac{2}{27}\epsilon$-VI-MDP problem. 

%This way we don't have to present the reduction from the other way around.

\begin{proof}
Since \bimatrix is $\mathrm{PPAD}$-complete by~\cite{chen2009settling}, we show \vimdp is $\mathrm{PPAD}$-hard by proving \bimatrix $\leq_{p}$ \vimdp.
%$\mathrm{hard}$
% The reduction from solving a  to a  follows by 
%$(V, W)$ 
%$(\Mcal, \zeta_E, \Bcal)$
%\edit{}{instance $(\Mcal, \zeta_E, \Bcal)$}
%\edit{}{the original}
%\edit{}{instance $(V, W)$}%$u$ 
%$(\hat{x}, \hat{y})$
Our proof is organized as follows: First, we describe map \reducef that maps an instance of \bimatrix to an instance of \vimdp, and map \reduceb that maps a solution of \vimdp  to a solution of \bimatrix. Then, we prove that \reducef and \reduceb run in polynomial time and satisfy:
% \chicheng{in general, be clear about problem, instance, solutions - ``X is a solution of a instance Y in problem Z''}
\begin{enumerate}
    \item If $(V,W)$ is an input of \bimatrix, then $\reducef(V,W)$ is an input of \vimdp.
    \item If $u^*$ is a solution of \vimdp instance $\reducef(V,W)$, then $\reduceb(u^*)$ is also a solution of \bimatrix instance $(V,W)$.
    \item If no $u$ is a solution of \vimdp instance $\reducef(V,W)$, then no $(x,y)$ a solution of \bimatrix instance $(V,W)$.
\end{enumerate}

% Finally, suppose there exist an algorithm \algr that solves all \vimdp, we prove that for any valid input $(V,W)$ of \bimatrix, $\reduceb(\algr(\reducef((V,W))))$ is a solution for \bimatrix.

% we present our reduction: given a bimatrix game, we construct a tabular MDP, benchmark policy class and expert feedback that encodes the game, and consider the \vimdp therein. Secondly, we show that any solution of \vimdp (induced by \bimatrix) can be transformed to a solution of \bimatrix. Finally, we show that the transformation from $\epsilon$-MNE to $\frac{2}{27}\epsilon$-VI-MDP as well as the transformation from the solution of $\frac{2}{27}\epsilon$-VI-MDP to $\epsilon$-MNE take polynomial time.
\paragraph{Map \reducef.}
Given any \bimatrix instance $(V,W)$ where $V,W \in \mathbb [0,1]^{m \times m}$ we construct a \vimdp instance where the MDP $\Mcal$ can be viewed as a three-layer tree (whose details will be given shortly), and 
every non-leaf node in the $\Mcal$ has $A = 2m+1$ children. See Figure~\ref{fig:VI-MDP formal}. 
% \red{change the figure}

\begin{figure}[h]
    \centering
    \includegraphics[width=0.99\textwidth]{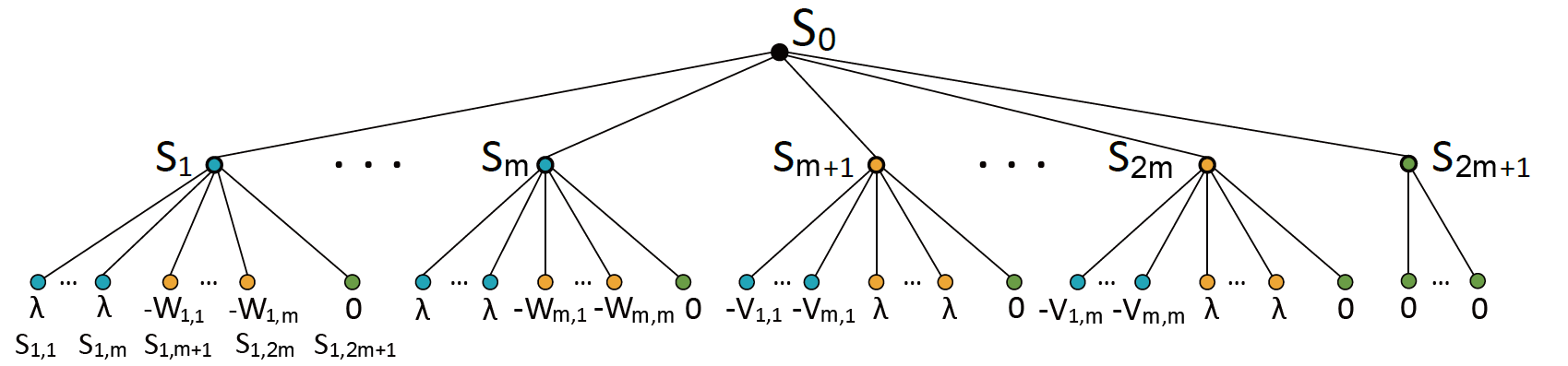}
    \caption{The MDP $\Mcal$ constructed in our reduction.}
    % \chicheng{Given this, we need to change the labels of the node in the graph?}
     \label{fig:VI-MDP formal}
\end{figure}

%The three layer tree structured MDP
%The MDP $\Mcal$ is pictorially shown in \edit{the}{}. 
Formally, layered MDP $\Mcal$ has episode length $H = 3$ with initial state $\Scal_1 = \cbr{S_0}$,  $A$ states at the second time step denoted as $\Scal_2 = \{S_1,S_2,\cdots S_{A}\}$, and $A^2$ states at the third time step denoted as $\Scal_3 = \{S_{1,1},S_{1,2},\cdots S_{1,A},S_{2,1}, \cdots, S_{A,A}\}$. We define state space $\Scal =\Scal_1 \cup \Scal_2 \cup \Scal_3$. Define action space $\Acal = \{a_1,a_2,\cdots, a_{A}\}$, initial distribution $\rho(S_0) = 1$ and deterministic transition dynamics $P_1(S_i | S_0,a_i) = 1$, $P_2(S_{i,j}|S_i,a_j) = 1$ for all $i,j \in [A]$. Define cost function $c(s,a) := \bar{c}(s)$, where\footnote{Strictly speaking, the cost should be within $[0,1]$; this can be achieved by shifting the cost function by $1$ and dividing by $\lambda +1$. Without affecting the correctness of the proof, we set the cost this way for the simplicity of presentation.}

%% \chicheng{Why was the above commented out? This is important?}
%
% and translation of the cost function
\[
\bar{c}(s) := \left\{
             \begin{array}{ll}
             0 , & s \in \Scal_1 \cup \Scal_2
             \\
               0 , & s = S_{i,j}, \text{where } i = A \text{ or }  j = A\\
             -V_{j,i} , & s = S_{i,j}, \text{where }  i \in [m]+m , j \in [m]\\
             -W_{i,j}, &  s = S_{i,j}, \text{where } i \in [m] , j \in [m] + m\\
             \lambda,  & \text{otherwise}
             \end{array}
\right. 
\]
where we denote $[m]+m := \{m+1,\cdots,2m\}$ and set $\lambda = 54$. 
% \edit{Since $c(s,a)$ depends only on $s$, from now on, we use $c(s)$ to denote $c(s,a)$.}{}

On the policy side, we define benchmark policy class $\Bcal$ that contains $A-1 = 2m$ deterministic policies $h_j $ ($j \in [A-1]$), such that $\forall s \in \Scal$, $h_j(s) = a_j$. Define deterministic expert policy $\pi^E$ as: $\pi^E(s)  = a_{A}$, $\forall s \in \Scal$. Note that action $a_{A}$ is never a choice for policies in $\Bcal$ but always chosen by the expert. Recall the mixed policy set $\Pi_{\Bcal}$ is defined as
\[
\Pi_\Bcal := \cbr{\pi_u(\cdot|s) := \sum_{h \in \Bcal} u\coord{h} \cdot h(\cdot|s): u \in \Delta(\Bcal) }
.
\]
For the expert feedback function $\feedback$, we use advantage function $A^E(s,a) = Q_{\pi^E}(s,a)-V_{\pi^E}(s)$. The values of $A^E(s,a)$ are calculated as follows:
%First, we note that, since the expert policy $\pi^E$ always chooses action $a_{A+1}$ , executing $\pi^E$ starting from any state $s \in \Scal_1 \cup \Scal_2$ ends in state $S_{i,j}$ at last step with $j = A+1 $ that induces $0$ cost by the definition of cost function $r$. Note that all states $ s \in \Scal_1\cup \Scal_2 \}$ have that $0$ cost.  This implies, $V_{\pi^E}(s) = 0$ for all $s \in \Scal_1 \cup \Scal_2$. 
\begin{itemize}
    \item 
    For $s$ in $\Scal_3$: since for every $a$, $Q_{\pi^E} (s,a) = c(s,a) = \bar{c}(s) = 0$, we have
    $A^E (s,a) = 0$.
    \item 
    For $s$ in $\Scal_2$:
    \begin{itemize}
        \item First, suppose $s = S_i$
    for $i \in \cbr{1,\ldots,A-1}$.
    Following $\pi^E$ directs the agent to $S_{i,A}$, which encounters zero subsequent cost. This implies that $V_{\pi^E}(S_i) = 0$. On the other hand, taking action $a_j$ transitions to $S_{i,j}$ which encounters cost $\bar{c}(S_{i,j})$ subsequently. This means that $A^E (S_i,a_j) = \bar{c}(S_{i,j})$. Recall that $\forall i \in [A]$ $\bar{c}(S_{i,A}) = 0$, by this we have that $A^E (S_i,a_A) = 0$ for all $S \in \Scal_2$.
    \item Next, suppose s = $S_{A}$. Taking any action $a$ and following policy $\pi^E$ afterwards encounters zero subsequent cost, which implies that $Q_{\pi^E}(S_{A}, a) = 0$. This implies that $V_{\pi^E}(S_{A}) = 0$ and $A^E(S_{A}, a) = 0$. 
    %For any action $a$ taken
    \end{itemize}
    \item For $s$ in $\Scal_1$: at the initial state $S_0$, 
    taking any action $a$ and following $\pi^E$ afterwards takes the agents to state $S_{i,A}$ for some $i$, which has zero cost. This implies that  $Q_{\pi^E}(S_{A}, a) = 0$, $V_{\pi^E}(S_{A}) = 0$, and $A^E(S_{A}, a) = 0$.
\end{itemize}
In summary, we have 
\[
\feedback(s,a) := A^E(s,a) = \left\{
             \begin{array}{ll}
             0 , & s \in \Scal_1 \cup \Scal_3 \; \text{or} \; s = S_{A} \; \text{or} \; a = a_{A}
             \\
            %   0 , & s = S_{A}\\
             -V_{j,i} , & s = S_{i}, a = a_j, \text{where }  i \in [m]+m , j \in [m]\\
             -W_{i,j}, &  s = S_{i}, a = a_j, \text{where } i \in [m] , j \in [m] + m\\
             \lambda,  & \text{otherwise}
             \end{array}
\right. 
\]
In summary, given any \bimatrix instance $(U,V)$, \reducef returns a \vimdp instance $(\Mcal, \feedback,\Bcal)$.

\textbf{Map \reduceb.
% The solution of \vimdp can be transformed to a solution of $\epsilon$-MNE
} 
% \edit{The mapping form $u^*=(u^*_x,u^*_y)$ that satisfy \vimdp with $\reducef(V,W)$ to a solution of \bimatrix with $(V,W)$ is by}{
The map \reduceb is defined as: given $u^* = (u^*_x,u^*_y) \in \Delta(\Bcal)$, return
\[
\hat{x} = \frac{u^*_x}{\|u^*_x\|_1}, \quad \hat{y} = \frac{u^*_y}{\|u^*_y\|_1}.
\]

\paragraph{Polynomial-time computability of the reduction.} For map \reducef, given any \bimatrix instance $(V,W)$ with $V,W \in \mathbb [0,1]^{m \times m}$, map $\reducef(V,W)$ returns $(\Mcal,\feedback,\Bcal)$, where $\Mcal=(\Scal,\Acal,H,c,\rho,P)$ has $|\Scal| = (2m+1)^2+(2m+1)+1 = O(m^2)$ states, $|\Acal| = 2m+1 = O(m)$ actions, $H=3$, $|\Scal|\cdot|\Acal| = O(m^3)$ cost function values, $(2m+2)(2m+1) = O(m^2)$ values for the deterministic transition probability and one fixed initial distribution. Meanwhile, map \reducef returns $2m$ deterministic benchmark policies and $\feedback$ function with $|\Scal|\cdot|\Acal| = O(m^3)$ values. Combining the above, we conclude that \reducef runs in $O(m^3)$ time.

For map \reduceb, by its definition, it can be computed in $O(m)$ time. In all, \reducef and \reduceb are computable in polynomial-time with respect to $m$.

\paragraph{Correctness of the reduction.}
\begin{enumerate}
     \item If $(V,W)$ is a valid input of \bimatrix, then $V,W \in \mathbb [0,1]^{m \times m}$. Given any $V,W \in \mathbb [0,1]^{m \times m}$, by the definition of \reducef, it is straightforward to see \reducef constructs an discrete state-action episodic MDP with expert feedback $\feedback: \Scal \times \Acal \to \RR$ and deterministic policy class $\Bcal$. Thus, $\reducef(V,W)$ is a valid input of \vimdp.
    \item By Lemma~\ref{lem:reduction} (given below), if $u^*$ is a solution of \vimdp instance $\reducef(V,W)$, then $\reduceb(u^*)$ is also a solution of \bimatrix instance $(V,W)$.
    \item By Lemma~\ref{lem: Existence of  VI Solution}, any \vimdp instance has a solution.
\end{enumerate}
In conclusion, \bimatrix $\leq_{p}$ \vimdp, thus \vimdp is $\mathrm{PPAD}$-hard.
\end{proof}

% \edit{the solution of the above \vimdp can be transformed back to a solution of the original $\epsilon$-MNE problem.}{
\begin{lemma}
\label{lem:reduction}
$\reduceb(u^*)$ is a solution of \bimatrix instance $(V,W)$ if $u^*$ is a solution of \vimdp instance $\reducef(V,W)$.
\end{lemma}
% \chicheng{Upgrade the above statement to a lemma?}

\begin{proof}
Recall that $u^*$ the solution of a \vimdp instance $(\Mcal, \feedback,\Bcal)$ satisfies $\forall u \in \Delta(\Bcal)$, $\inner{\theta(u^*)}{u-u^*} \geq -(S+A+B)^{-6}$, where $S = |\Scal|, A = |\Acal|, B = |\Bcal|$. The proof follows by first calculating vector field $\theta(u)$ induced by $(\Mcal, \feedback,\Bcal)$, and then showing that $\reduceb(u^*)$ is a solution of the original \bimatrix instance $(V,W)$ if $u^*$ is a solution of \vimdp instance $\reducef(V,W)$.

\paragraph{To begin with, given $V,W \in [0,1]^{m \times m}$ we calculate $\theta$ in the \vimdp instance $\reducef(V,W)$.} We have by definition in Section~\ref{sec:improper-learning}, $\forall u \in \Delta(\Bcal)$, 
\[
\begin{aligned}
\theta(u) 
=
\rbr{ \EE_{s \sim d_{\pi_u}} \EE_{a \sim h_j(\cdot \mid s)} \sbr{ \feedback(s,a) } }_{h_j \in \Bcal} 
=
\frac{1}{3}\rbr{ \EE_{s \sim d^2_{\pi_u}} \sbr{ \feedback(s,h_j(s)) } }_{h_j \in \Bcal},
% =&
% \frac{1}{3}\rbr{ \EE_{s \sim d^2_{\pi_u}} \sbr{ Q_{\pi^E}(s,h_j(s)) } }_{\in [2m]}\\
\end{aligned}
\]
where in the second equality we recall that $d^2_{\pi_u}$ denotes the state occupancy distribution at the second step and $d_{\pi_u} = \frac{1}{3}(d^1_{\pi_u} + d^2_{\pi_u}+d^3_{\pi_u})$. The second equality is by the fact that 
% $h_j(s) = a_j$ as defined in the proof of Theorem~\ref{thm:vi-mdp-ppad} and
$\feedback(s,a)$ is always $0$ when $s \in \Scal_1 \cup \Scal_3$. 

% Furthermore, we compute $\rbr{ \EE_{s \sim d^2_{\pi_u}} \sbr{ \feedback(s,h_j(s)) } }_{\in [2m]}$'s value explicitly. 
Recall that $h_j(s) = a_j$ and $B = 2m$ as defined in the proof of Theorem~\ref{thm:vi-mdp-ppad}, it can be verified that $\pi_u(\cdot|s) = \sum_{j \in [2m]} u\coord{j} \cdot h_j(\cdot|s) = (u\coord{1},u\coord{2}, \cdots, u\coord{2m},0 )^\top $ , where the last entry is $0$ since the last action is never chosen by any $h_j \in \Bcal$. We can calculate its state occupancy distribution at step 2 as: $\Prob_{s \sim d^2_{\pi_u}}(s = S_i) = u\coord{i}$ for $i \in [2m]$ and $\Prob_{s \sim d^2_{\pi_u}}(s = S_{A}) = 0$. Therefore, $\forall j \in [2m]$,
\[
\EE_{s \sim d^2_{\pi_u}} \sbr{  \feedback(s,a_j) } 
=
\sum_{i=1}^{2m} u\coord{i}\feedback(S_i,a_j).
% =
% \sum_{i=1}^{A} u\coord{i} c(S_{i,k})
 \]
 For any $u \in \Delta(\Bcal)$, we use $u_x, u_y$ to denote the vector that consists of the first $m$ elements and the last $m$ elements of $u$ respectively. 
 Given $V,W \in [0,1]^{m \times m}$, define matrix
 \[
 C := 
 \begin{pmatrix}
    \lambda \one_{m \times m} & -V \\
    -W^\top & \lambda \one_{m \times m}
 \end{pmatrix}
 \in \RR^{2m \times 2m},
 \]
 where $\one_{m \times m}$ denotes the matrix whose entries are all 1's.
 Notice that by the value of $\feedback$ calculated in the proof of Theorem~\ref{thm:vi-mdp-ppad}, it can be verified that $\forall i \in [2m], \forall j \in [2m]$, $\feedback(S_i,a_j) = C_{i,j}$.
 With this, 
 $\theta(u)$ can be written in matrix form:
%  \EE_{s \sim d^2_{\pi_u}} \sbr{ Q_{\pi^E}(s,h_j(s)) }
 \begin{equation}
 \theta(u)
 =
 \frac{1}{3}\rbr{ \sum_{i=1}^{2m} u\coord{i}\feedback(S_i,a_j) }_{j\in[2m]}
 = 
 \frac13
  \begin{pmatrix}
    \lambda \one_{m \times m} & -V \\
    -W^\top & \lambda \one_{m \times m}
 \end{pmatrix}
 \cdot 
  \begin{pmatrix}
   u_x \\
   u_y
 \end{pmatrix}
 = 
 \frac13 C u.
 \label{eq: vi-mdp-function}
 \end{equation}

%the
%This formulates
% (induced by 2-player $\epsilon$-MNE)
%a policy 
Therefore, the constructed discrete state-action episodic MDP $\Mcal$, expert feedback $\feedback$ and benchmark policy class $\Bcal$ induces the following instance of
\vimdp: find $u^* \in \Delta(\Bcal)$ such that $\forall u \in \Delta(\Bcal)$,
\begin{equation} 
\left\langle \theta(u^*), u-u^* \right\rangle 
% =  -(A^2 + 3A)^{-6} 
=
\inner{
\frac13 C u^*
}{u - u^*} 
\geq 
-(S+A+B)^{-6}
=
-(4m^2+ 10m + 4)^{-6},
\label{eqn:vi-mdp-matrix-form}
\end{equation}
where we recall that $S = (2m+1)^2+(2m+1)+1$, $A = 2m+1$, and $B = 2m$. 
% \chicheng{Perhaps it is better to write $S$, $A$, and $B$ all in terms of the original problem size $m$}
% \frac{2}{27}

 \paragraph{Next, we show that $\reduceb(u^*)$ is a solution of \bimatrix instance $(V,W)$.} 
 Recall that in Section~\ref{sec:two-player-game}, given $V,W \in [0,1]^{m \times m}$, the mixed strategies on set $\Acal_x$ and sets $\Acal_y$ and $|\Acal_x| = |\Acal_y| = m$ are represented by $x \in \Delta(\Acal_x)$ and $y \in \Delta(\Acal_y)$, respectively. If $(\hat{x},\hat{y}) \in \Delta(\Acal_x) \times \Delta(\Acal_y)$ is a solution of \bimatrix instance $(V,W)$, then $(\hat{x},\hat{y})$ satisfies for any $x \in \Delta(\Acal_x)$ and $y \in \Delta(\Acal_y)$,
\[
\left\{
\begin{aligned}
\hat{x}^\top V\hat{y} \geq x^\top V\hat{y}-m^{-12}, \\
\hat{x}^\top W\hat{y} \geq \hat{x}^\top  W y-m^{-12}.
\end{aligned}
\right. 
\]
 
Now, consider $u^* = (u^*_x,u^*_y) \in \Delta(\Bcal)$, a solution for \vimdp instance $\reducef(V,W)$,  such that $\forall u \in \Delta(\Bcal)$,
\[     \left\langle \theta(u^*), u-u^* \right\rangle \geq -(4m^2+ 10m + 4)^{-6} = -\epsilon,\]
where we use $\epsilon $ to denote $(4m^2+ 10m + 4)^{-6}$.

We will show that $(\hat{x}, \hat{y}) = \reduceb( u^* ) = (\frac{u_x^*}{\| u_x^* \|_1}, \frac{u_y^*}{\| u_y^* \|_1})$  
is a solution of \bimatrix instance $(V,W)$.
% $m^{-12}$-approximate 2-player Nash equilibrium of the bimatrix game $(V, W)$.
To see this, we first prove that $\forall x \in \Delta(\Acal_x)$,  
\begin{equation}
{\hat{x}}^\top V \hat{y} \geq x^\top V \hat{y}-m^{-12}.
\label{eqn:mne-x}
\end{equation}
To this end, $\forall x \in  \Delta(\Acal_x)$, by setting 
$u = (\| u^*_x \|_1 \cdot x, u^*_y ) \in \Delta(\Bcal)$ and plugging this choice of $u$ into Equation~\eqref{eqn:vi-mdp-matrix-form}, we have
\begin{align*}
-\epsilon
\leq &
\inner{\frac{1}{3}Cu^*}{u - u^*} \\
= & 
 \frac{\| u_x^* \|_1}{3}
\begin{pmatrix}
(x - \hat{x})^\top, \zero^\top   
\end{pmatrix}
\begin{pmatrix}
    \lambda \one_{m \times m} & -V \\
    -W^\top & \lambda \one_{m \times m}
 \end{pmatrix}
 \cdot 
  \begin{pmatrix}
   u_x^* \\
   u_y^*
 \end{pmatrix} 
 \\
 =&  \frac{\| u_x^* \|_1}{3} \rbr{ \lambda(x - \hat{x})^\top \one_{m \times m} u_x^* - (x -\hat{x})^\top V u_y^*}
 \\
 = &  
  \frac{\lambda \| u_x^* \|_1}{3}  \zero^\top u_x^* - \frac{\| u_x^* \|_1}{3} (x -\hat{x})^\top V u_y^*
 \\
  = &  
 -  \frac{\| u_x^* \|_1 \| u_y^* \|_1}{3} (x -\hat{x})^\top V \hat{y},
\end{align*}
where we use $\zero$ to denote the all $0$ vector in $\RR^m$. Combining $ \| u_x^* \|_1  \cdot \| u_y^* \|_1 \geq \frac{2}{9}$ as shown later by Lemma~\ref{lem:vi-approximate-limit}, we obtain 
\[
x^\top V \hat{y} -\hat{x}^\top V \hat{y}
\leq \frac{3}{\| u_x^* \|_1 \| u_y^* \|_1 } \epsilon
\leq 
\frac{27}{2}\epsilon
=
\frac{27}{2(4m^2+ 10m + 4)^{6}}
\leq
m^{-12}.
\]
% where the last inequality uses Equation~\eqref{eqn:u-norm-bounds}. 
This establishes Equation~\eqref{eqn:mne-x}.
Using a symmetrical argument,
for any $y \in \Delta(\Acal_y)$,
by taking $u = (u^*_x, \| u^*_y\|_1 \cdot y)$, 
we can also show that $\forall y \in \Delta(\Acal_y)$,
\begin{equation}
    {\hat{x}}^\top W \hat{y} \geq \hat{x}^\top W y -m^{-12}.
    \label{eqn:mne-y}
\end{equation} 
Combining Equations~\eqref{eqn:mne-x} and~\eqref{eqn:mne-y}, we conclude that $(\hat{x}, \hat{y})$ 
is a solution of \bimatrix instance $(V,W)$.
% $m^{-12}$-approximate 2-player Nash equilibrium of the bimatrix game $(V, W)$.
% \edit{\textbf{Any 2-player $\epsilon$-MNE problem ($\epsilon <9$) can be reduced to a $\frac{2}{27}\epsilon$-VI-MDP problem. }}{
\end{proof}

\begin{lemma}
 \label{lem:vi-approximate-limit}
 $\forall V,W \in [0,1]^{m \times m}$, if $u^* = (u_x^*, u_y^*)$ is a solution of \vimdp instance $\reducef(V,W)$, then
 \[
 \| u_x^* \|_1  \cdot \| u_y^* \|_1 \geq \frac{2}{9}.
 \]
 \end{lemma}
%  \chicheng{Move the label to the beginning of the lemma to avoid empty line.}
 
 \begin{proof}
The lemma is proved by showing
$\forall \; V,W \in [0,1]^{m \times m}$,  $\forall u = (u_x,u_y) \in \Delta(\Bcal)$ such that $\| u_x \|_1 \notin (\frac13, \frac23)$, for any $v = (v_x,v_y) \in \Delta(\Bcal)$ that satisfies $\| v_x \|_1 = \| v_y \|_1 = \frac12$,
 \[
\left\langle\theta(u), v-u \right\rangle
 \leq 
 -\frac{1}{3}
 < 
 -(4m^2+ 10m + 4)^{-6},
 \]
 where we recall the definition of $\theta(\cdot)$ from Equation~\eqref{eq: vi-mdp-function}. This implies $u^* = (u^*_x, u^*_y)$
%  any $u \in \Delta(\Bcal)$ such that $\| u_x \|_1 \notin (\frac13, \frac23)$ is not a solution of \vimdp instance $\reducef(V,W)$. Thus, $u^* = (u_x^*, u_y^*)$
 the solution of \vimdp instance $\reducef(V,W)$ satisfies $\| u^*_x \|_1, \| u^*_x \|_1 \in [\frac13, \frac23]$, $\forall \; V,W \in [0,1]^{m \times m}$. Finally, it can be easily verified that $\forall \; a,b \in [\frac13, \frac23]$, $ab \geq \frac{2}{9}$.

By the alternative expression of $\theta(\cdot)$ shown in Equation~\eqref{eq: vi-mdp-function}, we can write  $\inner{\theta(u)}{u}$ and $\inner{\theta(u)}{v}$ with $u = (u_x,u_y)$ and $v = (v_x,v_y)$ as:
\[
\begin{aligned}
\inner{\theta(u)}{u} = &
\frac{1}{3}(- u_x^\top  V u_y -  u_x^\top  W u_y +   \lambda(\|u_x\|_1^2 + \|u_y\|_1^2)),\\
\inner{\theta(u)}{v} = &
\frac{1}{3}(  -  v_x^\top  V u_y- u_x^\top  W v_y +   \lambda(\|u_x\|_1\|v_x\|_1 + \|u_y\|_1\|v_y\|_1) ).
\end{aligned}
\]
By algebra, $\forall\; x\in \Delta(\Acal_x)$, $\forall\; y \in \Delta(\Acal_y)$ for $|\Acal_x|=|\Acal_y|=m$, $\forall \; V,W \in [0,1]^{m \times m}$,
$0 \leq x^\top V y, \; x^\top W y \leq 1$, then
\[
\begin{aligned}
\inner{\theta(u)}{v - u} 
=&\frac{1}{3} \rbr{
 u_x^\top  V u_y +  u_x^\top  W u_y
-  v_x^\top  V u_y- u_x^\top  W v_y }\\
% V( u_x - v_x, u_y ) +W( u_x, u_y - v_y) \\
&+ 
\frac{1}{3} \rbr{\lambda \| u_x \|_1 ( \| v_x \|_1 - \| u_x \|_1 )
 + \lambda \| u_y \|_1 ( \| v_y \|_1 - \| u_y \|_1)}\\
 \leq &
 \frac{2}{3} + 18 \rbr{ \| u_x \|_1 ( \| v_x \|_1 - \| u_x \|_1 )
 + \| u_y \|_1 ( \| v_y \|_1 - \| u_y \|_1)},\\
%  \geq & -\epsilon.
 \end{aligned}
\]
where we recall that $\lambda = 54$. 
Therefore, consider $\| u_x \| \notin (\frac13, \frac23)$ and any $v$ such that $\| v_x \| = \| v_y \| = \frac12$, we have:
\begin{align*}
   & \| u_x \|_1 ( \| v_x \|_1 - \| u_x \|_1 )
    + \| u_y \|_1 ( \| v_y \|_1 - \| u_y \|_1 ) \\
    =&
     \| u_x \|_1 ( \frac{1}{2} - \| u_x \|_1 )
    + \| u_y \|_1 ( \frac{1}{2} - \| u_y \|_1 ) \\
    =&
    -2 (\| u_x \|_1  - \frac{1}{2})^2
    < -\frac1{18}.
\end{align*}
Plugging this back, we conclude that $\forall u= (u_x, u_y) \in \Delta(\Bcal)$ such that $ \| u_x \| \notin (\frac13, \frac23)$, for any $v = (v_x, v_y) \in \Delta(\Bcal)$ such that $\| v_x \| = \| v_y \| = \frac12$,
\[
\inner{\theta(u)}{v - u} 
\leq \frac{2}{3}-1 = -\frac{1}{3}.
\qedhere
\]
\end{proof}

\subsection{
Computational hardness of achieving sublinear dynamic regret in \algreduct}
\begin{theorem}[Restatement of Theorem~\ref{thm:hardness-dreg}]
\label{thm:hardness-dreg-restated}
Fix $\const>0$, if there exist a COIL algorithm such that for any $\Mcal$ and expert $\pi^E$, it interacts with $\Mcal$, CSC oracle $\Ocal$, expert feedback $\feedback(s,a) = A^E(s,a)$, and outputs a sequence of $\cbr{ \pi_{u_n}}_{n=1}^N \in \Pi_\Bcal$ such that with probability at least $\frac{1}{2}$,
\[
\DReg_N(\Bcal)
% ( \cbr{ \pi_{u_n} }_{n=1}^N )
\leq
O(\poly(S, A, B) \cdot N^{1-\const}),
\]
in $\poly(N, S, A, B)$ time, then all problems in $\mathrm{PPAD}$ are solvable in randomized polynomial time.
%$\subseteq \mathrm{RP}$.
%\mathrm{FBPP}
%\chicheng{the capitalized ``bimatrix'' here looks bizarre}
\end{theorem}
 \begin{proof}
% By Theorem~\ref{thm:vi-mdp-ppad}, all problems in $\mathrm{PPAD}$ are polynomial-time reducible to \vimdp. If there exists a algorithm that solves \vimdp in randomized polynomial time, then all problems in $\mathrm{PPAD}$ can be solved in randomized polynomial time.
% to show that all problems in $\mathrm{PPAD}$ can be solved in randomized polynomial time, it suffices to show that the \vimdp problem 
%  can be solved in randomized polynomial time. 
We start the proof by showing if  
% this, assuming the existence of a polynomial time sublinear dynamic regret imitation learning algorithm that outputs policies in $\Pi_\Bcal$.
% \chicheng{The fonts for p,q here are the same as f, g in the reduction, which is a bit confusing. I actually think regular font is better.}
there exists a COIL algorithm $\algo$ that achieves $\funsab(S,A,B) N^{1-\const}$ dynamic regret with probability $\frac12$ in time $\funnsab(N, S, A, B)$,
 where $\funsab$ and $\funnsab$ are polynomial functions, then, $\algo$
 yields an algorithm $\algop$ (Algorithm~\ref{alg:algoop}) that solves \vimdp with expected polynomial time $\poly(S,A,B)$. 

%Define \algop \red{as in }
\begin{algorithm}
\caption{\algop}
\label{alg:algoop}
\begin{algorithmic}[1]
\WHILE{\TRUE}
\STATE Run $\algo$ on $\Mcal$, $\Bcal$, $\feedback$, $\Ocal$ for $N = ( \funsab(S, A, B)\cdot (S+A+B)^{6})^{\frac1\const}$ rounds, obtaining a sequence of policies $\cbr{\pi_n}_{n=1}^N$ parameterized by $\cbr{u_n}_{n=1}^N$.
\STATE If any of $u_n$ is a 
        %  $\frac{1}{m^{12}}$-MNE of $(V,W)$
         $(S+A+B)^{-6}$-approximate solution of $\vi(\Delta(\Bcal),\theta)$, return $u_n$.
\ENDWHILE
\end{algorithmic}
\end{algorithm}

\paragraph{Correctness.} As $\algop$ return only if $u_n$ is an  $(S+A+B)^{-6}$-approximate solution of $\vi(\Delta(\Bcal),\theta)$, it solves the \vimdp problem. 

\paragraph{Time complexity.} We now bound the time complexity of $\algop$.
Note that \algop sets \algo with $N = (\funsab(S,A,B) \cdot (S+A+B)^{6})^{\frac 1 \const} = \poly(S+A+B)$, and 
\algo has a running time of  $\poly(N, S, A, B) = \poly(S, A, B)$, together, each iteration of \algop takes $\poly(S, A, B)$ time.

We now show that \algop  runs for 
an expected number of iterations at most a constant. Specifically, the guarantees of \algo implies that for each iteration, 
with probability at least $\frac{1}{2}$,
\[
\begin{aligned}
\DReg_N(\Bcal)( \cbr{ \pi_{u_n} }_{n=1}^N )
=&
\sum_{n=1}^N\iterloss_n(\pi_{u_n})-\min_{\pi \in \Bcal}\iterloss_n(\pi)\\
=&
%\sum_{n=1}^N\iterloss_n(\pi_{u_n}) - 
%\min_{u \in \Delta(\Bcal)}
%\iterloss_n(\pi_u) \\
%=&
\sum_{n=1}^N \inner {\theta(u_n)}{u_n } - \min_{u \in \Delta(\Bcal)}\inner {\theta(u_n)}{u} \\
=&
\sum_{n=1}^N \max_{u \in \Delta(\Bcal)}\inner {\theta(u_n)}{u_n - u}\\
\leq&
f(S, A, B) \cdot N^{1-\const},
\end{aligned}
\]
where the first equality is from the definition of dynamic regret in Section~\ref{sec:prelims}, and the second is by $\iterloss_n(\pi_u) = \inner{\theta(u_n)}{u}$. In this event, since $N = (\funsab(S,A,B) \cdot (S+A+B)^{6})^{\frac 1 \const}$, we have that $\exists n \in [N]$ s.t. 
\[
\max_{u \in \Delta(\Bcal)}\inner {\theta(u_n)}{u_n - u} 
\leq
\funsab(S, A, B) \cdot N^{-\const}
=
(S+A+B)^{-6},
\]
which means $u_n$ is a solution of \vimdp instance $(\Mcal, \feedback,\Bcal)$.
Hence, the expected number of iterations of running $\algo$ before a valid solution being returned is smaller or equal to $\sum_{t=0}^\infty \rbr{\frac{1}{2}}^{t} = 2$. Thus, the expected running time of $\algop$ is $O(\poly(S+A+B))$. 

% \edit{
% In summary, if there exist a COIL algorithm \algo that achieves sublinear dynamic regret in \algreduct, then a randomized polynomial time algorithm \algop can be constructed to  
% solve \vimdp. 
% By Theorem~\ref{thm:vi-mdp-ppad}, all problems in $\mathrm{PPAD}$ are polynomial-time reducible to \vimdp, then all problems in $\mathrm{PPAD}$ are solvable in randomized polynomial time.}
By Theorem~\ref{thm:vi-mdp-ppad}, all problems in $\mathrm{PPAD}$ are polynomial-time reducible to \vimdp. If there exist a COIL algorithm \algo that achieves sublinear dynamic regret in \algreduct, then \algop can be constructed to solve \vimdp in randomized polynomial time, which means all problems in $\mathrm{PPAD}$ are solvable in randomized polynomial time.
% \edit{, which is$\mathrm{PPAD}$-$\mathrm{hard}$and believed to be computationally intractable. }{(Reword to connect to the conclusion of the theorem)}
\end{proof}

\section{Online linear optimization results }
\label{sec:ftrl}
% \paragraph{Proposed changes}
% \chicheng{Here are my proposed changes. First,
% let's create a subsection on ``general FTRL results'' that presents the following algorithm and lemma.}
In this section, we first provide a recap on online linear optimization, the well-known Optimistic Follow the Regularized Leader (FTRL) algorithm (Algorithm~\ref{alg: Optimistic FTRL})~\cite{rakhlin2013online}, and its regret guarantees (Theorem~\ref{thm:optimistic-ftrl}). 
Section~\ref{sec:ftrl-ex-reg} instantiates this general result with the regularizer $R_\Ncal$ (Equation~\eqref{eqn:r}) defined in Section~\ref{sec:reg-sqrt-n-deferred}.

\subsection{Basic facts on convex analysis}
Before we delve into the optimistic FTRL algorithm, we state some useful definitions and facts from convex analysis. 

%that help simplify our presentation. 

%$v, w \in \RR^d$ and 
\begin{definition}
For a differentiable convex function $f: \RR^d \to \RR $, define
$D_f(v, w) 
=
f(v)
-
f(w)
-\left\langle v -w, \nabla f (w)  \right\rangle$
to be the Bregman divergence induced by $f$.
\end{definition}

% \begin{definition}
% Given a convex function $f: \Omega \to \RR \cup \cbr{+ \infty}$, where $\Omega \subseteq \RR^d$, define
% $f^*: \RR^d \to \RR \cup \cbr{+ \infty}$ as $f^*(\theta) 
% :=
% \sup_{w \in \Omega} \rbr{ \inner{\theta}{w} - f(w) } $
% to be its Fechel conjugate.
% \end{definition}

\begin{definition}
Given a convex function $f: \RR^d \to \RR \cup \cbr{+ \infty}$, where $\Omega \subseteq \RR^d$, define
$f^*: \RR^d \to \RR \cup \cbr{+ \infty}$ as $f^*(\theta) 
:=
\sup_{w \in \RR^d} \rbr{ \inner{\theta}{w} - f(w) } $
to be its Fechel conjugate.
\end{definition}
% \begin{definition}
% A convex function $f: \RR^d \rightarrow \RR\cup \cbr{+ \infty}$ is said to be closed if for each $\gamma \in \mathbb{R}$, the sublevel set $\{w \in \dom(f) \mid f(w) \leq \gamma\}$ is a closed set.
% \end{definition}

\begin{definition}
A convex function $f: \RR^d \to \RR\cup \cbr{+ \infty}$ is proper if it is not identically equal to $+\infty$.
\end{definition}

\begin{definition}
A function $f: \RR^d \rightarrow \RR \cup \cbr{+ \infty}$ is $\alpha$-strongly convex w.r.t. a norm $\|\cdot\|$ if for all $v, w$ in the relative interior of the effective domain of $f$ and $\lambda \in(0,1)$ we have
\[
f(\lambda v+(1-\lambda) w) \leq \lambda f(v)+(1-\lambda) f(w)-\frac{1}{2} \alpha \lambda(1-\lambda)\|v-w\|^2 .
\]
\end{definition}

\begin{definition}
A function $f: \RR^d \rightarrow \RR$ is $\beta$-strongly smooth w.r.t. a norm $\|\cdot\|$ if $f$ is everywhere differentiable and if for all $v, w \in \RR^d$ we have
\[
f(v+w) \leq f(v)+\langle\nabla f(v), w\rangle+\frac{1}{2} \beta\|w\|^2 .
\]
\end{definition}

\begin{definition}
For a convex function $f: \RR^d \to \RR \cup \cbr{+\infty}$, define its effective domain $\dom (f) := \cbr{w \in \RR^d: f(w) < +\infty}$. 
\end{definition}

Note that the effective domain of a strongly smooth function $f$ satisfies $\dom (f) = \RR^d$.

%:= \cbr{w \in \RR^d: f(w) < +\infty}
% \chicheng{$\dom(f)$ looks better, I think}

% \chicheng{Can we use this? A proof of this can be found in e.g.} 
\begin{fact}
\label{fact:conj}
Let $f: \RR^d \to \RR\cup \cbr{+ \infty}$ be a proper, closed and convex function, then: 
% \chicheng{Need to check the properness of function $f$}
\begin{enumerate} 
     \item $f^*$ is closed and convex (\cite{rockafellar1970convex}[Theorem 12.2]);
     \label{item:closed_convex}
     \item $f^{**} = f$ (\cite{rockafellar1970convex}[Corollary 12.2.1]);
    %  (Theorem 4.2.1 of~\cite{borwein2006convex});
     \label{item:one_to_one}
    %  (\url{http://www.seas.ucla.edu/~vandenbe/236C/lectures/conj.pdf}).
    %  \chicheng{Rockafellar, Section 12?}
    \item $\theta \in \partial f(w) \Leftrightarrow f(w) + f^*(\theta) = \inner{\theta}{w} \Leftrightarrow w \in \partial f^*(\theta)$, where $\partial g(w)$ denotes $g$'s subdifferential set at $w$ 
    (\cite{zalinescu2002convex}[Theorem 2.4.2]);
    % (Lemma 11 of~\cite{shalev2007online});\red{need to change}
    \label{item:fenchel-young-equality}

     \item  $f$ is $\alpha$-strongly convex with respect to a norm $\|\cdot\|$ if and only if $f^{\star}$ is $\frac{1}{\alpha}$-strongly smooth with respect to the dual norm $\|\cdot\|_{\star}$ (\cite{kakade2012regularization}[Theorem 3]).
     \label{item:sc_strongly smooth_duality}

    %\item \blue{$f^*$ is closed, convex, and differentiable, and $\nabla f^*(\theta) = \argmax_{w \in \Omega} \rbr{ \inner{\theta}{w} - f(w) }$.}
\end{enumerate}
\end{fact}
% \chicheng{Double check reference}

\begin{proposition}
\label{prop:sc-conj}
Let $f: \RR^d \rightarrow \RR \cup \cbr{+ \infty}$ be proper, closed and $\alpha$-strongly convex with respect to $\| \cdot \|$, then $f^*$ is differentiable, and $\nabla f^*(\theta) = \argmax_{w \in \dom (f)} \rbr{ \inner{\theta}{w} - f(w)}.$
\end{proposition}

\begin{proof}[Proof sketch]
Given a function $f$ that is proper, closed and strongly convex, define $f_1: \dom(f) \to \RR$, where $f_1(w) := f(w)$ on $\dom(f)$. 
It can be seen that $f_1: \Omega \rightarrow \RR$ is closed and strongly convex. It can be checked that $f^*(\theta) = \sup_{w \in \RR^d} \rbr{ \inner{\theta}{w} - f(w) }
=\sup_{w \in \dom (f)} \rbr{ \inner{\theta}{w} - f(w)} = f_1^*(\theta)$ where $f_1^*$ is defined using the notations in~\cite{shalev2007online}. The proposition follows from ~\cite{shalev2007online}[Lemma 15].
\end{proof}

\subsection{General results on FTRL and Optimistic FTRL}

Online linear optimization refers to the following $N$-round protocol: the learner is given a convex decision set $\Omega \subset \RR^d$. At every round $n \in [N]$, the learner chooses some decision $u_n \in \Omega$, and then receives a linear loss $\inner{g_n}{\cdot}$, where $g_n \in \RR^d$. The goal of the learner is to minimize its regret on this sequence of linear losses:
\[
\SLReg_N := \sum_{n=1}^N \left\langle \est_n, u_n\right\rangle -  \min_{u \in \Omega} \sum_{n=1}^N  \left\langle \est_n, u\right\rangle.
\]

Optimistic FTRL (Algorithm~\ref{alg: Optimistic FTRL}), works for online linear optimization with a general decision set $\Omega$. It
takes into input a strongly convex regularizer $R$ with effective domain $\Omega$ and a learning rate $\eta>0$. It maintains the cumulative linear loss $\Theta_n = \sum_{i=1}^n g_i$ over time (line~\ref{line:update-theta}); at round $n$, it first uses a predicted instantaneous loss $\hat{g}_n$ to construct a guess on the cumulative loss $\hat{\Theta}_n$ (line~\ref{line:predict}), then 
chooses the decision $u_n$ that minimizes the regularized guessed cumulative linear loss $\eta \langle \hat{\Theta}_n, {u} \rangle + R(u)$ (line~\ref{line:oftrl}), which by Proposition~\ref{prop:sc-conj}, has an equivalent form of $\nabla R^*(-\eta \hat{\Theta}_n)$. 
We have the following guarantee on the regret of Optimistic FTRL; it is largely inspired by and slightly generalizes the results of~\cite{shalev2011online,abernethy2014online,rakhlin2013online}.
%Our proof is largely inspired by

\begin{algorithm}
\begin{algorithmic}[1]
\REQUIRE Convex decision set $\Omega \subseteq \RR^d$,
regularizer $R$ that is closed and $\alpha$-strongly convex with bounded $\dom (R) = \Omega$,
% \chicheng{Perhaps we should have that $\mathrm{dom}(R) = \Omega$ instead?}
learning rate $\eta > 0$.
%\red{and $+\infty$ elsewhere}.
% \subseteq \RR^d
\STATE  Initialize $\Theta_0 = 0$.
\FOR{$n=1,2,\ldots,N$}
\STATE  The learner makes prediction $\pred_n$, $\hat{\Theta}_n = \Theta_{n-1} + \pred_n$ .
\label{line:predict}
% \subseteq \RR^d
\STATE The learner plays $u_n  = \argmin_{u \in \Omega} \rbr{ \eta \langle \hat{\Theta}_n, {u} \rangle + R(u) }
= 
\nabla R^*(-\eta \hat{\Theta}_{n})
$  .
\label{line:oftrl}
%\chicheng{Use notation $R_\Omega^*$}
\STATE The learner receives real loss $\est_n$.
\STATE Update $\Theta_n = \Theta_{n-1} + \est_n$ .
\label{line:update-theta}
\ENDFOR
\end{algorithmic}
\caption{Optimistic FTRL}
\label{alg: Optimistic FTRL}
\end{algorithm}

\begin{theorem}
\label{thm:optimistic-ftrl}
Let $R: \RR^d \rightarrow \RR \cup \cbr{+ \infty}$ be a closed and $\alpha$-strongly convex function with respect to $\| \cdot \|$, such that $\Omega=\dom (R)$. The linear regret of Optimistic FTRL with $R$ and learning rate $\eta > 0$, 
satisfies the following:
\begin{align}
\SLReg_N
\leq &
\frac{\sup_{w \in \Omega} R(w)
- 
\inf_{w \in \Omega} R(w)}{\eta}
+
 \frac{1}{\eta}\sum_{n=1}^{N}
 \rbr{
\underbrace{D_{R^*}\left(-\eta \Theta_{n}, -\eta \hat{\Theta}_{n}\right)}_{\text {divergence penalty}} 
-
\underbrace{D_{R^*}\left(-\eta\Theta_{n-1}, -\eta\hat{\Theta}_{n}\right)}_{\text {prediction gain}}}
\label{eqn:oftrl-1}
 \\
\leq &
\frac{\sup_{w \in \Omega} R(w)
- 
\inf_{w \in \Omega} R(w)}{\eta}
+
\sum_{n=1}^{N}
\frac{\eta \| \hat{\est}_n - \est_n \|_*^2}{2\alpha}
-
\frac{\alpha}{2\eta} \| \nabla R^*(-\eta\Theta_{n-1}) - \nabla R^*(-\eta \hat{\Theta}_{n}) \|^2.
\label{eqn:oftrl-2}
\end{align}
Specifically, if $\hat{g}_n = 0$ for all $n$, 
\[
\SLReg_N
\leq
\frac{\sup_{w \in \Omega} R(w)
- 
\inf_{w \in \Omega} R(w)}{\eta}
+
\sum_{n=1}^{N}
\frac{\eta \| \est_n \|_*^2}{2\alpha}.
\]
\end{theorem}

%%\label{lem: Regret of GBPA  with Prediction}
% \chicheng{Todo for both of us: Double check~\cite{rakhlin2013online,chiang2012online,syrgkanis2015fast} to see how much overlap is this proof with theirs. For~\cite{rakhlin2013online}, check both their journal and conference versions to make sure we don't miss their optimistic FTRL analysis. (6/22) Also, double check \url{https://www.jmlr.org/papers/volume18/14-428/14-428.pdf} }
%}{~\cite{} \edit{
\begin{proof}

By the definition of Bregman divergence, we have 
$$
    \left\{
\begin{aligned}
D_{R^*}(-\eta \Theta_{n},-\eta\hat{\Theta}_{n})
&=
R^*(-\eta\Theta_{n})
-
R^*(-\eta\hat{\Theta}_{n})
-
\left\langle -\eta\Theta_{n} +\eta\hat{\Theta}_{n}, \nabla R^*(-\eta\hat{\Theta}_{n}) \right\rangle  , \\
D_{R^*}(-\eta\Theta_{n-1},-\eta\hat{\Theta}_{n})
&=
R^*(-\eta\Theta_{n-1})
-
R^*(-\eta\hat{\Theta}_{n})
-
\left\langle-\eta\Theta_{n-1} +\eta\hat{\Theta}_{n}, \nabla R^*(-\eta\hat{\Theta}_{n}) \right\rangle  .
\end{aligned}
\right.
    $$
    By rearranging the terms, the two equations can be rewritten as
    $$
    \left\{
\begin{aligned}
\left\langle\eta\Theta_{n} -\eta\hat{\Theta}_{n}, \nabla R^*(-\eta\hat{\Theta}_{n}) \right\rangle
&=
-
R^*(-\eta\Theta_{n})
+
R^*(-\eta\hat{\Theta}_{n})
+
D_{R^*}(-\eta\Theta_{n},-\eta\hat{\Theta}_{n}) ,
\\
-\left\langle \eta\Theta_{n-1} -\eta\hat{\Theta}_{n}, \nabla R^*(-\eta\hat{\Theta}_{n})\right\rangle
&=
R^*(-\eta\Theta_{n-1})
-
R^*(-\eta\hat{\Theta}_{n})
-
D_{R^*}(-\eta\Theta_{n-1},-\eta\hat{\Theta}_{n}) .
\end{aligned}
\right.
  $$
  By adding the two equations and recall that $\Theta_n = \Theta_{n-1}+\est_n$ and $u_n = \nabla R^*(-\hat{\Theta}_n)$, we can write
  \[
  \begin{aligned}
  \eta \left\langle \est_{n}, u_n \right\rangle
  =&
  \left\langle  \eta\Theta_{n} -\eta\Theta_{n-1}, \nabla R^*(-\eta\hat{\Theta}_{n}) \right\rangle\\
  =&
  R^*(-\eta\Theta_{n-1})
  -
  R^*(-\eta\Theta_{n})
  +
  D_{R^*}(-\eta\Theta_{n},-\eta\hat{\Theta}_{n})
  -
  D_{R^*}(-\eta\Theta_{n-1},-\eta\hat{\Theta}_{n}).
  \end{aligned}
  \]
  Summing over $n=1,\ldots,N$, we have 
  \[
  \eta \sum_{n=1}^N \inner{g_n}{u_n}
  =
  R^*(0)
  -
  R^*(-\eta\Theta_{N})
  +
  \sum_{n=1}^N \rbr{ 
  D_{R^*}(-\eta\Theta_{n},-\eta\hat{\Theta}_{n})
  -
  D_{R^*}(-\eta\Theta_{n-1},-\eta\hat{\Theta}_{n}) }.
  \]
  Therefore, we can bound $\eta \SLReg_N$ as:
%Then, we apply the fact that  and write $\eta\SLReg_N$ as 
%   conclude for all $u \in \Delta_{\Bcal}$,
  %   by plugging it into the definition of the regret , 
  %\min_{u \in \Delta(\Bcal)}
  %&=
%\eta \rbr{ \sum_{n=1}^N \left\langle \est_n, u_n\right\rangle -  \min_{u \in \Omega}\sum_{n=1}^N  \left\langle \est_n, u\right\rangle}\\
 \[
 \begin{aligned}
&\eta \SLReg_N 
=
\sum_{n=1}^N \eta \left\langle  \est_n, u_n \right\rangle 
- \min_{u \in \Omega}\sum_{n=1}^N  \left\langle  \eta\est_n,u \right\rangle \\
&=
- \min_{u \in \Omega} \left\langle \eta\Theta_N, u \right\rangle
+
R^*(0)
  -
  R^*(-\eta\Theta_{N})
  +
  \sum_{n=1}^N \rbr{ 
  D_{R^*}(-\eta\Theta_{n},-\eta\hat{\Theta}_{n})
  -
  D_{R^*}(-\eta\Theta_{n-1},-\eta\hat{\Theta}_{n}) }\\
  &=
   \max_{u \in \Omega} \left\langle -\eta\Theta_N, u \right\rangle 
  -
  R^*(-\eta\Theta_{N})
  +
    R^*(0)
    +
    \sum_{n=1}^N \rbr{
    D_{R^*}(-\eta\Theta_{n},-\eta\hat{\Theta}_{n})
  -
    D_{R^*}(-\eta\Theta_{n-1},-\eta\hat{\Theta}_{n})}
  \\
 & \leq   
  \sup_{u \in \Omega} R(u) - \inf_{u \in \Omega} R(u) +
    \sum_{n=1}^N(  
    D_{R^*}(-\eta\Theta_{n},-\eta\hat{\Theta}_{n})
  -
    D_{R^*}(-\eta\Theta_{n-1},-\eta\hat{\Theta}_{n})),
\end{aligned}
\]
where the last inequality is by applying the definition of Fenchel conjugate and item~\ref{item:one_to_one} of Fact\ref{fact:conj}: 
\begin{enumerate}
    \item
    $
    \max_{u \in \Omega} \left\langle -\eta\Theta_N, u \right\rangle 
  -
  R^*(-\eta\Theta_{N})
  \leq
  \sup_{u \in \Omega} \rbr{ \sup_{\tilde{\Theta} \in \RR^d} \left\langle -\tilde{\Theta}, u \right\rangle - R^*(\tilde{\Theta})}
  =
   \sup_{u \in \Omega} R^{**}(u) 
   =
   \sup_{u \in \Omega} R(u) .
    $
    %\chicheng{Todo for Yichen: think about how to reorganize / prove this}
    \item 
    $
     R^*(0) = \sup_{u \in \RR^d} \left\langle 0, u \right\rangle - R(u)
     = 
     \sup_{u \in \Omega} \left\langle 0, u \right\rangle - R(u)
     =
     - \inf_{u \in \Omega} R(u) .
    $
%     \yichen{$\max_{u \in \Omega} \left\langle -\eta\Theta_N, u \right\rangle $ may not exist? The problem is from a more fundamental level?}
%     \chicheng{$\max_{u \in \Omega} R(u)$ may not exist?}
    
\end{enumerate}
This concludes the proof of Equation~\eqref{eqn:oftrl-1}. We next prove Equation~\eqref{eqn:oftrl-2}. 

\paragraph{Upper bounding the divergence penalty terms $D_{R^*}(-\eta\Theta_{n},-\eta\hat{\Theta}_{n})$.} Since $R$ is closed and $\alpha$-strongly convex, by item~\ref{item:sc_strongly smooth_duality} of Fact~\ref{fact:conj},
% strongly convex and strongly smooth duality~\cite[Section 2.7]{shalev2011online} and \cite[Section 3.5]{zalinescu2002convex},
% \chicheng{Move these references to Proposition~\ref{prop:sc-conj} and quote that fact here?}
$R^*$ is $\frac{1}{\alpha}$-strongly smooth and $\forall \Theta, \Theta' \in \RR^d$, $
  D_{R^*}(\Theta,\Theta') 
  \leq
  \frac{1}{2\alpha}\|\Theta-\Theta'\|^2,
$ which implies $\forall n \in [N]$,  
\[
D_{R^*}(-\eta\Theta_{n},-\eta\hat{\Theta}_{n}) 
\leq
 \frac{1}{2\alpha}\|\eta\Theta_{n}-\eta\hat{\Theta}_{n}\|_*^2
 =
 \frac{\eta^2}{2\alpha} \|\est_n-\pred_n\|_*^2,
\]
where the last equality is from the definition of $\Theta_n = \est_n + \Theta_{n-1}$ and $\hat{\Theta}_n = \pred_n + \Theta_{n-1}$.

\paragraph{Lower bounding the prediction gain terms $D_{R^*}(-\eta\Theta_{n-1},-\eta\hat{\Theta}_{n})$.} Since $R$ is closed and $\alpha$-strongly convex, 
% by Fact \ref{fact:conj}, $R^*$ is closed, convex, and differentiable. 
% \chicheng{Check the fact I added. Also, $R$ being differentiable seems unnecessary.}
% By the $\frac{1}{\alpha}$-strongly smoothness of $R^*$ and 
by Lemma~\ref{lem: lowerbound for strongly smooth function}, we have 
% \edit{$\forall \Theta, \Theta' \in \RR^d$, $
%   D_{R^*}(\Theta,\Theta') 
%   \geq
%   \frac{\alpha}{2}  \|\nabla R^*(\Theta)-\nabla R^*(\Theta')\|^2,
% $
% which implies}{}
\[
D_{R^*}(-\eta\Theta_{n-1},-\eta\hat{\Theta}_{n})) 
\geq
\frac{\alpha}{2} \|\nabla R^*(-\eta\Theta_{n-1})-\nabla R^*(-\eta\hat{\Theta}_{n})\|^2.
\]
Finally, 
combining Equation~\eqref{eqn:oftrl-1} with the bounds on divergence penalty and prediction gain terms, Equation~\eqref{eqn:oftrl-2} is proved.
\end{proof}

\begin{remark}
An alternative proof of this theorem can be done using the Stronger Follow the Leader Lemma~\cite{cheng2019accelerating,mcmahan2017survey}.
\end{remark}

\subsection{Regularizer induced by example-based perturbations}
\label{sec:ftrl-ex-reg}

In this section, we instantiate the Optimistic FTRL algorithm and regret guarantee in the previous section with a specific $R$ that appears in our \algm, \algmp algorithms. Recall that in these algorithms, we use samples in the separator set $X$ and assign cost vector $\ell_x \sim \mathcal{N}(0,I_A)$ independently on each of them. 
As a result, we define $R$ in Equation~\eqref{eqn:r} as the Fenchel conjugate of $\Phi_{\mathcal{N}}(\Theta) =  \EE_{\ell \sim \mathcal{N}(0,I_{XA})} \sbr{ \mathop{\max}\limits_{u \in \Delta(\Bcal)} \left\langle  \Theta+q(\ell) , u\right\rangle }$; recall that
$q(\ell) = (\sum_{x\in \Xcal}\ell_x(h(x))_{h\in B}$ and $\ell = \rbr{\ell_x}_{x \in \Xcal}$.

First, we prove several useful properties of $\Phi_{\mathcal{N}}$.
% }

\begin{lemma}[Restatement of Lemma~\ref{lem: differentiable function}]
\label{lem: differentiable function_restated}
$\Phi_\mathcal{N}(\Theta)$ is differentiable for any $\Theta \in \mathbb{R}^B$ and $\nabla\Phi_{\mathcal{N}}(\Theta) = \EE_{\ell \sim \mathcal{N}(0,I_{XA})}\sbr{\mathop{\argmax}\limits_{u\in \Delta(\Bcal)} \left\langle  \Theta +   q(\ell) , u\right\rangle}$.
% \edit{, where $q(\ell) = (\sum_{x\in \Xcal}\ell_x(h(x)))_{h\in B}$}{}
\end{lemma}
% \chicheng{The macro $\argmax$ is defined. Use it instead throughout?}

\begin{proof}
To prove the lemma, by~\cite{bertsekas1973stochastic}[Propositions 2.2, 2.3], it suffices to prove that for any $\Theta$,  $\argmax_{u \in \Delta(\Bcal)} \left\langle \Theta +   q , u\right\rangle$ is unique with probability $1$, over the draw of $\ell\sim \mathcal{N}(0,I_{XA})$.

By the definition of $\argmax$, it can be seen that the solution of $\mathop{\argmax}\limits_{u \in \Delta(\Bcal)} \left\langle \Theta +   q ,u\right\rangle$ is not unique if and only if there exist $h,h' \in \Bcal$ and $h \neq h'$,
such that $\Theta\coord{h}+q\coord{h} = \Theta\coord{h'}+q\coord{h'} = \mathop{\max}\limits_{h}(\Theta\coord{h}+q\coord{h})$. 

Define event 
% \edit{ $E$:  the solution of $\mathop{\argmax}\limits_{u \in \Delta(\Bcal)} \left\langle \Theta +   q ,u\right\rangle$ is not unique}{
$E = \cbr{ \argmax_{u \in \Delta(\Bcal)} \left\langle \Theta +   q ,u\right\rangle \text{is not unique} }$ and event $E_{hh'}$:  $\Theta\coord{h}+ q\coord{h} = \Theta\coord{h'}+q\coord{h'} $, for all $h,h' \in \Bcal$ and $h \neq h'$.  
% \edit{Notice the necessary condition for event $E$ to happen is the existence of $h,h' \in \Bcal$ and $ h\neq h'$ such that $\tilde{\Theta}\coord{h} = \tilde{\Theta}\coord{h'} $, which corresponds to event $E_{hh'}$. By this observation, we have }{
For $E$ to happen, it is necessary that one of $E_{h,h'}$ happens. Formally,
$E \subseteq \mathop{\cup}\limits_{h\neq h'} E_{hh'}$. By applying the union bound, we obtain
\[
\Prob(E) 
\leq \Prob(\mathop{\cup}\limits_{h\neq h'}  E_{hh'}) 
\leq  \sum_{h \neq h'} \Prob(E_{hh'})  .
\]

We will now show that for any  $ h, h' \in \Bcal$ and $ h\neq h'$, $\Pr(E_{hh'}) = 0$. By its definition
, $E_{hh'}$ happens if and only if $\Theta\coord{h} +   q\coord{h} = \Theta\coord{h'} +   q\coord{h'}$. We can rearrange the terms and get
\[
q\coord{h}-q\coord{h'} 
=
\eta \cdot \rbr{\Theta\coord{h}-\Theta\coord{h'}} .
\]
The following proof shows given any constant $C \in \RR$,  $q\coord{h}-q\coord{h'} =C $ happens with probability 0. 
% For $q\coord{h}-q\coord{h'} $, 
Here we first recall the definition of separator set $\Xcal$: $\exists x \in \Xcal$ s.t. $ h(x) \neq h'(x)$ for any $h \neq h'$ and define $\Xcal_{h,h'} := \cbr{x\in\Xcal |h(x) \neq h'(x)}$. It can be seen that $\Xcal_{h,h'}$ is nonempty for $ h\neq h'$ by the definition of $\Xcal$.
Now we can rewrite $q\coord{h}-q\coord{h'}$ with the help of $\Xcal_{h,h'}$, 
\[
q\coord{h}-q\coord{h'}  =  \sum_{x \in \Xcal} \ell_{x}(h(x)) - \ell_{x}(h'(x)) = \sum_{x \in \Xcal_{h,h'} } \ell_{x}(h(x)) - \ell_{x}(h'(x)).
\]
Since $\ell = \rbr{\ell_x}_{x \in \Xcal} \sim  \mathcal{N}(0,I_{XA})$, $q\coord{h}-q\coord{h'}$ can be viewed as a sum of $2|\Xcal_{h,h'}|$ independent Gaussian variables following distribution $\Ncal(0,1)$. By this observation, we have that $q\coord{h}-q\coord{h'} \sim \Ncal(0,4|\Xcal_{h,h'}|^2)$, which implies 
$\forall C \in \mathbb{R}$, $\Prob(q\coord{h}-q\coord{h'} = C) = 0$.
% Then, by the property of Gaussian variables, $\forall C \in \mathbb{R}$, $\Prob(q\coord{h}-q\coord{h'} = C) = 0$. 
This in turn shows that $\Pr(E) \leq \sum_{h \neq h'} \Prob(E_{hh'}) = 0$, which concludes the proof of the lemma.
\end{proof}

  \begin{lemma}
   \label{lem: closed and convex function}
   $\Phi_{\mathcal{N}}$ is closed and convex on  $\mathbb{R}^B$. 
  \end{lemma}
  \begin{proof}
  To begin with, we show  $\Phi_{\mathcal{N}}$ is convex.
  Recall that $\Phi_{\mathcal{N}}(\Theta) =  \EE_{\ell \sim \mathcal{N}(0,I_{XA})}  \sbr{
  \mathop{\max}\limits_{u \in \Delta(\Bcal)} 
  \left\langle \Theta +   q(\ell) , u\right\rangle }$ where $q(\ell) = (\sum_{x\in \Xcal}\ell_x(h(x)))_{h\in B}$. To check the convexity of  $\Phi_{\mathcal{N}}$, given any $\Theta,\Theta'\in \mathbb{R}^B$ and any $\const \in [0,1]$, we have
 \[
 \begin{aligned}
 \Phi_{\mathcal{N}}(\const\Theta + (1-\const)\Theta')
 =&
 \EE_{\ell \sim \mathcal{N}(0,I_{XA})}
  \mathop{\max}\limits_{u \in \Delta(\Bcal)} 
  \left\langle  \const\Theta + (1-\const)\Theta' +   q(\ell) ,u\right\rangle \\
  =&
   \EE_{\ell \sim \mathcal{N}(0,I_{XA})}
  \mathop{\max}\limits_{u \in \Delta(\Bcal)} 
  \left\langle \const(\Theta+   q(\ell) ) + (1-\const)(\Theta' +   q(\ell)) ,u\right\rangle \\
  \leq &
   \EE_{\ell \sim \mathcal{N}(0,I_{XA})}
   ( \mathop{\max}\limits_{u \in \Delta(\Bcal)}\left\langle  \const(\Theta+   q(\ell) ) ,u\right\rangle 
  +
   \mathop{\max}\limits_{u \in \Delta(\Bcal)}\left\langle  (1-\const)(\Theta' +   q(\ell)) ,u\right\rangle \\
   = & 
    \const\EE_{\ell \sim \mathcal{N}(0,I_{XA})}
  \mathop{\max}\limits_{u \in \Delta(\Bcal)} 
  \left\langle  \Theta +   q(\ell) ,u\right\rangle 
  +
    (1-\const)\EE_{\ell \sim \mathcal{N}(0,I_{XA})} 
  \mathop{\max}\limits_{u \in \Delta(\Bcal)} 
  \left\langle \Theta' +   q(\ell).u\right\rangle \\
  =&
   \const\Phi_{\mathcal{N}}(\Theta )
   +
    (1-\const)\Phi_{\mathcal{N}}( \Theta') .
 \end{aligned}
 \]
 The inequality is by the fact that $\forall u \in \Delta(\Bcal)$, $\forall a,b \in \mathbb{R}^B$ , $ \left\langle a+b, u \right\rangle \leq 
 \mathop{\max}\limits_{u \in \Delta(\Bcal)} \left\langle a, u\right\rangle  
 +   
 \mathop{\max}\limits_{u \in \Delta(\Bcal)}\left\langle  b, u \right\rangle$. 
%  \edit{By this we show  $\Phi_{\mathcal{N}}$ is convex.}{}
 
 Secondly, to show $\Phi_{\mathcal{N}}$ is closed, by \cite{boyd2004convex}[Section A.3.3], since $\mathbb{R}^B$ is closed, it suffice to show $\Phi_{\mathcal{N}}$ continuous on $\mathbb{R}^B$.  Since $\Phi_{\mathcal{N}}$ is differentiable on $\mathbb{R}^B$ by Lemma~\ref{lem: differentiable function}, we have that $\Phi_{\mathcal{N}}$ is continuous, which concludes that $\Phi_{\mathcal{N}}$ is closed. 
\end{proof}
The following two properties of $\Phi_\Ncal$ are largely inspired by~\cite{syrgkanis2016efficient}.
   \begin{lemma}
   \label{lem: strongly smooth zero potential upperbound}
 $\Phi_{\mathcal{N}}(0) \leq \sqrt{2X\ln(B)}$.
  \end{lemma}
  
  \begin{proof}
  By the definition of $\Phi_{\mathcal{N}}$, we have  $
  \Phi_{\mathcal{N}}(0) 
  =
  \EE_{\ell \sim \mathcal{N}(0,I_{XA})} \sbr{ \mathop{\max} \limits_{u \in \Delta(\Bcal)} \inner{q(\ell)}{u}}$, where 
  $q(\ell) = (\sum_{x\in \Xcal}\ell_x(h(x)))_{h\in B}$.
  For the remainder of the proof, we use $\EE$ as an abbreviation for $\EE_{\ell \sim \mathcal{N}(0,I_{XA})}$.
  Recall that $q\coord{h} = \sum_{x\in \Xcal}\ell_x(h(x))$, we can write   $
  \Phi_{\mathcal{N}}(0) 
  =
  \EE \sbr{
  \mathop{\max}\limits_{h \in \mathcal{B}} q\coord{h} } $. 
%   each $\ell_{x}(h(x))$ within  $\{\ell_x(h(x))\}_{x \in \Xcal}$  is iid distributed Gaussian variable following $\mathcal{N}(0,1)$.
%   $(\ell_{x_1}(h(x)), \ell_{x_2},..., \ell_{x_{X}})$  are independent 
%   By utilizing this observation, we have
  For any $b > 0$,
  \[
  \begin{aligned}
 \exp\rbr{ b \cdot \EE \sbr{  \mathop{\max}\limits_{h \in \mathcal{B}} q\coord{h} } }
 \leq 
 \EE\sbr{ \exp( b  \cdot \mathop{\max}\limits_{h \in \mathcal{B}} q\coord{h}) }
 = 
 \EE \sbr{ \mathop{\max}\limits_{h \in \mathcal{B}}\exp\rbr{ b \cdot  q\coord{h} } }
 \leq   
  \sum_{h \in \mathcal{B}}\EE \sbr{ \exp\rbr{ b \cdot q\coord{h} } },
 \end{aligned}
  \]
  where the first inequality is  
%   \chicheng{in general, use ``is from'' or ``follows from'' in formal math writing}
  from the convexity of exponential function, while the last inequality is form  $\mathop{\max}\limits_{h \in \mathcal{B}}\exp( b   q\coord{h}) \leq
  \sum_{h \in \mathcal{B}}\exp( b   q\coord{h})$. 
  
  By the property of the sum of independent Gaussian variables, since $\ell \sim \mathcal{N}(0, I_{XA})$, we have that for any $ h \in \Bcal$, $q\coord{h} = (\sum_{x\in \Xcal}\ell_x(h(x)))_{h\in B}$ follows Gaussian distribution $\mathcal{N}(0,X)$. 
  Then, $\forall h \in \mathcal{B}$, by a standard fact on the moment generating function of Gaussian random variables, we have that $\forall h \in \Bcal$,
  \[
  \EE \exp( b \cdot  q\coord{h}) = \exp\rbr{ \frac{b^2X}{2} },
  \]
which implies
  \[
  \exp\rbr{ b \cdot \EE \mathop{\max}\limits_{h \in \mathcal{B}} q\coord{h} }
 \leq 
 B\exp\rbr{ \frac{b^2X}{2} }.
  \]
  Hence, by taking the natural logarithm and dividing by $b > 0$ on both sides, we get 
  \begin{equation}
    \EE \sbr{ \mathop{\max}\limits_{h \in \mathcal{B}} q\coord{h} }
  \leq
  \frac{\ln(B)}{b} + \frac{b X}{2}.
  \label{eq:gaussian_moment}
  \end{equation}
Since Equation~\eqref{eq:gaussian_moment} holds for any $b > 0$, by choosing $b = \sqrt{\frac{2\ln(B)}{X}}$, we obtain $ \EE\sbr{  \mathop{\max}\limits_{h \in \mathcal{B}} q\coord{h} }\leq  \sqrt{2X\ln(B)}$, which concludes $\Phi_{\mathcal{N}}(0) \leq \sqrt{2X\ln(B)}$.
\end{proof}
%_{\ell \sim \mathcal{N}(0,I_{XA})} 
   
  \begin{lemma}
  \label{lem: beta strongly smooth of Gaussian perturbation}
  $\Phi_\mathcal{N}$ is $\beta$-strongly smooth with respect to $\| \cdot \|_\infty$
%   \chicheng{There is a small difference between these - I recommend doing a search-replace over the paper.}
  with $\beta = \sqrt{\frac{8}{\pi}} AX$.
  \end{lemma}
  
  \begin{proof}
  To prove that $\Phi_\mathcal{N}$ is $\beta$-strongly smooth with respect to $\| \cdot \|_\infty$, by Definition 4.18 in \cite{orabona2019modern}, it suffices to show that $\forall \Theta, \Theta' \in \mathbb{R}^B$,
  \[
  \|\nabla\Phi_\mathcal{N}(\Theta)  -  \nabla\Phi_\mathcal{N}(\Theta')\|_1
  \leq
  \beta\|\Theta - \Theta'\|_\infty.
  \]
  By Lemma~\ref{lem: differentiable function}, $\nabla\Phi_\mathcal{N}(\Theta) =  \EE_{\ell \sim \mathcal{N}(0,I_{XA})} \sbr{ \mathop{\argmax}\limits_{u \in \Delta(\Bcal)} \left\langle \Theta +  q(\ell) , u\right\rangle }$, where $q(\ell) = (\sum_{x\in \Xcal}\ell_x(h(x)))_{h\in B}$. Given $\ell$, we introduce shorthands
  \[
 u_q = \mathop{\argmax}\limits_{u \in \Delta(\Bcal)} \left\langle  u ,\Theta +  q(\ell) \right\rangle,u'_q = \mathop{\argmax}\limits_{u \in \Delta(\Bcal)} \left\langle  u ,\Theta' +  q(\ell) \right\rangle.
  \]
  We also introduce the short 
  hand $h_{q}$ and $h'_{q}$  to represent the policy in class $\mathcal{B}$ selected by  $u_q $ and $u'_q$. More explicitly, 
 \[
 h_{q} = \argmax_{h \in B} \Theta[h] +  q[h],
 \;\;
 h'_{q} = \argmax_{h \in B} \Theta'[h] +  q[h].
 \]
  By Lemma~\ref{lem: differentiable function}, 
 $u_q $, $u'_q$, $h_q$, $h_q'$ are well-defined with probability $1$ over the randomness of $\ell$.
  With this notation, we have  $\|u_q -u'_q\|_1 = 0$ when $h_{q}  = h'_{q}$ , and  $\|u_q -u'_q\|_1 = 2$ when $h_{q}  \neq h'_{q}$, which means $\|u_q -u'_q\|_1 = 2 I(h_{q}  \neq h'_{q})$. From now on, we use  $P(\ell)$ to denote the probability density function of $\ell$.
%   , and $\dif \ell$ to denote the differential one-form for $\ell$. 
%   \chicheng{Is this standard notation? Do we need to introduce this?}
  By this, we have
  \[
  \begin{aligned}
   \|\nabla\Phi_\mathcal{N}(\Theta)  -  \nabla\Phi_\mathcal{N}(\Theta')\|_1
  =&
  \|\EE_{\ell \sim \mathcal{N}(0,I_{XA})}u_q  - \EE_{\ell \sim \mathcal{N}(0,I_{XA})}u'_q  \|_1\\
%   =&
%   \| \int_\ell u_q P(\ell) \dif \ell
%   -
%  \int_\ell u'_q P(\ell) \dif \ell \|_1\\
%  =&
%   \| \int_\ell  (u_q-u'_q) P(\ell) \dif \ell \|_1\\
  =&
  \sum_{h \in \mathcal{B}} \abr{ \int_\ell (I(h =h_{q} )- I(h =h'_{q} )) P(\ell) \dif \ell }\\
\leq &
    \sum_{h \in \mathcal{B}}  \int_\ell \abr{ (I(h =h_{q} )- I(h =h'_{q} )) } P(\ell) \dif \ell \\
= &
    \int_\ell \sum_{h \in \mathcal{B}} \abr{  (I(h =h_{q} )- I(h =h'_{q} )) } P(\ell) \dif \ell\\
%  =&
%   \int_\ell \|u_q -u'_q\|_1  P(\ell) \dif \ell\\
 = & 
    \int_\ell 2I(h_{q} \neq h'_{q}) P(\ell) \dif \ell\\
  =&
   2\Prob(h_{q} \neq h'_{q}), 
  \end{aligned}
  \]
%   \yichen{I believe some the middle three lines are unnecessary and we can directly get $  \int_\ell \|u_q -u'_q\|_1  P(\ell) \dif \ell$}
% \chicheng{Yichen, as discussed last time, the second line is integration over vector valued functions, which is a bit awkward. Do you know how to do this integration? Or, do you know how to avoid writing this and still derive the results below?}
where  $ \Prob(h_{q} \neq h'_{q})$ denotes the probability of   $h_{q} \neq h'_{q}$ under the distribution of $\ell$. By the definition of separator set $\Xcal$, $ h \neq h'$ if and only if $\exists x \in \Xcal$ s.t. $ h(x) \neq h'(x)$. Then, by bringing in $h_{q},h'_{q}$ and apply the union bound, we have 
\[
 \Prob(h_{q} \neq h'_{q})
 \leq 
 \sum_{x \in \Xcal}   \Prob(h_{q}(x) \neq h'_{q}(x))
 =
 \sum_{x \in \Xcal} \sum_{a \in \Acal}   \Prob(a = h_{q}(x) \neq h'_{q}(x)).
\]
Then, given any $x$ and $a$, we denote $\ell_{-xa}$ as all other Gaussian variables in set  $\{\ell_x(a)\}_{x\in X, a \in \Acal}$ except $\ell_x(a)$ and $\Prob(a = h_{q}(x) \neq h'_{q}(x))$ as the probability of $a = h_{q}(x) \neq h'_{q}(x)$ under the distribution of $\ell_{-xa}$. Then, $\forall x \in X, a \in \Acal$, 
% \chicheng{Need to define $P$ here formally}
\[
\begin{aligned}
\Prob(a = h_{q}(x) \neq h'_{q}(x)) 
= 
\int_{\ell_{-xa}} \Prob(a = h_{q}(x) \neq h'_{q}(x) | \ell_{-xa}) P(\ell_{-xa}) d(\ell_{-xa}).
\end{aligned}
\]
% \chicheng{In the main text, we use $\ell_x(a)$ as opposed to $\ell_x[a]$. Perhaps we should stick to the former and be consistent throughout?}
Conditioned on $\ell_{-xa}$, we denote $\tilde{\ell} = \{\tilde{\ell}_x(a)\}_{x\in X, a \in \Acal}$  as the corresponding perturbation vector that share the same value with $\ell$ on all other entries and set $\tilde{\ell}_x(a) = 0$. Define 
\[ 
\Theta_{xa} = \Theta + \rbr{ \sum_{x\in \Xcal}\tilde{\ell}_x(h(x)) }_{h\in \Bcal}, 
\;\; 
\Theta'_{xa} = \Theta' + \rbr{ \sum_{x\in \Xcal}\tilde{\ell}_x(h(x)) }_{h\in \Bcal}.
\]
By algebra, $\Theta_{xa}-\Theta'_{xa} = \Theta -\Theta'$. By the definition of $u_q $ and $u'_q$, with the new notation, 
% \chicheng{I think ``notion'' is more high-level than ``notation'' - see e.g. \url{https://wikidiff.com/notation/notion}}
we can rewrite
\[u_q = \mathop{\argmax}\limits_{u \in \Delta(\Bcal)} \left\langle  u ,\Theta_{xa} +   ( I(h(x)=a)\ell_x(a))_{h \in \Bcal} \right\rangle,
\;\; 
u'(q) = \mathop{\argmax}\limits_{u \in \Delta(\Bcal)} \left\langle  u ,\Theta'_{xa} +   ( I(h(x)=a)\ell_x(a))_{h \in \Bcal} \right\rangle.
\]
By using $(\Theta\coord{h})_{h\in \Bcal} := \Theta$, we can write  
\[
h_{q} =\mathop{\argmax}\limits_{h \in \Bcal } \Theta_{xa}\coord{h} +  I(h(x)=a)\ell_x(a),
\;\;
h'_{q} =\mathop{\argmax}\limits_{h \in \Bcal } \Theta'_{xa}\coord{h} +  I(h(x)=a)\ell_x(a).
\]
Then, by dividing the set $\Bcal$ into disjoint subsets $\Bcal_{xa} = \{h|h(x)=a,h\in \Bcal\}$ and $\Bcal \setminus \Bcal_{xa}$. If $\Bcal_{xa} = \varnothing$ or $\Bcal \setminus \Bcal_{xa} = \varnothing$ , we have that $\Prob(a = h_{q}(x) \neq h'_{q}(x))=0 $. Otherwise, we can view $h_{q}$ as the $h$ that corresponds to $\mathop{\max} \{
\mathop{\max} \limits_{h \in \Bcal_{xa}}\Theta_{xa}\coord{h} +  \ell_x(a),
\mathop{\max} \limits_{h \in \Bcal \setminus \Bcal_{xa}}\Theta_{xa}\coord{h} 
\}$, and $h'_{q}$ as the $h$ that corresponds to $\mathop{\max} \{\mathop{\max} \limits_{h \in \Bcal_{xa}}\Theta'_{xa}\coord{h} +  \ell_x(a),
\mathop{\max} \limits_{h \in \Bcal \setminus \Bcal_{xa}}\Theta'_{xa}\coord{h} 
\}$.  
With this insight, it can be seen that 
 \[
\left\{
\begin{aligned}
\delta := & \mathop{\max} \limits_{h \in \Bcal \setminus \Bcal_{xa}}\Theta_{xa}\coord{h} - \mathop{\max} \limits_{h \in \Bcal_{xa}}\Theta_{xa}\coord{h} >   \ell_x(a) \rightarrow a \neq h_{q}(x), \\
\delta' :=& \mathop{\max} \limits_{h \in \Bcal \setminus \Bcal_{xa}}\Theta_{xa}'\coord{h} - \mathop{\max} \limits_{h \in \Bcal_{xa}}\Theta_{xa}' \coord{h} <   \ell_x(a) \rightarrow a = h'_{q}(x).
\end{aligned}
\right. 
\]

Therefore, 
both $\delta \leq \ell_x(a) $ and $\ell_x(a) \leq \delta'$ are necessary for $a = h_{q}(x) \neq h_{q}'(x)$ to happen, which implies 
\[
\Prob(a = h_{q}(x) \neq h'_{q}(x) | \ell_{-xa})  \leq \Prob(\delta \leq \ell_x(a) \leq \delta').
\]
If $\delta' < \delta$, then 
$\Prob(a = h_{q}(x) \neq h'_{q}(x) | \ell_{-xa}) = 0$. Otherwise ($\delta \leq \delta'$), 
conditioned on $\ell_{-xa}$,  by $\ell_x(a) \sim \Ncal(0,1)$, we have
\[
\begin{aligned}
\Prob(a = h_{q}(x) \neq h'_{q}(x) | \ell_{-xa}, \delta' \geq \delta) 
\leq
\int_{\delta}^{\delta' }  
\frac{1}{\sqrt{2\pi}}\exp(-\frac{\ell_x(a)^2}{2}) \dif (\ell_x(a))
\leq 
% \frac{1}{\sqrt{2\pi}}\int_{\delta}^{ \delta'}   \dif(\ell_x(a))\\
% = &
\frac{1}{\sqrt{2\pi} } (\delta' - \delta),
%\leq &
%\frac{1}{\sqrt{2\pi} } | \delta - \delta' | \\
%  \leq 
%  \sqrt{\frac{2}{\pi}} \|\Theta - \Theta'\|_\infty,
\end{aligned}
\]
% \edit{
% Finally, by the fact that 
% }
% {
 in which case,
\[
\begin{aligned}
\delta' - \delta 
\leq& 
\abr{ \mathop{\max} \limits_{h \in \Bcal \setminus \Bcal_{xa}}\Theta_{xa}\coord{h} 
-
\mathop{\max} \limits_{h \in \Bcal \setminus \Bcal_{xa}}\Theta’_{xa}\coord{h} }
+
\abr{ \mathop{\max} \limits_{h \in \Bcal_{xa}}\Theta_{xa}\coord{h} 
-
\mathop{\max} \limits_{h \in \Bcal_{xa}}\Theta'_{xa}\coord{h} }\\
\leq& 
\abr{ \mathop{\max} \limits_{h \in \Bcal}\Theta_{xa}\coord{h} 
-
\mathop{\max} \limits_{h \in \Bcal}\Theta'_{xa}\coord{h}} 
+
\abr{ \mathop{\max} \limits_{h \in \Bcal}\Theta_{xa}\coord{h} 
-
\mathop{\max} \limits_{h \in \Bcal}\Theta'_{xa}\coord{h} }\\
\leq&
2\abr{ \mathop{\max} \limits_{h \in \Bcal} (\Theta_{xa}\coord{h} -\Theta'_{xa}\coord{h}) }\\
\leq&
2\mathop{\max} \limits_{h \in \Bcal} \abr{ \Theta_{xa}\coord{h} -\Theta'_{xa}\coord{h} }\\
=&
2 \|\Theta_{xa} - \Theta'_{xa}\|_\infty 
=
2 \|\Theta - \Theta'\|_\infty .
\end{aligned}
\]
% \edit{we bring in Gaussian distribution and conclude $\forall x \in X, a \in \Acal$, $\forall \ell_{-xa} $,}{
By this we conclude $\forall x \in \Xcal$, $\forall a \in \Acal$, $\Prob(a = h_{q}(x) \neq h'_{q}(x) | \ell_{-xa}) \leq \sqrt{\frac{2}{\pi}} \|\Theta - \Theta'\|_\infty$, and
  \[
  \begin{aligned}
    \|\nabla\Phi_\mathcal{N}(\Theta)  -  \nabla\Phi_\mathcal{N}(\Theta')\|_1
%     \leq
%   2 \Prob(h_{q} \neq h'_{q})
  \leq
  2\sum_{x \in \Xcal} \sum_{a \in \Acal}    \Prob(a = h_{q}(x) \neq h'_{q}(x)) 
  \leq
  \sqrt{\frac{8}{\pi}} AX\|\Theta - \Theta'\|_\infty.
  \qedhere
  \end{aligned}
  \]
% This concludes that  $\Phi_D$ is $\beta$-strongly smooth with $\beta =  \sqrt{\frac{8}{\pi}} AX$.
% \chicheng{Check empty spaces before the period punctuation and remove unnecessary ones}
\end{proof}

% Secondly, we provide proofs for the properties of 
% \edit{$R_\Ncal: \Delta(\Bcal) \rightarrow \RR$.}{
Secondly, we prove useful properties of $R_\Ncal: \RR^B \to \RR \cup \cbr{+\infty}$, where $R_\Ncal = \Phi_\Ncal^*$.
%\chicheng{We should now let $R_\Ncal$ to have the type $R_\Ncal: \RR^d \to \bar{\RR}$
%}
\begin{lemma}[Restatement of Lemma~\ref{lem:sc_dual}]
\label{lem:sc_dual_restated}
$R_{\Ncal}(u) = \Phi_\mathcal{N}^*(u)$ is closed and $\sqrt{\frac{\pi}{8}}{\frac{1}{ AX}}$-strongly convex with respect to $\| \cdot \|_1$.
\end{lemma}
\begin{proof}
Since $\Phi_\mathcal{N}$ is closed and convex by Lemma~\ref{lem: closed and convex function}, by item~\ref{item:closed_convex} of Fact~\ref{fact:conj}, $R_{\Ncal}(u) = \Phi_\mathcal{N}^*(u)$ is closed and convex. 
% \red{should we mention proper function to make it more concrete?} 
% \chicheng{where is ``proper'' used?}
Also, by Lemma~\ref{lem: beta strongly smooth of Gaussian perturbation}, $\Phi_\mathcal{N}$ is $\sqrt{\frac{8}{\pi}} AX$-strongly smooth with respect to $\|\cdot\|_\infty$. Then, we apply item~\ref{item:sc_strongly smooth_duality} of Fact~\ref{fact:conj} and conclude $R_{\Ncal}(u)$ is $\sqrt{\frac{\pi}{8}}{\frac{1}{ AX}}$-strongly convex with respect to $\| \cdot \|_1$. 
\end{proof}

\begin{lemma}
\label{lem:upperbound_of_R}
$\forall u \in \Delta(\Bcal)$, $R_\Ncal(u) \leq 0$.
\end{lemma}
\begin{proof}
Recall that $R_{\Ncal}(u) = \sup_{\Theta \in \RR^B} \inner{\Theta}{u}- \Phi_{\Ncal}(\Theta)$, where $\Phi_\mathcal{N}(\Theta) = \EE_{\ell \sim \mathcal{N}(0,I_{XA})} \sbr{ \mathop{\max} \limits_{v \in \Delta(\Bcal)} \inner{\Theta+q(\ell)}{v}}$, it suffices to show $\forall u \in \Delta(\Bcal)$,  $\forall \Theta\in \RR^B$,  $\inner{\Theta}{u} - \Phi_{\Ncal}(\Theta) \leq 0$. 
For the remainder of the proof, we use $\EE$ as an abbreviation for $\EE_{\ell \sim \mathcal{N}(0,I_{XA})}$.
Then, $\forall u \in \Delta(\Bcal)$,  $\forall \Theta\in \RR^B$, 
  \[
  \begin{aligned}
  \Phi_{\mathcal{N}}(\Theta) 
  =
  \EE \sbr{ \mathop{\max} \limits_{v \in \Delta(\Bcal)} \inner{\Theta+q(\ell)}{v}}
  \geq 
  \mathop{\max}\limits_{v \in \Delta(\Bcal)} 
  \EE
  \sbr{\left\langle \Theta + q(\ell), v \right\rangle }
  =
  \mathop{\max}\limits_{v \in \Delta(\Bcal)} \left\langle  \Theta  , v \right\rangle
  \geq \left\langle  \Theta  , u\right\rangle,
  \end{aligned}
  \]
where the first inequality is from Jensen's inequality, while the second equality is by the fact that 
\[
\EE \sbr{ q(\ell) }  
=
\EE \sbr{ (\sum_{x\in \Xcal} \ell_x(h(x)))_{h\in 
\Bcal} }
=
\sum_{x\in \Xcal} \EE \sbr{ \ell_x(h(x))_{h\in \Bcal} }
=  0.
\qedhere
\]
\end{proof}
% \chicheng{Yichen, I added this lemma. (1) Do you think it is necessary? (2) If yes, do you think we can prove it?}
\begin{lemma}
\label{lem:r_domain}
$\dom( R_\Ncal) = \Delta(\Bcal)$. 
\end{lemma}
\begin{proof}
We show the lemma in two steps. First, by Lemma~\ref{lem:upperbound_of_R}, $\forall u \in \Delta(\Bcal)$, $R_\Ncal(u) \leq 0$, which is finite. Secondly we show $\forall u \in \RR^B \setminus \Delta(\Bcal)$, $R_\Ncal(u) = +\infty$.

For the second step, $\forall u \in \RR^B \setminus \Delta(\Bcal)$, by using $\EE$ as an abbreviation for $\EE_{\ell \sim \mathcal{N}(0,I_{XA})}$, we have
\[
\begin{aligned}
R_\Ncal(u) 
= &
\sup_{\Theta \in \RR^B} \inner{\Theta}{u}- \Phi_{\Ncal}(\Theta)\\
= &
\sup_{\Theta \in \RR^B} \inner{\Theta}{u} - \EE \sbr{ \mathop{\max} \limits_{v \in \Delta(\Bcal)} \inner{\Theta+q(\ell)}{v}}\\
\geq&
% \sup_{\Theta \in \RR^B} \inner{\Theta}{u} - \EE \sbr{ \mathop{\max} \limits_{v \in \Delta(\Bcal)} \inner{\Theta}{v}}-  \EE \sbr{ \mathop{\max} \limits_{v \in \Delta(\Bcal)} \inner{q(\ell)}{v}}\\
% =&
\sup_{\Theta \in \RR^B} \rbr{\inner{\Theta}{u} - \mathop{\max}\limits_{v \in \Delta(\Bcal)} \inner{\Theta}{v} }- \EE \sbr{ \mathop{\max} \limits_{v' \in \Delta(\Bcal)} \inner{q(\ell)}{v'}} ,\\
\end{aligned}
\]
where the first inequality is by the convexity of $\max$ and inner product functions. 

By Lemma~\ref{lem: strongly smooth zero potential upperbound}, $\Phi_\Ncal(0)=\EE \sbr{ \mathop{\max} \limits_{v' \in \Delta(\Bcal)} \inner{q(\ell)}{v'}} \leq \sqrt{2X\ln(B)}$, which is a constant. Now, it suffices to show that $\forall u \in \RR^B \setminus \Delta(\Bcal)$, $\sup_{\Theta \in \RR^B} \rbr{\inner{\Theta}{u} - \mathop{\max}\limits_{v \in \Delta(\Bcal)} \inner{\Theta}{v} } = +\infty$.

We divide $\RR^B \setminus \Delta(\Bcal)$ into two disjoint sets:
\[
\left\{
\begin{aligned}
U^- :=& \{u \in \RR^B  \: |  \: u[i]<0 \: \text{for some} \: i\in \{1,2,\cdots,B\}\}, \\
U^+ :=& \{u \in \RR^B  \: |  \: \|u\|_1>1, u \succeq 0\} ,
\end{aligned}
\right. 
\]
where it can be verified that $\RR \setminus \Delta(\Bcal) = U^- \cup U^+$.

For any $u \in U^-$, where $u[i]<0$ for some $i\in \{1,2,\cdots,B\}$, we have that $\forall C \in \RR$, by setting $\theta(u,i,C) = \frac{|C|+1}{u[i]}\Onehot(i,\Bcal) \in \RR^B$, 
\[
\begin{aligned}
\mathop{\max}\limits_{v \in \Delta(\Bcal)} \inner{\theta(u,i,C)}{u-v}
=&
\inner{\frac{|C|+1}{u[i]}\Onehot(i,\Bcal)}{u} - \mathop{\max}\limits_{v \in \Delta(\Bcal)} \inner{\frac{|C|+1}{u[i]}\Onehot(i,\Bcal)}{v}\\
\geq &
\frac{|C|+1}{u[i]} \cdot u[i] = |C|+1 > C,
\end{aligned}
\]
where the first inequality is by $\inner{\frac{|C|+1}{u[i]}\Onehot(i,\Bcal)}{v} \leq 0$, $\forall u[i] < 0$, $\forall v \in \Delta(\Bcal)$. This implies $\forall u \in U^-$, $\sup_{\Theta \in \RR^B} \rbr{\mathop{\max}\limits_{v \in \Delta(\Bcal)} \inner{\Theta}{u-v}} = +\infty$. Thus, $\forall u \in U^-$, $R_\Ncal(u)= +\infty$.

Similarly, for any $u \in U^+$, we have that $\forall C \in \RR$, by setting $\theta(u,C) = \frac{|C|+1}{\|u\|_1 - 1} \cdot u \in \RR^B$,
\[
\begin{aligned}
\mathop{\max}\limits_{v \in \Delta(\Bcal)} \inner{\theta(u,C)}{u-v}
=&
\mathop{\max}\limits_{v \in \Delta(\Bcal)} \inner{\frac{|C|+1}{\|u\|_1 - 1} \cdot u}{u-v}\\
=&
\inner{\frac{|C|+1}{\|u\|_1 - 1} \cdot u}{u-\frac{u}{\|u\|_1}}\\
=&
\frac{|C|+1}{\|u\|_1 - 1} \cdot \|u\|_1\cdot (\|u\|_1 - 1) > C,
\end{aligned}
\]
which is by basic algebra. This implies $\forall u \in U^+$, $R_\Ncal(u) = +\infty$.

In conclusion, we have that $\forall u \in \RR^B \setminus \Delta(\Bcal) = U^- \cup U^+$, $R_\Ncal(u) = +\infty$,
% while $\forall u \in \RR^B \setminus \Delta(\Bcal)$ and $\forall v \in \Delta(\Bcal)$, $u \neq v$. This implies,
% \[
% R_\Ncal(u) 
% \geq
% \sup_{\Theta \in \RR^B} \rbr{\mathop{\max}\limits_{v \in \Delta(\Bcal)} \inner{\Theta}{u-v} - \sqrt{2X\ln(B)}}
% = 
% +\infty,
% \] 
which concludes the the proof. 
\end{proof}

% \eta \rbr{ \sup_{\tilde{\Theta}} \inner{u}{\tilde{\Theta}} - \Phi_\mathcal{N}(\tilde{\Theta})}.

%\chicheng{We want to verify that $R_\Ncal$ has domain $\Delta(\Bcal)$, }

\begin{lemma}[Restatement of Lemma~\ref{lem:ftrl-ftpl}]
\label{lem:ftrl-ftpl-restated}
% \edit{Suppose $R_\Ncal=\Phi_\mathcal{N}^*$ as defined in Equation~\eqref{eqn:r}. Then,}{
For any $\Theta \in \RR^d$,
  \[
 \mathop{\argmin} \limits_{u \in \Delta(\Bcal)} \rbr{  \langle{\Theta},{u} \rangle + R_{\Ncal}(u) }
%   \EE_{\ell_x\sim \mathcal{N}(0,I_A)} \sbr{ \mathop{\argmax} \limits_{u \in \Delta(\Bcal)} \inner{-\Theta+ q}{u}}
  =
\nabla \Phi_\mathcal{N}(-\Theta).
  \]
%  where $q = (\sum_{x\in \Xcal}\ell_x(h(x)))_{h\in \Bcal}$. 
\end{lemma}

% Using this fact, $R_\Ncal^* = \Phi_\Ncal$. Lemma~\ref{lem:ftrl-ftpl-restated} now follows from Fact~\ref{fact:sc-conj}.
% so Lemma~\ref{lem:ftrl-ftpl-general} can be removed?

\begin{proof}
As shown by Lemma~\ref{lem:sc_dual_restated} and Lemma~\ref{lem:r_domain}, $R_{\mathcal{N}}$ is closed and strongly convex with $\dom (R_{\mathcal{N}} )= \Delta(\Bcal)$. By applying Proposition~\ref{prop:sc-conj} on $R_{\mathcal{N}}$, we have 
\[ 
   \nabla R^*_\mathcal{N}(-\Theta)
   =
     \mathop{\argmax}\limits_{u\in \RR^B}  \inner{ -\Theta}{u} - R_{\Ncal}(u)
   =
     \mathop{\argmin}\limits_{u\in \Delta(\Bcal)}  \inner{ \Theta}{u} +R_{\Ncal}(u),
\]  
where the second equality is by $\dom (R_\Ncal) = \Delta(\Bcal)$ shown in Lemma~\ref{lem:r_domain}. 

As shown by Lemma~\ref{lem: closed and convex function}, $\Phi_{\mathcal{N}}$ is closed and convex. By item~\ref{item:one_to_one} of Fact~\ref{fact:conj}, $R^*_\mathcal{N} = \Phi_{\mathcal{N}}^{**} = \Phi_{\mathcal{N}}$, which concludes the proof.
\end{proof}

\begin{lemma}
\label{lem:maximum-distance}
$\sup_{u \in \Delta(\Bcal)} R_\Ncal(u) - \inf_{u \in \Delta(\Bcal)} R_\Ncal(u) \leq \sqrt{2 X \ln(B)}.$
\end{lemma}
\begin{proof}
First, by Lemma~\ref{lem:upperbound_of_R}, since $\forall u \in \Delta(\Bcal)$, $ R_\Ncal(u) \leq 0$. we have that $\sup_{u \in \Delta(\Bcal)} R_\Ncal(u) \leq 0$.
Next, we show $-\inf_{u \in \Delta(\Bcal)} R_\Ncal(u) \leq \sqrt{2X \ln(B)}$. Since $- \inf_{u \in \Delta(\Bcal)} R_\Ncal(u) = \sup_{u \in \Omega} \left\langle 0, u \right\rangle - R_\Ncal(u)
     =\Phi_\Ncal(0)$, it suffices to show $\Phi_\Ncal(0) \leq\sqrt{2X \ln(B)}$, which we already shown in Lemma~\ref{lem: strongly smooth zero potential upperbound}.
     
 Together,  we conclude $\sup_{u \in \Delta(\Bcal)} R_\Ncal(u) - \inf_{u \in \Delta(\Bcal)} R_\Ncal(u) \leq \sqrt{2 X \ln(B)}$.
\end{proof}

Now, combining the above lemmas with the general optimistic FTRL lemma, we get the following central regret theorem for optimistic FTRL with separator perturbation-based regularizers for our results:
\begin{theorem}
\label{thm: Regret of  EFTPL  with Prediction}
Suppose $\Xcal$ is a separator set for $\Bcal$, Optimistic FTRL (Algorithm \ref{alg: Optimistic FTRL}) with $R = R_\Ncal$
achieves regret 
$$\SLReg_N
\leq
\frac{\sqrt{2 X \ln(B)}}{\eta} 
    +
    \sum_{n=1}^N(  
  \eta XA \|\est_n - \hat{\est}_n\|_\infty ^2
  -
    \frac{1 }{4 \eta XA} \|\nabla R^*_\mathcal{N}(-\hat{\Theta}_n) - \nabla R^*_\mathcal{N}(-\Theta_{n-1})\|_1^2).
$$

Furthermore, if $\hat{g}_n = 0$ for all $n$, $$\SLReg_N
\leq
\frac{\sqrt{2 X \ln(B)}}{\eta} 
    +
    \eta XA \sum_{n=1}^N 
  \| \est_n \|_\infty ^2.
$$
% \chicheng{Need to change all $\eta$ to $\frac1\eta$ here?}
\end{theorem}
\begin{proof}
Since $R_\Ncal$ is $\sqrt{\frac{\pi}{8}}{\frac{1}{ AX}}$-strongly convex by Lemma~\ref{lem:sc_dual_restated}, by the regret guarantee of Optimistic FTRL in Theorem~\ref{thm:optimistic-ftrl}, 
\begin{align*}
\SLReg_N
\leq &
\frac{\sup_{w \in \Omega} R(w)
- 
\inf_{w \in \Omega} R(w)}{\eta}\\
&+
 \sum_{n=1}^N(  
 \sqrt{\frac{2}{\pi}}\eta XA \|\est_n - \hat{\est}_n\|_\infty ^2
  -
    \sqrt{\frac{\pi}{32}}\frac{1 }{\eta XA} \|\nabla R^*_\mathcal{N}(-\hat{\Theta}_n) - \nabla R^*_\mathcal{N}(-\Theta_{n-1})\|_1^2).
\end{align*}
By bringing in $\sup_{u \in \Delta(\Bcal)} R_\Ncal(u) - \inf_{u \in \Delta(\Bcal)} R_\Ncal(u) \leq \sqrt{2 X \ln(B)}$ proved by Lemma~\ref{lem:maximum-distance} and using the simple facts that $\sqrt{\frac{2}{\pi}} \leq 1$ and $\sqrt{\frac{\pi}{32}} \geq \frac{1}{4}$, we conclude the proof of the first inequality. 
Specifically, when  $\hat{g}_n = 0$ for all $n$, by Algorithm~\ref{alg: Optimistic FTRL} we have that $\hat{\Theta}_n = \Theta_{n-1}$ and

\[\SLReg_N
\leq
\frac{\sqrt{2 X \ln(B)}}{\eta} 
    +
   \eta XA \sum_{n=1}^N 
  \| \est_n \|_\infty ^2.
  \qedhere
\]
\end{proof}

\begin{lemma}
\label{lem: lowerbound for strongly smooth function}
Let $R: \RR^d \rightarrow \RR \cup \cbr{+ \infty}$ be a closed and $\alpha$-strongly convex function with respect to $\| \cdot \|$, 
% If closed and convex function $R$ is $\alpha$-strongly convex with respect to $\| \cdot \|$,  
 then, for $\Theta, \Theta' \in \RR^d$,
\[
D_{R^*}(\Theta, \Theta') 
\geq
\frac{\alpha}{2} \| \nabla R^*(\Theta) - \nabla R^*(\Theta) \|^2.
\]
\end{lemma}

  \begin{proof}
   Since $R$ is closed and $\alpha$-strongly convex, by item~\ref{item:closed_convex},\ref{item:sc_strongly smooth_duality} of Fact~\ref{fact:conj} and Proposition~\ref{prop:sc-conj}, $R^*$ is closed, convex, differentiable, and $\dom (R^*) = \RR^d$. By item~\ref{item:fenchel-young-equality} of Fact~\ref{fact:conj}, $\forall \Theta, \Theta' \in \RR^B$,
  \[
\left\{
\begin{aligned}
R^*(\Theta) =&  \left\langle \Theta, \nabla{R}^*(\Theta)\right\rangle - R( \nabla{R}^*(\Theta)), \\
R^*(\Theta') =& \left\langle \Theta', \nabla{R}^*(\Theta')\right\rangle - R( \nabla{R}^*(\Theta')).\\
\end{aligned}
\right. 
\] 
As a consequence, both $\nabla R^*(\Theta)$ and $\nabla R^*(\Theta')$ are in $\dom (R)$.
%$D_{R^*}(\Theta', \Theta) \geq 0$,  
% \edit{Furthermore, since $R^*$ is convex, $\forall \Theta, \Theta' \in \RR^B$, $ R( \nabla{R^*}(\Theta')) - R( \nabla{R^*}(\Theta)) -
%    \left\langle \Theta, \nabla{R^*}(\Theta') - \nabla{R^*}(\Theta) \right\rangle \geq 0$. This implies $\Theta$ is a subgradient of $R( \nabla{R^*}(\Theta)) $.}{
   Furthermore, by item~\ref{item:fenchel-young-equality} of Fact~\ref{fact:conj}, $\Theta \in \partial R( \nabla R^*(\Theta) )$
and the definition of Bregman divergence, we have
  \[
  \begin{aligned}
  D_{R^*}(\Theta', \Theta) 
  = &
  {R^*}(\Theta') - {R^*}(\Theta) 
  - \left\langle \Theta'-\Theta, \nabla{R^*}(\Theta')\right\rangle \\
  = &
  \left\langle \Theta, \nabla{R^*}(\Theta)\right\rangle - R( \nabla{R^*}(\Theta))
  -\left\langle \Theta', \nabla{R^*}(\Theta')\right\rangle + R( \nabla{R^*}(\Theta'))
     \\
   & - \left\langle \Theta'-\Theta, \nabla{R^*}(\Theta')\right\rangle \\
   = &
   R( \nabla{R^*}(\Theta')) - R( \nabla{R^*}(\Theta)) -
   \left\langle \Theta, \nabla{R^*}(\Theta') - \nabla{R^*}(\Theta) \right\rangle \\
   \geq &
   \frac{\alpha}{2} \| \nabla R^*(\Theta) - \nabla R^*(\Theta) \|^2,
   \end{aligned}
   \]
   where the last inequality uses the  $\alpha$-strong convexity of $R$, as well as $\Theta \in \partial R( \nabla R^*(\Theta) )$.
  \end{proof}

\section{Auxiliary Lemmas}
\label{sec:auxiliary}
\begin{lemma}
\label{lem: performance difference lemma}
For two stationary policies $\pi$ and $\pi^E$ $: \Scal \to \Delta(\Acal)$, we have
\[
J(\pi) - J(\pi^E) = H \cdot\EE_{s \sim d_{\pi}}\EE_{a \sim \pi(\cdot|s)} \sbr{A^{E}(s,a)},
\]
where $A^E(s,a) :=Q_{\pi^E}(s, a) - V_{\pi^E}(s) $,  $V_{\pi^E}(s) := \EE\sbr{ \sum_{t=\step(s)}^{H} c(s_t,a_t) \mid s, \pi^E}$, and $Q_{\pi^E}(s, a) :=  c(s,a) + \EE\sbr{ \sum_{t=\step(s)+1}^{H} c(s_t,a_t) \mid  s, a, \pi^E}$.
\end{lemma}
The proof can be found at e.g. \cite{ross2014reinforcement}[Lemma 4.3]

\begin{lemma}
\label{lem:miniseparator}
For benchmark policy class $\Bcal$ that contains $B$ deterministic policies $h: \Scal \rightarrow \Acal$, consider separator set $\Xcal$
~\ref{def:small-sep}  with $X = \abr{ \Xcal }$, $A = |\Acal|$. Then, 
\[
X \geq \log_A(B).
\]
\end{lemma}

\begin{proof}
Define $\Bcal_\Xcal = \cbr{ (h(x_1), \ldots, h(x_X)) }_{h \in \Bcal}$, where $(h(x_1), \ldots, h(x_X)) \in \Acal^X$. 
First, note that $\Bcal_\Xcal \subset \Acal^\Xcal$, which implies that $|\Bcal_{\Xcal}| \leq |\Acal^\Xcal| = A^X$.

%the proof follows by showing $B = |\Bcal_\Xcal| \leq A^X$. \\
%To begin with, by definition of $\Bcal_\Xcal$, we have that $B = |\Bcal_\Xcal|$.
Secondly, by the definition of separator set $\Xcal$, $\forall h,h' \in \Bcal$, $\exists x \in \Xcal = \{x_1, \cdots, x_X\}$ , s.t. $h(x) \neq h'(x)$. This implies $\forall h,h' \in \Bcal$, $(h(x_1), \ldots, h(x_X)) \neq (h'(x_1), \ldots, h'(x_X))$, and 
every $h$ in $\Bcal$ induces unique $(h(x_1), \ldots, h(x_X))$; this implies that $|\Bcal_X| = |\Bcal| = B$.

Combining the above two observations, we conclude $B = |\Bcal_\Xcal| \leq A^X$, thus $X \geq \log_A(B)$.
\end{proof}

\end{document}